\documentclass{article}
	\PassOptionsToPackage{numbers,compress}{natbib}

\usepackage{multicol}
\usepackage{epsfig,epsf,psfrag}
\usepackage{amssymb,amsmath,amsthm,amsfonts,latexsym}
\usepackage{amsmath,graphicx,bm,url,overpic}
\usepackage{mathrsfs}
\usepackage{fixltx2e}%ordering of single and double column floats
\usepackage{array}%array and tabular environments
\usepackage{verbatim}
\usepackage{bm}
\usepackage{algorithmic}
\usepackage{algorithm}
\usepackage{verbatim}
\usepackage{textcomp}
\usepackage{epstopdf}
\usepackage{subfigure}
\usepackage{tikz}
\usepackage{acronym}
\usepackage{diagbox}
\usepackage{wrapfig}
    \usepackage[final]{neurips_2021}

\usepackage[utf8]{inputenc} % allow utf-8 input
\usepackage[T1]{fontenc}    % use 8-bit T1 fonts
\usepackage{hyperref}       % hyperlinks
\usepackage{url}            % simple URL typesetting
\usepackage{booktabs}       % professional-quality tables
\usepackage{amsfonts}       % blackboard math symbols
\usepackage{nicefrac}       % compact symbols for 1/2, etc.
\usepackage{microtype}      % microtypography
\usepackage{xcolor}         % colors
\usepackage{soul}
\newtheorem{theorem}{Theorem}
\newtheorem{proposition}{Proposition}
\newtheorem{lemma}[proposition]{Lemma}%[proposition]

\newtheorem{corollary}{Corollary}
\newtheorem{definition}{Definition}

\newtheorem{assumption}{Assumption}

\newcommand\independent{\protect\mathpalette{\protect\independenT}{\perp}}
\def\independenT#1#2{\mathrel{\rlap{$#1#2$}\mkern2mu{#1#2}}}

\acrodef{rg}[RG]{Recursive Grouping}
\acrodef{rrg}[RRG]{Robust Recursive Grouping}
\acrodef{rclrg}[RCLRG]{Robust Chow-Liu Recursive Grouping}
\acrodef{clrg}[CLRG]{Chow-Liu Recursive Grouping}
\acrodef{mst}[MST]{minimum spanning tree}
\acrodef{ggm}[GGM]{Gaussian graphical models}
\acrodef{rsnj}[RSNJ]{Robust Spectral NJ}
\acrodef{nj}[NJ]{Neighbor Joining}
\acrodef{snj}[SNJ]{Spectral NJ}
\acrodef{rnj}[RNJ]{Robust Neighbor Joining}
\acrodef{hmm}[HMM]{hidden Markov model}
\acrodef{cl}[CL]{Chow-Liu}

\title{Robustifying Algorithms of Learning Latent Trees with Vector Variables}

\author{%
Fengzhuo Zhang\\
Department of Electrical and Computer Engineering\\
National University of Singapore\\
  \texttt{fzzhang@u.nus.edu} \\
  \And
  Vincent Y.~F.~Tan\\
  Department of Electrical and Computer Engineering\\
  Department of Mathematics\\
  National University of Singapore\\
  \texttt{vtan@nus.edu.sg} \\
}

\begin{document}

\maketitle

\begin{abstract}
	We consider  learning the structures of  Gaussian latent tree models with vector observations when a subset of them are arbitrarily corrupted. First, we present the sample complexities of \ac{rg} and \ac{clrg} without the assumption that the effective depth 
	is bounded in the number of observed nodes, significantly generalizing the results in Choi et al.\ (2011). We show that Chow-Liu initialization 
	in \ac{clrg} greatly reduces the sample complexity of \ac{rg} from  being exponential in the diameter of the tree to only logarithmic in the diameter for the \ac{hmm}. Second, we robustify 
	\ac{rg}, \ac{clrg}, \ac{nj} and \ac{snj} by using the truncated inner product. These robustified algorithms can tolerate a number of corruptions up to the square root of the number of clean samples. Finally, we derive the first known instance-dependent impossibility result for structure learning of latent trees. The optimalities of the robust version of \ac{clrg} and \ac{nj} are verified by comparing their sample complexities and the impossibility result.
\end{abstract}
\section{Introduction}\label{sec:intro}
Latent graphical models provide a succinct representation of the dependencies among observed and latent variables. 
Each node in the graphical model represents a random variable or a random vector, and 
the dependencies among these variables are captured by the edges among nodes. Graphical models are widely used in domains from biology \cite{saitou1987neighbor}, computer vision \cite{tang2014latent} 
and social networks \cite{eisenstein2010latent}.

This paper focuses on the structure learning problem of latent  tree-structured \ac{ggm} in which the node observations are random {\em vectors} and a subset of the  observations can be {\em arbitrarily corrupted}. This classical problem, in which the variables are {\em clean scalar} random variables, has been studied extensively in the past decades. The first information distance-based method, \ac{nj}, 
was proposed in \cite{saitou1987neighbor} to learn the structure of phylogenetic trees. This method makes use of additive information distances to deduce the existence of hidden nodes and introduce 
edges between hidden and observed nodes. \ac{rg}, proposed in \cite{choi2011learning}, generalizes the information distance-based methods to make it applicable for the latent graphical models 
with general structures. Different from these information distance-based methods, quartet-based methods \cite{anandkumar2011spectral} utilize the relative geometry of every four nodes 
to estimate the structure of the whole graph. Although experimental comparisons of these algorithms were conducted in some works \cite{choi2011learning, jaffe2021spectral, Casanellas21arxiv}, since there is no instance-dependent impossibility 
result of the sample complexity of structure learning problem of latent tree graphical models, no thorough theoretical comparisons have been made, and the optimal dependencies on the 
diameter of graphs and the maximal distance between nodes $\rho_{\max}$ have not been found.

The success of the previously-mentioned algorithms relies on the assumption that the observations are i.i.d.\ samples from the generating distribution.  The structure learning of latent graphical models in 
presence of  (random or adversarial) noise remains a relatively unexplored problem. The presence of the noise in the samples violates the i.i.d.\ assumption. Consequently,   classical algorithms may suffer from   severe performance degradation in the noisy setting. There are some works studying the problem of structure learning of graphical models with noisy samples, where all the nodes in the graphical models are
observed and not hidden. Several assumptions on the additive noise are made in these works, which limit the use of these proposed algorithms. For example, the covariance matrix of the noise is specified in 
\cite{katiyar2019robust}, and the independence and/or distribution of the noise is assumed in \cite{nikolakakis2019learning,tandon2020exact,tandon2021sga,Casanellas21arxiv}. In contrast, we consider the structure 
learning of latent tree graphical models with \emph{arbitrary} corruptions, where assumptions on the distribution and independence of the noise across nodes are not required \cite{wang2017robust}. Furthermore, the 
corruptions are allowed to be presented at {\em any position} in the data matrix; they do not appear solely as outliers. In this work, we derive  bounds on the maximum number of corruptions that can be tolerated for a variety of algorithms, and yet structure  learning can succeed with high probability.

Firstly, we derive the sample complexities of \ac{rg} and \ac{clrg} where each node represents a random {\em vector}; this differs from previous works where each node is {\em scalar} random variable (e.g.,  \cite{choi2011learning,parikh11}).
We explore the dependence of the sample complexities on the parameters of the graph. Compared with \cite[Theorem~12]{choi2011learning}, the derived sample complexities are applicable to a wider class of latent trees and capture the dependencies on more parameters of the underlying graphical models, such as $\rho_{\max}$, the maximum distance between any two nodes, and $\delta_{\min}$, the minimum over all determinants of the covariance matrices of the vector variables. Our sample complexity analysis clearly demonstrates and precisely quantifies the effectiveness of the Chow-Liu~\cite{chow1968approximating}  initialization  step
  in \ac{clrg}; this has been only verified experimentally~\cite{choi2011learning}. For the particular case of the  \ac{hmm}, we show that the Chow-Liu initialization  step reduces the sample complexity of 
\ac{rg} which is $O\big((\frac{9}{2})^{\mathrm{Diam}(\mathbb{T})}\big)$ to $O\big(\log \mathrm{Diam}(\mathbb{T})\big)$, where $\mathrm{Diam}(\mathbb{T})$ is the tree diameter.

Secondly, we robustify \ac{rg}, \ac{clrg}, \ac{nj} and \ac{snj} by using the truncated inner product \cite{chen2013robust} to estimate the information distances in the presence of arbitrary corruptions. We derive their  sample complexities and show that they can tolerate  $n_{1}=O\big(\frac{\sqrt{n_{2}}}{\log n_{2}}\big)$ corruptions, where $n_{2}$ is the number of clean samples.

Finally, we derive the first known instance-dependent impossibility result for   learning   latent trees. The dependencies on the number of observed nodes $|\mathcal{V}_{\mathrm{obs}}|$ and the maximum distance 
$\rho_{\max}$ are delineated. The comparison of the sample complexities of the structure learning algorithms and the impossibility result demonstrates the optimality of \ac{rclrg} and \ac{rnj} in 
$\mathrm{Diam}(\mathbb{T})$ for some archetypal latent tree structures.

\paragraph{Notation}
We use san-serif letters $x$, boldface letters $\mathbf{x}$, and bold uppercase letters $\mathbf{X}$ to denote variables, vectors and matrices, respectively. The notations $[\mathbf{x}]_{i}$, 
$[\mathbf{X}]_{ij}$, $[\mathbf{X}]_{:,j}$ and $\mathrm{diag}(\mathbf{X})$ are respectively the $i^{\mathrm{th}}$ entry of vector $\mathbf{x}$, the  $(i,j)^{\mathrm{th}}$ entry of $\mathbf{X}$, the  
$j^{\mathrm{th}}$ column of  $\mathbf{X}$, and the diagonal entries of matrix $\mathbf{X}$. The notation $x^{(k)}$ represents the $k^{\mathrm{th}}$ sample of  $x$. $\|\mathbf{x}\|_{0}$ is the $l_{0}$ norm of the vector $\mathbf{x}$, i.e., the number of non-zero terms in $\mathbf{x}$. The set $\{1,\ldots,n\}$ is denoted as $[n]$.
For a tree $\mathbb{T}=(\mathcal{V},\mathcal{E})$, the internal (non-leaf) nodes, the maximal degree and the diameter of $\mathbb{T}$ are denoted as $\mathrm{Int}(\mathbb{T})$, 
$\mathrm{Deg}(\mathbb{T})$, and $\mathrm{Diam}(\mathbb{T})$, respectively. We denote the closed neighborhood and the degree of   $x_{i}$ as $\mathrm{nbd}[x_{i};\mathbb{T}]$ and $\mathrm{deg}(i)$, respectively. The 
length of the (unique) path connecting  $x_{i}$ and $x_{j}$ is denoted as $\mathrm{d}_{\mathbb{T}}(x_{i},x_{j})$.

\section{Preliminaries and problem statement}
A \ac{ggm} \cite{lauritzen1996graphical, tan2011}  is a multivariate Gaussian distribution that factorizes according to an undirected graph $\mathbb{G}=(\mathcal{V},\mathcal{E})$. More precisely, a $l_{\mathrm{sum}}$-dimensional random vector $\mathbf{x}=[\mathbf{x}_{1}^{\top},\ldots,\mathbf{x}_{p}^{\top}]^{\top}$, 
where $\mathbf{x}_{i}\in \mathbb{R}^{l_{i}}$ and $l_{\mathrm{sum}}=\sum_{i=1}^{p} l_{i}$, follows a Gaussian distribution $\mathcal{N}(\mathbf{0},\mathbf{\Sigma})$, and it is said to be \emph{Markov} on a graph $\mathbb{G}=(\mathcal{V},\mathcal{E})$ with vertex set $\mathcal{V}=\{x_{1},\ldots,x_{p}\}$ and 
edge set $\mathcal{E}\subseteq \binom{\mathcal{V}}{2}$ and $(x_{i},x_{j})\in \mathcal{E}$ if and only if the $(i,j)^{\mathrm{th}}$ block $\mathbf{\Theta}_{ij}$ of the precision $\mathbf{\Theta}=\mathbf{\Sigma}^{-1}$ is not the zero matrix $\mathbf{0}$. We focus on tree-structured graphical models, which factorize according to acyclic and connected (tree) graphs.

A special class of graphical models is the set of {\em latent} graphical models $\mathbb{G}=(\mathcal{V},\mathcal{E})$. The vertex set $\mathcal{V}$ is decomposed as $\mathcal{V}=\mathcal{V}_{\mathrm{hid}}\cup\mathcal{V}_{\mathrm{obs}}$. %Acquiring $n$ i.i.d.\ samples from this graphical 
%model, 
We only have access to $n$ i.i.d.\ samples drawn from  the observed set of nodes $\mathcal{V}_{\mathrm{obs}}$. The two goals of any structure learning algorithm are to learn the identities of the hidden nodes $\mathcal{V}_{\mathrm{hid}}$ and how they are connected to the observed nodes.

\subsection{System model for arbitrary corruptions}\label{subsec:sysmodel}
We consider tree-structured \ac{ggm}s $\mathbb{T}=(\mathcal{V},\mathcal{E})$ with observed nodes $\mathcal{V}_{\mathrm{obs}}=\{x_{1},\cdots,x_{o}\}$ and hidden nodes $\mathcal{V}_{\mathrm{hid}}=\{x_{o+1},\cdots,x_{o+h}\}$, 
where $\mathcal{V}=\mathcal{V}_{\mathrm{hid}}\cup\mathcal{V}_{\mathrm{obs}}$ and $\mathcal{E}\subseteq \binom{\mathcal{V}}{2}$. Each node $x_{i}$ represents a random {\em vector} $\mathbf{x}_{i}\in \mathbb{R}^{l_{i}}$. The concatenation of 
these random vectors is a multivariate Gaussian random vector with zero mean and covariance matrix $\mathbf{\Sigma}$ with size $l_{\mathrm{sum}}\times l_{\mathrm{sum}}$. %where $l_{\mathrm{sum}}=\sum_{i=1}^{|\mathcal{V}_{\mathrm{obs}}|}l_{i}$.

We have $n$ i.i.d.\ samples $\tilde{\mathbf{X}}_{j}=[\tilde{\mathbf{x}}_{1}^{(j)\top},\cdots,\tilde{\mathbf{x}}_{o}^{(j)\top}]^{\top}\in\mathbb{R}^{l_{\mathrm{sum}}},  j=1,\ldots,n$ drawn from the  observed nodes $\mathcal{V}_{\mathrm{obs}}=\{x_{1},\cdots,x_{o}\}$. 
However, the {\em observed} data matrix $\tilde{\mathbf{X}}_{1}^{n}=[\tilde{\mathbf{X}}_{1},\cdots,\tilde{\mathbf{X}}_{n}]^{\top}\in \mathbb{R}^{n\times l_{\mathrm{sum}}}$ may contain some corrupted elements. We allow an level-$(n_1/2)$ arbitrary corruption  in the data matrix. This is made precise in the following definition.  
\begin{definition}[Level-$m$ arbitrary corruption]\label{def:corrup}
    For the data matrix $\tilde{\mathbf{X}}_{1}^{n}\in\mathbb{R}^{n\times k}$ formed by $n$ clean samples of $k$ random variables (or a random vector of dimension $k$), an  {\em level-$m$ arbitrary corruption}  transforms  $\tilde{\mathbf{X}}_{1}^{n}$ into $\mathbf{X}_{1}^{n}\in \mathbb{R}^{n\times k}$ such that
    \begin{align}
        \|[\tilde{\mathbf{X}}_{1}^{n}]_{;,i}-[\mathbf{X}_{1}^{n}]_{;,i}\|_{0}\leq m\quad  \text{ for all }\quad i=1,\ldots ,k.
    \end{align}
\end{definition}
 Definition \ref{def:corrup} implies that there are at most $n_{1}/2$ corrupted terms in each column of $\mathbf{X}_{1}^{n}$; the remaining $n-n_{1}/2$ samples in this column are clean. In particular, 
the corrupted samples in different columns need not to be in the same rows. If the corruptions in different columns lie in the same rows, as shown in (the left of) Fig.~\ref{corrup_pic}, all the samples in the corresponding rows 
are corrupted; these are called \emph{outliers}. Obviously, outliers form a special case of our corruption model. Since each variable has at most $n_{1}/2$ corrupted samples, the sample-wise inner product between two variables has 
at least $n_{2}=n-n_{1}$ clean samples. There is no constraint on the statistical dependence or patterns of  the corruptions. 
Unlike fixing the covariance matrix of the noise \cite{katiyar2019robust} or keeping the noise independent \cite{nikolakakis2019learning}, 
we allow \emph{arbitrary} corruptions on the samples, which means that the noise can have unbounded amplitude, can be dependent, and even can be 
generated from another graphical model (as we will see in the experimental results in Section~\ref{sec:simu}).

\subsection{Structural and distributional assumptions}\label{sec:assump}
To construct the correct latent tree from samples of observed nodes, it is imperative to constrain the class of latent trees to guarantee that the information from the distribution of observed nodes $p(\mathbf{x}_{1},\ldots,\mathbf{x}_{o})$ is sufficient to construct the 
tree. The distribution $p(\mathbf{x}_{1},\ldots,\mathbf{x}_{o+h})$ of the observed and hidden nodes is said to have a \emph{redundant} hidden node $x_{j}$ if the distribution of the observed nodes $p(\mathbf{x}_{1},\ldots,\mathbf{x}_{o})$ remains the same after we marginalize over $x_{j}$. To ensure that a latent tree can be constructed 
with no ambiguity, we need to guarantee that the true distribution does not have any redundant hidden node(s), which is achieved by following two conditions \cite{pearl2014probabilistic}: (C1) Each hidden node has at least three neighbors; the set of such latent trees is denoted as $\mathcal{T}_{\geq 3}$; (C2)
	Any two variables connected by an edge are neither perfectly dependent nor independent.

%In the rest of this paper, we only consider latent graphical models that satisfy these two conditions. We denote the set of trees satisfying (C1) as $\mathcal{T}_{\geq 3}$. 
%We also consider graphical models that satisfy the following assumptions.

\begin{assumption}\label{assupleng}
	The dimensions of all the random vectors are all equal to $l_{\max}$.
\end{assumption}
%In fact, it is not necessary for all the random vectors $\mathbf{x}_{i}$ to have the same length. We only need the random vectors of the internal nodes to have the same length. However, for ease of notation, 
In fact,  we only require the random vectors of the internal (non-leaf) nodes to have the same length. However, for ease of notation, 
we assume that the dimensions of all the random vectors are $l_{\max}$.
\begin{assumption}\label{assupsing}
	For every $x_{i},x_{j}\in \mathcal{V}$, the covariance matrix $\mathbf{\Sigma}_{ij}=\mathbb{E}\big[\mathbf{x}_{i}\mathbf{x}_{j}^{\top}\big]$ has full rank, and the smallest singular value of $\mathbf{\Sigma}_{ij}$ is lower bounded by $\gamma_{\min}$, i.e.,
	\begin{align}
		\sigma_{l_{\max}}(\mathbf{\Sigma}_{ij})\geq \gamma_{\min}  \quad\text{for all}\quad x_{i},x_{j}\in \mathcal{V},
	\end{align}
	where $\sigma_{i}(\mathbf{\Sigma})$ is the $i^{\mathrm{th}}$ largest singular value of $\mathbf{\Sigma}$.
\end{assumption}
This   assumption is a strengthening of  Condition (C2) when each node represents a random vector.
\begin{assumption}\label{assupdet}
	The determinant of the covariance matrix of any node $\mathbf{\Sigma}_{ii}=\mathbb{E}\big[\mathbf{x}_{i}\mathbf{x}_{i}^{\top}\big]$ is lower bounded by $\delta_{\min}$, and the diagonal terms of the covariance matrix are upper bounded by $\sigma_{\max}^{2}$, i.e., 
	\begin{align}
		\min_{x_{i}\in\mathcal{V}} \det(\mathbf{\Sigma}_{ii})\geq \delta_{\min} \quad\mbox{and}\quad \max_{x_{i}\in\mathcal{V}}\mathrm{diag}\big(\mathbf{\Sigma}_{ii}\big) \leq \sigma_{\max}^{2}.
	\end{align}
\end{assumption}
Assumption~\ref{assupdet} is  natural; otherwise, $\mathbf{\Sigma}_{ii}$ may be arbitrarily close to a singular matrix.
\begin{assumption}\label{assupdegree}
	The degree of each node is upper bounded by $d_{\max}$, i.e., $\mathrm{Deg}(\mathbb{T})\leq d_{\max}$.
\end{assumption}
%This degree assumption constrains the number of neighbors of each node to be at most $d_{\max}$.

\subsection{Information distance}
We define the \emph{information distance} for Gaussian random vectors and prove that it is additive for trees. %-structured graphical models. For the sake of completeness, we provide 
%another justification of this fact in Appendix \ref{appendix:add_dist}.
\begin{definition}\label{def:dist}
	The information distance between nodes $x_{i}$ and $x_{j}$ is
	\begin{align}
		\mathrm{d}(x_{i},x_{j})=-\log\frac{\prod_{k=1}^{l_{\max}}\sigma_{k}\big(\mathbf{\Sigma}_{ij}\big)}{\sqrt{\mathrm{det}\big(\mathbf{\Sigma}_{ii}\big)\mathrm{det}\big(\mathbf{\Sigma}_{jj}\big)}}.
	\end{align}
\end{definition}
Condition (C2) can be equivalently restated as constraints on the information distance.
\begin{assumption}\label{assupdist}
	There exist two constants $0<\rho_{\min}\leq\rho_{\max}<\infty$ such that.
	\begin{align}
		\rho_{\min}\leq\mathrm{d}(x_{i},x_{j})\leq \rho_{\max}  \quad \text{for all}\quad  x_{i},x_{j}\in \mathcal{V}.
	\end{align}
\end{assumption}
Assumptions \ref{assupsing} and \ref{assupdist} both describe the properties of the correlation between random vectors from different perspectives. In fact, we can relate the constraints in these two 
assumptions as follows:
\begin{align}\label{eq:simi}
	\gamma_{\min}e^{\rho_{\max}/l_{\max}}\geq \delta_{\min}^{1/l_{\max}}.
\end{align}

\begin{proposition}\label{prop:add}
	If Assumptions \ref{assupleng} and \ref{assupsing} hold, $\mathrm{d}(\cdot,\cdot)$ defined in Definition \ref{def:dist} is additive on the tree-structured \ac{ggm} $\mathbb{T}=(\mathcal{V},\mathcal{E})$. In 
	other words, $\mathrm{d}(x_{i},x_{k})=\mathrm{d}(x_{i},x_{j})+\mathrm{d}(x_{j},x_{k})$ holds for any two nodes $x_{i},x_{k}\in \mathcal{V}$ and any node $x_{j}$ on the path connecting $x_{i}$ and $x_{k}$ in $\mathbb{T}$.
\end{proposition}
This additivity property is used extensively in the following algorithms. It was first stated and proved in Huang et al.~\cite{huang2020guaranteed}. We provide an alternative proof in Appendix \ref{appendix:add_dist}.
%In our paper, the definition of information distances involves the singular values of 
%$\mathbf{\Sigma}_{ij}$, while some other works proved the additivity of information distances involved with eigenvalues of $\mathbf{\Sigma}_{ij}$ \cite{huang2020guaranteed}. We provide the proof of additivity of information distances 
%defined by the singular values of $\mathbf{\Sigma}_{ij}$ in Appendix {\color{blue}XXX}.

\section{Robustifying latent tree structure learning algorithms} \label{sec:robustifying}
\subsection{Robust estimation of information distances}\label{subsec:infodistest}
Before delving into the details of robustifying latent tree structure learning algorithms, we first introduce the truncated inner product \cite{chen2013robust}, which estimates the correlation against arbitrary corruption effectively and 
serves as a basis for the robust latent tree structure learning algorithms.  Given $\mathbf{a},\mathbf{b}\in \mathbb{R}^{n}$ and an integer $n_{1}$, we compute $q_{i}=a_{i}b_{i}$ for $i=1,2,\ldots,n$ and sort $\{|q_{i}|\}$. Let  $\Upsilon$ be the index set of the $n-n_{1}$  smallest $|q_i|$'s. The truncated inner product is $\langle\mathbf{a},\mathbf{b}\rangle_{n_{1}}=\sum_{i\in\Upsilon}q_{i}$. 
Note that the implementation of the  truncated inner product requires the knowledge of corruption level $n_{1}$. % This parameter can be learned by cross-validation, if we have no knowledge of it.

To estimate the information distance defined in Definition \ref{def:dist}, we implement the truncated inner product to estimate each term of $\mathbf{\Sigma}_{ij}$, i.e., $[\hat{\mathbf{\Sigma}}_{ij}]_{st}=\frac{1}{n-n_{1}}\langle[\mathbf{X}_{1}^{n}]_{:,(i-1)l_{\max}+s},[\mathbf{X}_{1}^{n}]_{:,(j-1)l_{\max}+t}\rangle_{n_{1}}$. 
Then the information distance is computed based on this estimate 
of $\mathbf{\Sigma}_{ij}$ as
\begin{align}
	\hat{\mathrm{d}}(x_{i},x_{j})=-\log\prod_{k=1}^{l_{\max}}\sigma_{k}\big(\hat{\mathbf{\Sigma}}_{ij}\big)+\frac{1}{2}\log\mathrm{det}\big(\hat{\mathbf{\Sigma}}_{ii}\big)+\frac{1}{2}\log\mathrm{det}\big(\hat{\mathbf{\Sigma}}_{jj}\big).
\end{align}
The truncated inner product guarantees that $\hat{\mathbf{\Sigma}}_{ij}$ converges in probability to $\mathbf{\Sigma}_{ij}$, which further ensures the convergence of the singular values and the determinant of $\mathbf{\Sigma}_{ij}$ to their nominal values.
\begin{proposition}\label{distconcen}
	If Assumptions \ref{assupleng} and \ref{assupsing} hold, the estimate of the information distance between $x_{i}$ and $x_{j}$ based on the truncated inner product $\hat{\mathrm{d}}(x_{i},x_{j})$ satisfies 
	\begin{align}
		\mathbb{P}\Big(\big|\hat{\mathrm{d}}(x_{i},x_{j})-\mathrm{d}(x_{i},x_{j})\big|> \frac{2l_{\max}^{2}}{\gamma_{\min}}(t_{1}+t_{2})\Big)\le 2l_{\max}^{2}e^{-\frac{3n_{2}}{16\kappa n_{1}} t_{1}}+l_{\max}^{2}e^{-c\frac{n_{2}}{\kappa^{2}}t_{2}^{2}}, \label{eqn:tail} %\ \forall x_{i},x_{j}\in \mathcal{V}_{\mathrm{obs}}
	\end{align}
	where $t_{2}<\kappa=\max\{\sigma_{\max}^{2},\rho_{\min}\}$, and $c$ is an absolute constant.
\end{proposition}
%The tail probability in Proposition \ref{distconcen} is the sum of two parts. 
The first and second parts of \eqref{eqn:tail} originate from the corrupted  and clean samples respectively. 

\subsection{Robust Recursive Grouping algorithm}\label{subsec:rrg}
The \ac{rg} algorithm was proposed in \cite{choi2011learning} to learn latent tree  models with additive information distances. We extend the \ac{rg} to be applicable to \ac{ggm}s with vector observations and robustify it to learn the 
tree structure against arbitrary corruptions. We call this robustified algorithm \ac{rrg}. \ac{rrg} makes use of the additivity of information distance to identify the relationship between nodes. For any three nodes $x_{i}$, $x_{j}$ and $x_{k}$, the difference between the information distances $\mathrm{d}(x_{i},x_{k})$ and $\mathrm{d}(x_{j},x_{k})$ 
is denoted as $\Phi_{ijk}=\mathrm{d}(x_{i},x_{k})-\mathrm{d}(x_{j},x_{k})$. %Then the relationship between nodes could be identified as follows
\begin{lemma}{\cite{choi2011learning}}\label{lem:identify}
	For information distances $\mathrm{d}(x_{i},x_{j})$ for all nodes $x_{i},x_{j}\in \mathcal{V}$ in a tree $\mathbb{T}\in \mathcal{T}_{\geq 3}$, $\Phi_{ijk}$ has following two properties: (1) $\Phi_{ijk}=\mathrm{d}(x_{i},x_{j}) \text{ for all }  x_{k}\in \mathcal{V}\backslash\{x_{i},x_{j}\}$ if and only if $x_{i}$ is a leaf node and $x_{j}$ is the parent of $x_{j}$ and (2)  $-\mathrm{d}(x_{i},x_{j})<\Phi_{ijk^{\prime}}=\Phi_{ijk}<\mathrm{d}(x_{i},x_{j}) \text{ for all }  x_{k},x_{k^{\prime}}\in \mathcal{V}\backslash\{x_{i},x_{j}\}$ if and only if $x_{i}$ and $x_{j}$ are leaves and share the same parent.
%	\begin{itemize}
%		\item[(1)] 
%		\item[(2)] 
%	\end{itemize}
\end{lemma}

\ac{rrg} initializes the active set $\Gamma^{1}$ to be the set of all observed nodes. In the $i^{\mathrm{th}}$ iteration, as shown in Algorithm \ref{algo:rrg}, \ac{rrg} adopts Lemma \ref{lem:identify} to identify relationships among nodes in active set $\Gamma^{i}$, and it 
removes the nodes identified as siblings and children from $\Gamma^{i}$ and adds newly introduced hidden nodes to form the active set $\Gamma^{i+1}$ in the $(i+1)^{\mathrm{st}}$ iteration. The procedure of estimating the distances between the newly-introduced hidden node $x_{\mathrm{new}}$ and other nodes is as 
follows. For the node $x_{i}$ which is the child of $x_{\mathrm{new}}$, i.e., $x_{i}\in \mathcal{C}(x_{\mathrm{new}})$, the information distance is estimated as
\begin{align}\label{distup1}
	\hat{\mathrm{d}}(x_{i},x_{\mathrm{new}})=\frac{1}{2\big(|\mathcal{C}(x_{\mathrm{new}})|-1\big)}\bigg(\sum_{j\in\mathcal{C}(x_{\mathrm{new}})}\hat{\mathrm{d}}(x_{i},x_{j})+\frac{1}{|\mathcal{K}_{ij}|}\sum_{k\in\mathcal{K}_{ij}}\hat{\Phi}_{ijk}\bigg),
\end{align}
where $\mathcal{K}_{ij}=\big\{x_{k}\in \mathcal{V}\backslash \{x_{i},x_{j}\} :\max\big\{\hat{\mathrm{d}}(x_{i},x_{k}),\hat{\mathrm{d}}(x_{j},x_{k})\big\}<\tau \big\}$ for some threshold $\tau>0$. For $x_{i}\notin \mathcal{C}(x_{\mathrm{new}})$, 
the distance is estimated as
\begin{align}\label{distup2}
	\hat{\mathrm{d}}(x_{i},x_{\mathrm{new}})=\left\{
		\begin{array}{ll}
		\sum_{x_{k}\in \mathcal{C}(x_{\mathrm{new}})}\frac{\hat{\mathrm{d}}(x_{k},x_{i})-\hat{\mathrm{d}}(x_{k},x_{\mathrm{new}})}{|\mathcal{C}(x_{\mathrm{new}})|}. & \text{if }x_{i}\in \mathcal{V}_{\mathrm{obs}} \\
		\sum_{(x_{k},x_{j})\in \mathcal{C}(x_{\mathrm{new}})\times \mathcal{C}(i)}\frac{\hat{\mathrm{d}}(x_{k},x_{j})-\hat{\mathrm{d}}(x_{k},x_{\mathrm{new}})-\hat{\mathrm{d}}(x_{j},y_{i})}{|\mathcal{C}(x_{\mathrm{new}})||\mathcal{C}(i)|} &\text{otherwise }
		\end{array}
		\right. .
\end{align}
\begin{wrapfigure}{r}{4.8cm}
	\centering
	\includegraphics[width=0.335\textwidth]{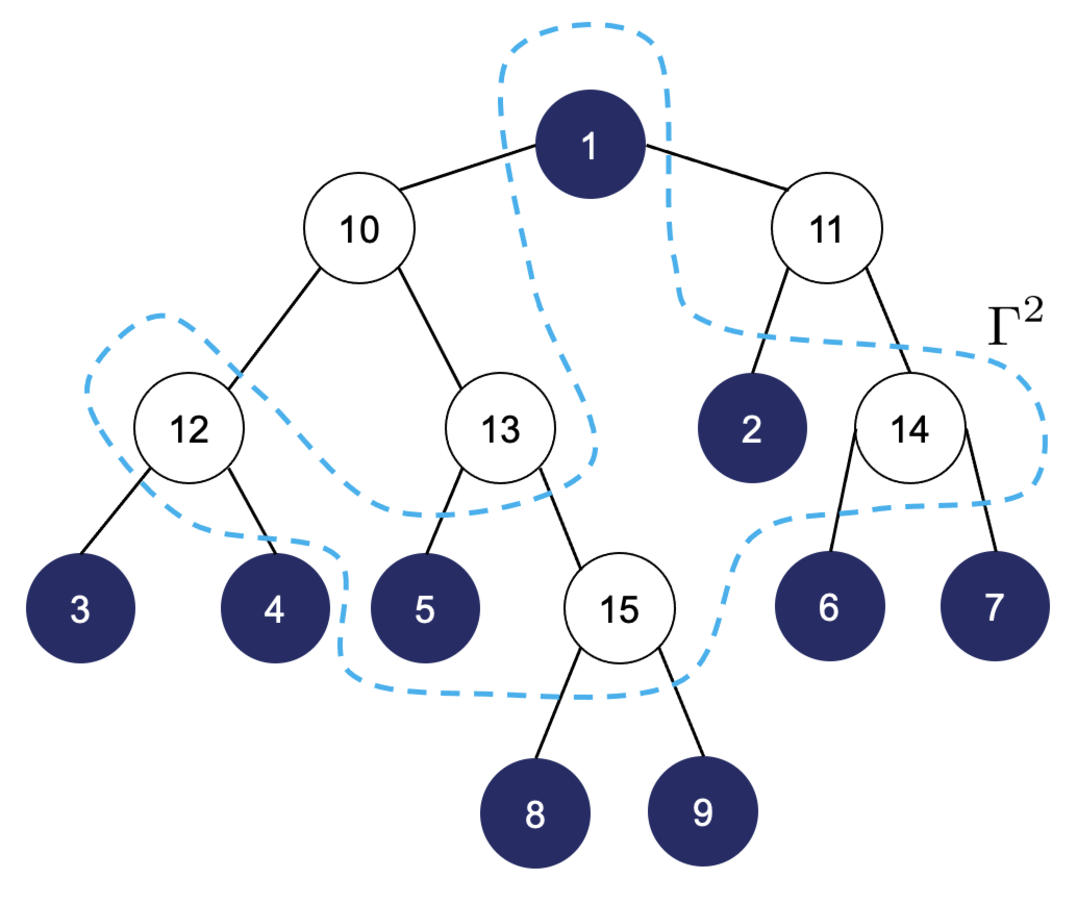}
	\caption{\small An illustration of the active set. The shaded nodes are the observed nodes and the rest are hidden nodes.}
	\label{fig:itefig}
\end{wrapfigure}
The set $\mathcal{K}_{ij}$ is designed to ensure that the nodes involved in the calculation of information distances are not too far, since estimating long distances accurately requires a large number of samples. The maximal cardinality of $\mathcal{K}_{ij}$ over all nodes $x_{i},x_{j}\in \mathcal{V}$ can be found, and we denote this as $N_{\tau}$, i.e., $|\mathcal{K}_{ij}|\leq N_{\tau}$.

The observed nodes are placed in the $0^{\mathrm{th}}$ layer. The hidden nodes introduced in $i^{\mathrm{th}}$ iteration are placed in $i^{\mathrm{th}}$ layer. The nodes in the $i^{\mathrm{th}}$ layer are in the active set $\Gamma^{i+1}$ 
in the $(i+1)^{\mathrm{st}}$ iteration, but nodes in $\Gamma^{i+1}$ can be nodes created in the $j^{\mathrm{th}}$ iteration, where $j<i$. For example, in Fig.~\ref{fig:itefig}, nodes $x_{12}$, $x_{14}$ and $x_{15}$ are created in the $1^{\mathrm{st}}$ iteration, and they 
are in $\Gamma^{2}$. Nodes $x_{1}$, $x_{2}$ and $x_{5}$ are also in $\Gamma^{2}$, which are observed nodes. Eqns.~\eqref{distup1} and~\eqref{distup2} imply that 
the estimation error in the $0^{\mathrm{th}}$ layer will propagate to the nodes in higher layers, and it is necessary to derive concentration results for the information distance related to the nodes in higher layers. To avoid repeating  complicated
expressions in the various concentration bounds to follow, we define the function
\begin{align} 
	f(x)\triangleq 2l_{\max}^{2}e^{-\frac{3n_{2}}{32\lambda\kappa n_{1}} x}+l_{\max}^{2}e^{-c\frac{n_{2}}{4\lambda^{2}\kappa^{2}}x^{2}}=: ae^{-wx}+be^{-ux^{2}},\nonumber
\end{align}
where $\lambda=2l_{\max}^{2}e^{\rho_{\max}/l_{\max}}/\delta_{\min}^{1/l_{\max}}$, $w=\frac{3n_{2}}{32\lambda\kappa n_{1}}$, $u=c\frac{n_{2}}{4\lambda^{2}\kappa^{2}}$, $a=2l_{\max}^{2}$ and $b=l_{\max}^{2}$. To assess the proximity of the estimates $\hat{\mathrm{d}}(x_{i},x_{\mathrm{new}})$ in 
\eqref{distup1} and \eqref{distup2} to their nominal versions, we define
\begin{align}
	h^{(l)}(x)\triangleq s^{l}f(m^{l}x)=s^{l}\big(ae^{-wm^{l}x}+be^{-um^{2l}x^{2}}\big) \quad \text{for all}\quad l  \in \mathbb{N}\cup\{0\}.   \label{eqn:hl}
\end{align}
where $s=d_{\max}^{2}+2d_{\max}^{3}(1+2N_{\tau})$ and $ m=2/9$. The following proposition yields {\em recursive estimates} for the errors of the distances at various layers of the learned latent tree. 
\begin{proposition}\label{errorpropg}
	With Assumptions \ref{assupleng}--\ref{assupdist}, if we implement the truncated inner product to estimate the information distance among observed nodes and adopt \eqref{distup1} and \eqref{distup2} to estimate 
	the information distances related to newly introduced hidden nodes, then the information distance related to the hidden nodes $x_{\mathrm{new}}$ created in the $l^{\mathrm{th}}$ layer $\hat{\mathrm{d}}(x_{i},x_{\mathrm{new}})$ satisfies
	\begin{align}\label{expotail}
		\mathbb{P}\Big(\big|\hat{\mathrm{d}}(x_{i},x_{\mathrm{new}})-\mathrm{d}(x_{i},x_{\mathrm{new}})\big|>\varepsilon\Big)<h^{(l)}(\varepsilon)\quad  \mbox{for all}\quad x_{i}\in \Gamma^{l+1} \quad\mbox{and}\quad l \in \mathbb{N}\cup\{0\}. 
	\end{align}
\end{proposition}
We note that Proposition \ref{errorpropg} demonstrates that the coefficient of exponential terms in \eqref{expotail} grow exponentially with increasing layers (i.e., $m^{l}$ and $m^{2l}$ in \eqref{eqn:hl}), which requires a commensurately large number of samples to control 
the tail probabilities. %Equipped with this concentration results, we can state the sample complexity of \ac{rrg} as follows
\begin{theorem}\label{theo:rrgsamplecomp}
	Under Assumptions \ref{assupleng}--\ref{assupdist},   \ac{rrg}   learns the correct latent tree with probability  $1-\eta$ if
	\begin{align}\label{eqn:sampcomplex}
		n_{2}=\tilde{\Omega}\Big(\frac{l_{\max}^{4}e^{2\rho_{\max}/l_{\max}} \kappa^{2}}{\delta_{\min}^{2/l_{\max}}\rho_{\min}^{2}}\big(\frac{9}{2}\big)^{2L_{\mathrm{R}}}\log\frac{|\mathcal{V}_{\mathrm{obs}}|^{3}}{\eta}\Big)\quad \text{and}\quad n_{1}=O\Big(\frac{\sqrt{n_{2}}}{\log n_{2}}\Big),
	\end{align}
	where $L_{\mathrm{R}}$ is the number of iterations of \ac{rrg} needed to construct the tree.
\end{theorem}
Theorem \ref{theo:rrgsamplecomp} indicates that the number of clean samples $n_{2}$ required by \ac{rrg} to learn the correct structure  grows exponentially with the number of iterations $L_{\mathrm{R}}$. Specifically, for the full $m$-tree illustrated in Fig.~\ref{fig:comp},  $n_{2}$ is exponential in the depth of the tree with high probability for structure learning to succeed. The sample complexity of \ac{rrg} depends on $e^{2\rho_{\max}/l_{\max}}$, and the exponential relationship with 
$\rho_{\max}$ will be shown to be unavoidable in view of our impossibility result in Theorem~\ref{theo:converse}. Huang et al.~\cite[Lemma~7.2 ]{huang2020guaranteed} also derived a sample complexity result for learning latent trees but the algorithm is based on~\cite{anandkumar2011spectral} instead of \ac{rg}. \ac{rrg} is able to tolerate $n_{1}=O(\sqrt{n_{2}}/\log n_{2})$ corruptions. 
This  tolerance level originates from the properties of the truncated inner product; similar tolerances will also be seen for the sample complexities of  subsequent algorithms. We expect this is also the case for \cite{huang2020guaranteed}, which is based on \cite{anandkumar2011spectral}, though we have not shown this formally. In addition,  the sample complexity  is applicable to a wide class of graphical models that satisfies the Assumptions~\ref{assupleng} to~\ref{assupdist}, while the sample complexity result \cite[Theorem 11]{choi2011learning}, which hides the dependencies on the parameters,  only holds for a limited class of graphical models whose effective depths (the maximal length of paths between hidden nodes and their closest observed nodes) are bounded in $|\mathcal{V}_{\mathrm{obs}}|$.

\subsection{Robust Neighbor Joining and Spectral Neighbor Joining algorithms}\label{subsec:snjnj}
%In the seminal work of Saitou and Nei, they proposed the \ac{nj} algorithm to learn the structure of phylogenetic trees \cite{saitou1987neighbor}. 
The \ac{nj} algorithm \cite{saitou1987neighbor} also makes use of additive distances to identify the existence of hidden nodes. To robustify the \ac{nj} 
algorithm, we adopt robust estimates of information distances as the additive distances in the so-called \ac{rnj} algorithm. We first recap a  result by Atteson~\cite{atteson1999performance}. %If the estimates of information distances are sufficiently close to the true values, \ac{nj} will reconstruct the correct tree.
\begin{proposition}\label{prop:njsuff}
	If all the nodes have exactly two children, \ac{nj}  will output the correct latent tree if
	\begin{align}
		\max_{x_{i},x_{j}\in\mathcal{V}_{\mathrm{obs}}} \big|\hat{\mathrm{d}}(x_{i},x_{j})-\mathrm{d}(x_{i},x_{j})\big|\leq {\rho_{\min}}/{2}.
	\end{align}
\end{proposition}
Unlike \ac{rg}, \ac{nj} does not identify the parent relationship among nodes, so it is only applicable to binary trees in which each node has at most two children. 
\begin{theorem}\label{theo:rnjsamplecomp}
	If Assumptions \ref{assupleng}--\ref{assupdist} hold and all the nodes have exactly two children, \ac{rnj} constructs the correct latent tree with probability at least $1-\eta$ if
	\begin{align}
		n_{2}=\Omega\Big(\frac{l_{\max}^{4}e^{2\rho_{\max}/l_{\max}}\kappa^{2}}{\delta_{\min}^{2/l_{\max}}\rho_{\min}^{2}}\log\frac{|\mathcal{V}_{\mathrm{obs}}|^{2}}{\eta}\Big)\quad \text{and}\quad n_{1}=O\Big(\frac{\sqrt{n_{2}}}{\log n_{2}}\Big).
	\end{align}
\end{theorem}
Theorem \ref{theo:rnjsamplecomp} indicates that the sample complexity of \ac{rnj} grows as $\log |\mathcal{V}_{\mathrm{obs}}|$, which is much better than \ac{rrg}. Similarly to RRG, 
the sample complexity has an exponential dependence on $\rho_{\max}$. 

In recent years, several variants of \ac{nj} algorithm have been proposed. The additivity of information distances results in certain properties of the rank of the matrix $\mathbf{R}\in \mathbb{R}^{|\mathcal{V}_{\mathrm{obs}}|\times|\mathcal{V}_{\mathrm{obs}}|}$, where 
$\mathbf{R}(i,j)=\exp(-\mathrm{d}(x_{i},x_{j}))$ for all $x_{i},x_{j}\in \mathcal{V}_{\mathrm{obs}}$. Jaffe \emph{et al.} \cite{jaffe2021spectral} proposed \ac{snj} which utilizes the rank of $\mathbf{R}$ to deduce the 
sibling relationships among nodes. We robustify the \ac{snj} algorithm by implementing the robust estimation of information distances, as shown in Algorithm \ref{algo:rsnj}.

Although   \ac{snj}   was   designed for  discrete random variables, the additivity of the information distance proved in Proposition \ref{prop:add} guarantees the consistency of \ac{rsnj} for \ac{ggm}s with vector variables. 
A sufficient condition for \ac{rsnj} to learn the correct tree can be generalized from \cite{jaffe2021spectral}.
%\begin{minipage}{0.48\textwidth}
%\end{minipage}
\begin{proposition}\label{prop:snjsuff}
	If Assumptions \ref{assupleng}--\ref{assupdist} hold and all the nodes have exactly two children, a sufficient condition for \ac{rsnj} to recover 
	the correct tree from $\hat{\mathbf{R}}$ is
	\begin{align}
		\|\hat{\mathbf{R}}-\mathbf{R}\|_{2}\leq g(|\mathcal{V}_{\mathrm{obs}}|,\rho_{\min},\rho_{\max}),
	\end{align}
	where
	\begin{align}
		g(x,\rho_{\min},\rho_{\max})&=\left\{
			\begin{array}{lr}
			\frac{1}{2}(2e^{-\rho_{\max}})^{\log_{2}(x/2)}e^{-\rho_{\max}}(1-e^{-2\rho_{\min}}),&  \quad e^{-2\rho_{\max}}\leq 0.5\\
			e^{-3\rho_{\max}}(1-e^{-2\rho_{\min}}),& \quad e^{-2\rho_{\max}}> 0.5
			\end{array}
		\right. .\nonumber
	\end{align}
\end{proposition} 
Similar with \ac{rnj},  
\ac{rsnj} also does not identify the parent relationship between nodes, so it   only applies to binary trees. To state the next result succinctly, we assume that $\rho_{\max}\ge \frac{1}{2}\log 2$; this is the regime of interest because we consider large trees which implies that $\rho_{\max}$ is typically large.
\begin{theorem}\label{theo:rsnjsamplecomp}
	If Assumptions \ref{assupleng}--\ref{assupdist} hold, $\rho_{\max}\ge \frac{1}{2}\log 2$, and all the nodes have exactly two children, \ac{rsnj} reconstructs the correct latent tree with probability at least $1-\eta$ if
	\begin{align}
		n_{2}=\Omega\Big( \frac{l_{\max}^{4}e^{2\rho_{\max}(1/l_{\max}+\log_{2}(|\mathcal{V}_{\mathrm{obs}}|/2)+1)}\kappa^{2}  }{\delta_{\min}^{2/l_{\max}}e^{2\rho_{\min}}}\log\frac{|\mathcal{V}_{\mathrm{obs}}|^{2}}{\eta}\Big)\quad \text{and}\quad n_{1}=O\Big(\frac{\sqrt{n_{2}}}{\log n_{2}}\Big).
	\end{align}
\end{theorem}
%Theorem \ref{theo:rsnjsamplecomp} indicates that the sample complexity of \ac{rsnj} grows as $|\mathcal{V}_{\mathrm{obs}}|^{2}$ when $e^{-2\rho_{\max}}> 0.5$. Specifically, in the binary tree case, the sample complexity grows exponentially with the depth of the tree.
Theorem \ref{theo:rsnjsamplecomp} indicates that the sample complexity of \ac{rsnj} grows as $\mathrm{poly}(|\mathcal{V}_{\mathrm{obs}}|)$. Specifically, in the binary tree case, the sample complexity grows exponentially with the depth of the tree. Also, the dependence of sample complexity on $\rho_{\max}$ is exponential, i.e., $O\big(e^{2(1/l_{\max}+\log_{2}(|\mathcal{V}_{\mathrm{obs}}|/2)+1)\rho_{\max}}\big)$, but the coefficient of $\rho_{\max}$ is 
larger than those of \ac{rrg} and \ac{rnj}, which are $O\big(e^{2\rho_{\max}/l_{\max}}\big)$. Compared to the sample complexity of \ac{snj} in \cite{jaffe2021spectral}, the sample complexity of \ac{rsnj} has the same dependence on the number of observed 
nodes $|\mathcal{V}_{\mathrm{obs}}|$, which means that the robustification of \ac{snj} using the truncated inner product  is able to tolerate $O\big(\frac{\sqrt{n_{2}}}{\log n_{2}}\big)$ corruptions.

\subsection{Robust Chow-Liu Recursive Grouping}\label{subsec:rclrg}

%From Theorem \ref{theo:rrgsamplecomp}, we can see that the sample complexity of \ac{rrg} grows exponentially with $L_{\mathrm{R}}$, which is prohibitively large in practice. 
In this section, we show that the exponential dependence on $L_{\mathrm{R}}$ in Theorem \ref{theo:rrgsamplecomp} can be provably  mitigated with an accurate initialization of the structure. Different from \ac{rrg}, \ac{rclrg} takes Chow-Liu algorithm as the initialization stage, as shown in Algorithm \ref{algo:rclrg}. 
The Chow-Liu algorithm~\cite{chow1968approximating} learns the maximum likelihood estimate of the tree structure by finding the maximum weight spanning tree of the graph whose edge weights are the mutual information quantities 
between these variables. In the estimation of the hidden tree structure, instead of taking the mutual information as the weights, we find the \ac{mst} of the graph whose weights are information distances, i.e., 
\begin{align}
	\mathrm{MST}(\mathcal{V}_{\mathrm{obs}};\mathbf{D}):=\mathop{\arg\min}_{\mathbb{T} \in \mathcal{T}_{\mathcal{V}_{\mathrm{obs}}}} \ \ \sum_{(x_{i},x_{j})\in \mathbb{T}} \mathrm{d}(x_{i},x_{j}),
\end{align}
where $\mathcal{T}_{\mathcal{V}_{\mathrm{obs}}}$ is the set of all the trees with node set $\mathcal{V}_{\mathrm{obs}}$. To describe the process of finding the \ac{mst}, we recall the definition of the {\em surrogate node} from \cite{choi2011learning}.
%To describe the output of Chow-Liu algorithm on latent trees and to analyze the sample complexity of \ac{rclrg}, we recall the definition of \emph{surrogate node} from \cite{choi2011learning}.
\begin{definition}\label{def:surnode}
	Given the latent tree $\mathbb{T}=(\mathcal{V},\mathcal{E})$ and any node $x_{i} \in \mathcal{V}$, the surrogate node  \cite{choi2011learning} of $x_{i}$   is $\mathrm{Sg}(x_{i};\mathbb{T},\mathcal{V}_{\mathrm{obs}})=\mathop{\arg\min}_{x_{j}\in \mathcal{V}_{\mathrm{obs}}} \  \mathrm{d}(x_{i},x_{j}).$ % with respect to 
	%$\mathcal{V}$ is defined as
	%\begin{align}
	%	\mathrm{Sg}(x_{i};\mathbb{T},\mathcal{V}_{\mathrm{obs}})=\mathop{\arg\min}_{x_{j}\in \mathcal{V}_{\mathrm{obs}}} \ \ \mathrm{d}(x_{i},x_{j}).
	%\end{align}
\end{definition}
We introduce a new notion of distance that quantifies the sample complexity of \ac{rclrg}.
\begin{definition}\label{def:contdist}
	Given the latent tree $\mathbb{T}=(\mathcal{V},\mathcal{E})$ and any node $x_{i}\in\mathcal{V}$, the \emph{contrastive distance} of $x_{i}$ with respect to 
	$\mathcal{V}_{\mathrm{obs}}$ is defined as
	\begin{align}
		\mathrm{d_{ct}}(x_{i};\mathbb{T},\mathcal{V}_{\mathrm{obs}})=\mathop{\min}_{x_{j}\in \mathcal{V}_{\mathrm{obs}}\backslash\{\mathrm{Sg}(x_{i};\mathbb{T},\mathcal{V}_{\mathrm{obs}})\}} \ \ \mathrm{d}(x_{i},x_{j})-\mathop{\min}_{x_{j}\in \mathcal{V}_{\mathrm{obs}}} \ \ \mathrm{d}(x_{i},x_{j}).
	\end{align}
\end{definition}
Definitions \ref{def:surnode} and \ref{def:contdist} imply that the surrogate node $\mathrm{Sg}(x_{i};\mathbb{T},\mathcal{V}_{\mathrm{obs}})$ of any observed node $x_{i}$ is itself $x_{i}$, 
and its contrastive distance is the information distance between the closest observed node and itself. It is shown that the Chow-Liu tree $\mathrm{MST}(\mathcal{V}_{\mathrm{obs}};\mathbf{D})$ is equal 
to the tree where all the hidden nodes are contracted to their surrogate nodes \cite{choi2011learning}, so it will be difficult to identify the surrogate node of some node if its contrastive 
distance is small. Under this scenario, more accurate estimates of the information distances are required to construct the correct Chow-Liu tree $\mathrm{MST}(\mathcal{V}_{\mathrm{obs}};\mathbf{D})$.
\begin{proposition}\label{prop:corctmst}
	The Chow-Liu tree $\mathrm{MST}(\mathcal{V}_{\mathrm{obs}};\mathbf{\hat{D}})$ is constructed correctly if
	\begin{align}
		\big|\hat{\mathrm{d}}(x_{i},x_{j})-\mathrm{d}(x_{i},x_{j})\big|< {\Delta_{\mathrm{MST}}}/{2} \quad \text{for all}\quad x_{i},x_{j}\in \mathcal{V}_{\mathrm{obs}},
	\end{align}
	where $\Delta_{\mathrm{MST}}:=\mathop{\min}_{x_{j}\in \mathrm{Int}(\mathbb{T})} \, \mathrm{d_{ct}}(x_{j};\mathbb{T},\mathcal{V}_{\mathrm{obs}})$.
\end{proposition}
Hence, the contrastive distance describes the difficulty of learning the correct Chow-Liu tree.
\begin{theorem}\label{theo:clrgsamplcomp}
	With Assumptions \ref{assupleng}--\ref{assupdist}, \ac{rclrg} constructs the correct latent tree with probability 
	at least $1-\eta$ if
	\begin{align}
		\!\!\!\!\! n_{2}\!=\!\tilde{\Omega}\bigg(\!\max\Big\{\frac{1}{\rho_{\min}^{2}}\!\big(\frac{9}{2}\big)^{2L_{\mathrm{C}}},\frac{1}{\Delta_{\mathrm{MST}}^{2}}\Big\}\frac{l_{\max}^{4}e^{2\rho_{\max}/l_{\max}}\kappa^{2}}{\delta_{\min}^{2/l_{\max}}}\log\frac{|\mathcal{V}_{\mathrm{obs}}|^{3}}{\eta}\! \bigg)\;\;\mbox{and}\;\; n_{1}\!=\!O\Big(\frac{\sqrt{n_{2}}}{\log n_{2}}\Big),\!\!
	\end{align}
	where $L_{\mathrm{C}}$ is the maximum number of iterations of \ac{rrg} (over  each internal node of the constructed Chow-Liu tree) in \ac{rclrg} needed to construct the tree.
\end{theorem}
If we implement \ac{rclrg} with \emph{true} information distances, $L_{\mathrm{C}}\le \lceil \frac{1}{2}\mathrm{Deg}(\mathrm{MST}(\mathcal{V}_{\mathrm{obs}};\mathbf{\hat{D}})) - 1\rceil$. Theorem~\ref{theo:clrgsamplcomp} 
indicates that the sample complexity of \ac{rclrg} grows exponentially in $L_{\mathrm{C}}\ll L_{\mathrm{R}}$. %The exponential dependence of $\rho_{\max}$ is also seen in the sample complexity of \ac{rclrg}. 
Compared with   
\cite[Theorem 12]{choi2011learning}, the sample complexity of \ac{rclrg} in Theorem \ref{theo:clrgsamplcomp} is applicable to a wide class of graphical models that satisfy Assumptions \ref{assupleng} to \ref{assupdist}, while the \cite[Theorem 12]{choi2011learning}  requires the assumption that the effective depths of latent trees are {\em bounded} in $|\mathcal{V}_{\mathrm{obs}}|$, which is rather restrictive.

\subsection{Comparison of robust latent tree learning algorithms}\label{subsec:comp}
Since the sample complexities of \ac{rrg}, \ac{rclrg}, \ac{rsnj} and \ac{rnj} depend on different parameters  and different structures of the underlying graphs, it is instructive to compare the sample complexities of these algorithms on some representative tree structures. These trees are illustrated in Fig.~\ref{fig:comp}.  \ac{rsnj} and \ac{rnj} are not able to identify the parent relationship among nodes, so they are only applicable to trees whose maximal degrees are no larger that $3$, 
including the double-binary tree and the \ac{hmm}. In particular, \ac{rnj} and \ac{rsnj} are not applicable to the full $m$-tree $(\text{for }m\geq 3)$ and the double star. %\blue{The parameter $t$ in the table scales as $O(\frac{1}{l_{\max}}+\log |\mathcal{V}_{\mathrm{obs}}|)$}. 
Derivations  and  more detailed discussions of the sample complexities are deferred to Appendix~\ref{app:table}.  
\begin{table}[H]
	\scriptsize
	\centering
	%\resizebox{\textwidth}{15mm}{
	\begin{tabular}{|c|c|c|c|c|}
	\hline
	\diagbox{Tree}{$n_{2}$}{Algorithm}&\ac{rrg}&\ac{rclrg}&\ac{rsnj}&\ac{rnj}\\
	\hline
	Double-binary tree&$ O\big(\psi(\frac{9}{2})^{\mathrm{Diam}(\mathbb{T})}\big)$&$O\big(\psi(\frac{9}{2})^{\frac{1}{2}\mathrm{Diam}(\mathbb{T})}\big)$&$O\big(e^{2t\rho_{\max}}\mathrm{Diam}(\mathbb{T})\big)$&$O\big(\psi\mathrm{Diam}(\mathbb{T})\big)$\\
	\hline
	\ac{hmm}&$O\big(\psi(\frac{9}{2})^{\mathrm{Diam}(\mathbb{T})}\big)$&$O\big(\psi\log\mathrm{Diam}(\mathbb{T})\big)$&$O\big(e^{2t\rho_{\max}}\log \mathrm{Diam}(\mathbb{T})$&$O\big(\psi\log\mathrm{Diam}(\mathbb{T})\big)$\\
	\hline
	Full $m$-tree&$O\big(\psi(\frac{9}{2})^{\mathrm{Diam}(\mathbb{T})}\big) $&$O\big(\psi\mathrm{Diam}(\mathbb{T})\big)$& N.A.& N.A.\\
	\hline
	Double star&$O(\psi\log d_{\max})$&$O\big(\psi\log d_{\max}\big)$&N.A.&N.A.\\
	\hline
	\end{tabular}
	%}
	\vspace{.1in}
	\caption{The sample complexities of \ac{rrg}, \ac{rclrg}, \ac{rsnj} and \ac{rnj} on the double-binary tree, the \ac{hmm}, the full $m$-tree and the double star. We set $\psi:=e^{2\rho_{\max}/l_{\max}}$ and $t=O(l_{\max}^{-1}+\log |\mathcal{V}_{\mathrm{obs}}|)$.}
	\label{table:samplecompare}
\end{table}
 
\subsection{Experimental results}\label{sec:simu}

We present simulation results to demonstrate the efficacy of the  robustified  algorithms. Samples are generated from a \ac{hmm} with $l_{\max}=3$ and $\mathrm{Diam}(\mathbb{T})=80$.  The Robinson-Foulds distance~\cite{robinson1981comparison} between the true   and  estimated trees is adopted to measure the performances of the  algorithms. For the implementations of \ac{clrg} and \ac{rg}, we use the code from~\cite{choi2011learning}. Other settings and more extensive experiments are given in Appendix~\ref{app:num}. %The information distance between neighboring nodes is $0.24$, which implies that $\rho_{\min}=0.24$ and $\rho_{\max}=0.23\mathrm{Diam}(\mathbb{T})=19.2$.

We consider three corruption patterns here. (i)  {\em Uniform corruptions} are independent additive noises in $[-2A,2A]$; (ii) {\em Constant magnitude corruptions} are also independent additive noises but taking values in $\{-A,+A\}$ with probability $0.5$. These two types of noises  are distributed randomly in   $\mathbf{X}_1^n$; (iii) {\em \ac{hmm} corruptions} are generated by a \ac{hmm} which has the same structure as the original \ac{hmm}  but has different parameters. They replace the entries in $\mathbf{X}_1^n$ with   samples generated by the variables in the same positions. In our simulations, $A$ is set to $60$, and the number of  corruptions $n_{1}$ is $100$.
\begin{figure}[H]
\centering
\subfigure[Uniform corruptions]{
\begin{minipage}[t]{0.33\linewidth}
\centering
\includegraphics[width=1.9in]{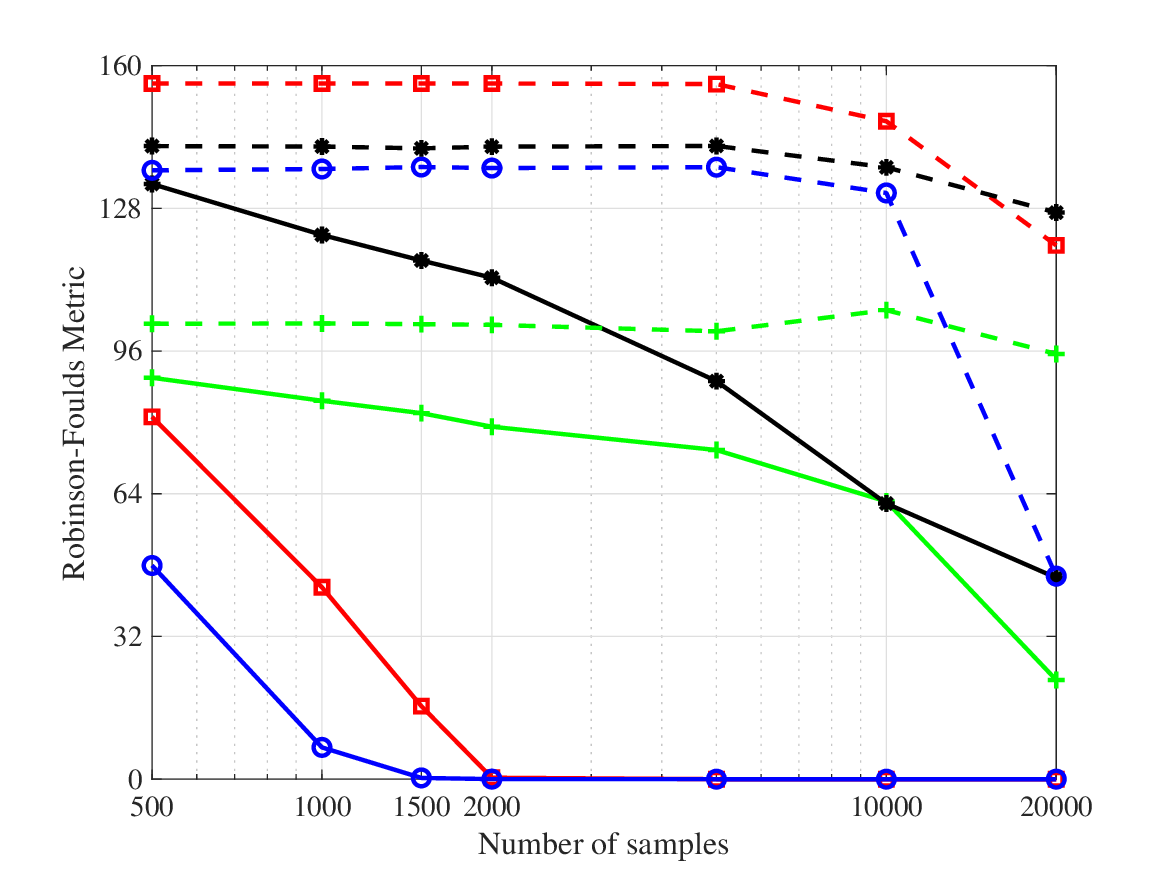}
%\caption{fig2}
\end{minipage}%
}%
\subfigure[Constant magnitude corruptions]{
\begin{minipage}[t]{0.33\linewidth}
\centering
\includegraphics[width=1.9in]{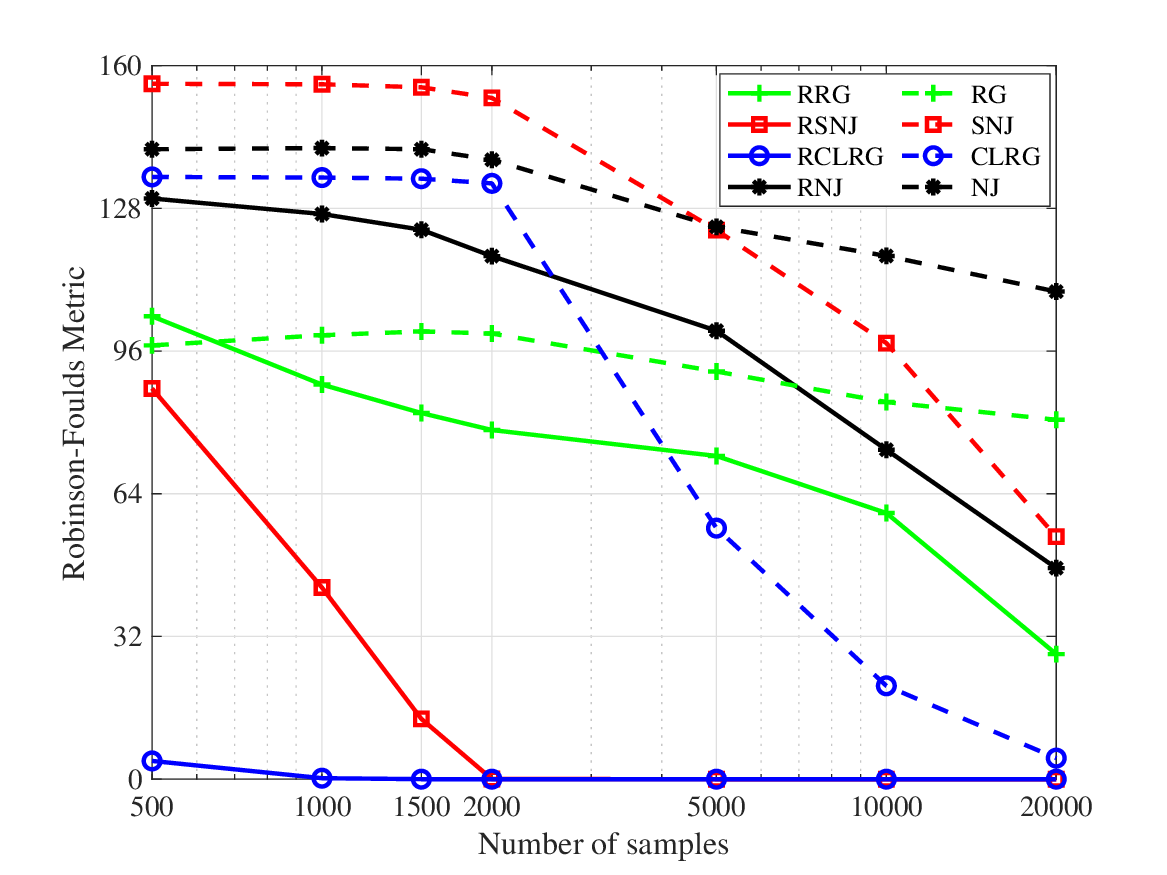}
%\caption{fig1}
\end{minipage}%
}%
\subfigure[\ac{hmm} corruptions]{
\begin{minipage}[t]{0.33\linewidth}
\centering
\includegraphics[width=1.9in]{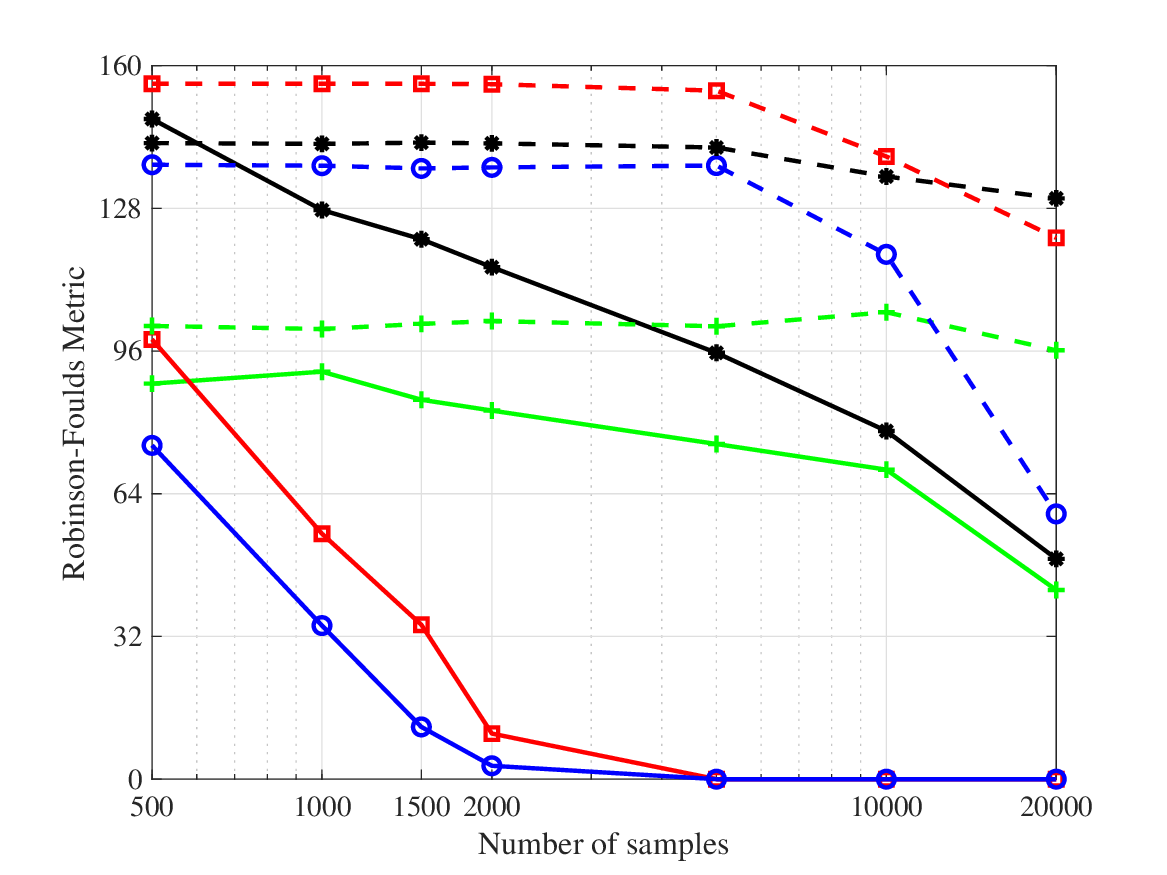}
%\caption{fig2}
\end{minipage}
}%
\centering
\caption{Robinson-Foulds distances of robustified and original  algorithms averaged over $100$ trials}
\label{fig:simu}
\end{figure}

Fig.~\ref{fig:simu} (error bars  are in Appendix~\ref{app:stddev}) demonstrates the  superiority of \ac{rclrg} in learning \ac{hmm}s compared to other algorithms. %Compared with \ac{rrg}, Chow-Liu initialization greatly reduces the estimation error in the following \ac{rrg} procedures. 
The robustified   algorithms also result in smaller estimation  errors  (Robinson-Foulds distances) 
compared to their unrobustified counterparts in  presence of   corruptions.
%of corruptions, which verifies that the truncated inner product is effective against arbitrary corruptions, even the corruptions are generated from another graphical model.

%\section{First known  impossibility result for learning latent trees}\label{subsec:converse}

\section{Impossibility result}\label{subsec:converse}
%We introduce the first known impossibility result of the number of samples of learning latent trees.
\begin{definition}
	Given a triple $(|\mathcal{V}_{\mathrm{obs}}|,\rho_{\max},l_{\max})$, the set $\mathcal{T}(|\mathcal{V}_{\mathrm{obs}}|,\rho_{\max},l_{\max})$ consists of all multivariate Gaussian distributions $\mathcal{N}(\mathbf{0},\mathbf{\Sigma})$ such that: (1) The 
	underlying graph $\mathbb{T}=(\mathcal{V},\mathcal{E})$ is a tree $\mathbb{T}\in \mathcal{T}_{\geq 3}$, and the size of the set of observed nodes is $|\mathcal{V}_{\mathrm{obs}}|$. (2) The distribution $\mathcal{N}(\mathbf{0},\mathbf{\Sigma})$ satisfies Assumptions~\ref{assupleng} and~\ref{assupdist} 
	with parameters $l_{\max}$ and $\rho_{\max}$.
\end{definition}

For the given class of graphical models $\mathcal{T}(|\mathcal{V}_{\mathrm{obs}}|,\rho_{\max},l_{\max})$, nature chooses some parameter $\theta=\mathbf{\Sigma}$ and generates $n$ i.i.d.\ samples $\mathbf{X}_{1}^{n}$ from $\mathbb{P}_{\theta}$. The 
goal of the statistician is to use the observations $\mathbf{X}_{1}^{n}$ to learn the underlying graph $\mathbb{T}$, which entails the design of a {\em decoder} $\phi:\mathbb{R}^{n\times |\mathcal{V}_{\mathrm{obs}}|l_{\max} } \rightarrow  \mathcal{T}_{ |\mathcal{V}_{\mathrm{obs}}|}$, 
where $\mathcal{T}_{ |\mathcal{V}_{\mathrm{obs}}|}$ is the set of latent trees whose size of the observed node set is $|\mathcal{V}_{\mathrm{obs}}|$. %The associated risk of decoder $\phi$ is zero-one loss $\mathbb{I}[\phi(\mathbf{X}_{1}^{n})\neq \mathbb{T}]$.

\begin{theorem}\label{theo:converse}
	Consider the class of graphical models $\mathcal{T}(|\mathcal{V}_{\mathrm{obs}}|,\rho_{\max},l_{\max})$, where $|\mathcal{V}_{\mathrm{obs}}|\geq 3$.   If there exists a graph decoder learns from $n$ i.i.d.\ samples such that
	\begin{align}
		\max_{\theta(\mathbb{T})\in \mathcal{T}(|\mathcal{V}_{\mathrm{obs}}|,\rho_{\max},l_{\max})} \mathbb{P}_{\theta(\mathbb{T})}(\phi(\mathbf{X}_{1}^{n})\neq \mathbb{T})< \delta,
	\end{align}
	then 	(as $\rho_{\max}\rightarrow \infty$ and $|\mathcal{V}_{\mathrm{obs}}|\rightarrow \infty$),
	\begin{align}
		n=\max\big\{\Omega\big((1-\delta)e^{\frac{\rho_{\max}}{\lfloor\log_{3} |\mathcal{V}_{\mathrm{obs}}| \rfloor l_{\max}}} \log |\mathcal{V}_{\mathrm{obs}}|\big),\Omega\big((1-\delta)e^{\frac{2\rho_{\max}}{3l_{\max}}}\big)\big\}. \label{eqn:n_bound}
	\end{align}
\end{theorem}
Theorem~\ref{theo:converse} implies that the optimal sample complexity grows as $\Omega (\log|\mathcal{V}_{\mathrm{obs}}|)$ as $|\mathcal{V}_{\mathrm{obs}}|$ grows. To prove this theorem, we construct several classes of Gaussian latent trees parametrized as linear dynamical systems (see Appendix~\ref{app:converse}) and apply the ubiquitous Fano technique to derive the desired impossibility result. Table~\ref{table:samplecompare} indicates that the sample complexity of \ac{rclrg} when the underlying latent tree is a  full $m$-tree  (for $m\geq3$) or a \ac{hmm}  is optimal in the dependence on $|\mathcal{V}_{\mathrm{obs}}|$. The sample complexity of \ac{rnj} is also optimal in  $|\mathcal{V}_{\mathrm{obs}}|$ for double binary trees and \ac{hmm}s. In contrast, the derived sample complexities  of \ac{rrg} and \ac{rsnj} are suboptimal in relation  to Theorem~\ref{theo:converse}. However, one caveat of our analyses of the latent tree learning algorithms in Section~\ref{sec:robustifying}  is that we are not claiming that they are the best possible for the given algorithm; there may be room for improvement.

%The comparison between \ac{rrg} and \ac{rclrg} also corroborates the effectiveness of Chow-Liu initialization procedure in \ac{rclrg}. 
  
 When the maximum information distance $\rho_{\max}$ grows, Theorem~\ref{theo:converse} indicates that the optimal sample 
complexity grows as $\Omega (e^{\frac{2\rho_{\max}}{3l_{\max}}})$. Table \ref{table:samplecompare} shows the sample complexities of \ac{rrg}, \ac{rclrg} and \ac{rnj} grow as $O(e^{2\frac{\rho_{\max}}{l_{\max}}})$, which has the alike dependence as the impossibility result. However, the sample complexity of \ac{rsnj} grows as $O\big(e^{2t\rho_{\max}}\big)$, 
which is larger (looser) than that prescribed by Theorem~\ref{theo:converse}.

\section{Conclusions and future works}
In this paper, we first derived the more refined sample complexities of \ac{rg} and \ac{clrg}. The effectiveness of \ac{clrg} was observed to be due to the reduction in the effective length that the  error propagates, i.e., from  $L_{\mathrm{R}}$ to $L_{\mathrm{C}}\ll L_{\mathrm{R}}$. Second, to combat potential adversarial corruptions in the data matrix, we robustified \ac{rg}, \ac{clrg}, \ac{nj} and \ac{snj} by adopting the truncated inner product technique. The derived sample complexity results showed that all the common latent tree learning algorithms can tolerate level-$O\big(\frac{\sqrt{n_{2}}}{\log n_{2}}\big)$ arbitrary corruptions. The varying efficacies of these robustified algorithms were then corroborated through extensive simulations with different types of corruptions and on different graphs. Finally, we derived the first known instance-dependent impossibility result for learning latent trees. The optimalities of \ac{rclrg} and \ac{rnj} in their dependencies on $|\mathcal{V}_{\mathrm{obs}}|$ were also discussed in the context of   various  latent tree structures.

There are several promising avenues for future research. First, the design and   analysis of the initialization process of \ac{clrg} can be further improved. The correctness of \ac{clrg} relies only  on the fact that if a hidden node is contracted to an observed node, then {\em all} the hidden nodes on the path between the hidden node and the observed nodes are  contracted to the {\em same} observed node. One can conceive of   a more general initialization algorithm other than that using the \ac{mst} of the weighted graph with weights being the  information distances. Second, the analysis of \ac{rg} can be tightened with more sophisticated concentration bounds. In particular, the exponential behavior of the sample complexity of \ac{rg} can also refined by performing a more careful analysis of the error propagation through the learned tree.

\paragraph{Acknowledgements}
We would like to thank the NeurIPS reviewers for their valuable and detailed reviews. This work is supported by a National University of Singapore (NUS) President's Graduate Fellowship, Singapore National Research Foundation (NRF) Fellowship (R-263-000-D02-281), a Singapore Ministry of Education (MoE) AcRF Tier 1 Grant (R-263-000-E80-114), and a Singapore MoE AcRF Tier 2 Grant (R-263-000-C83-112).

\bibliographystyle{unsrt}%abbrvnat
\bibliography{TGroup}

\clearpage
\appendix

\begin{center}
{\Large {\bf Supplementary materials for the NeurIPS 2021 submission \\ ``Robustifying Algorithms of Learning Latent Trees with Vector Variables''}}
\end{center}

\setcounter{equation}{0}
\counterwithin*{equation}{section}

\renewcommand{\theequation}{\thesection.\arabic{equation}}
\section{Illustrations of corruption patterns in Section~\ref{subsec:sysmodel}}

\begin{figure}[H]
	\centering\includegraphics[width=1\columnwidth,draft=false]{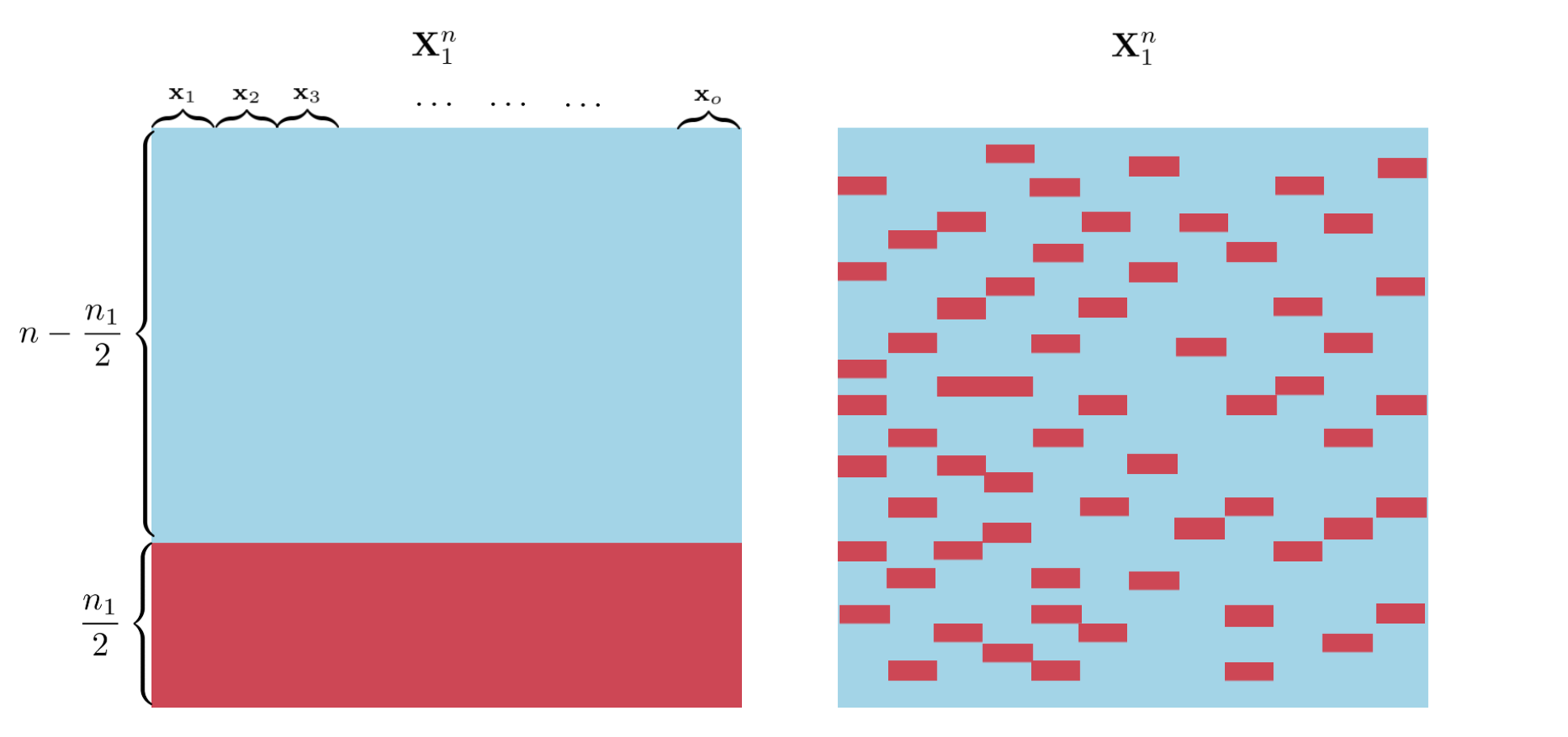}
	\caption{The left figure shows the corruption pattern that corrupted terms lie {\em in the same rows}. This corruption patterm is known as {\em outliers}. The right figure shows an {\em arbitrary corruption pattern} where corrupted entries in each column can be in {\em any} 
	$n_{1}/2$ rows. }
	\label{corrup_pic}
\end{figure}

\section{Illustrations of active sets defined in Section~\ref{subsec:rrg}}
\begin{figure}[H]
	\centering
	
	\subfigure[Illustration of $\Gamma^{1}$]{
	\begin{minipage}[t]{0.32\linewidth}
	\centering
	\includegraphics[width=1.8in]{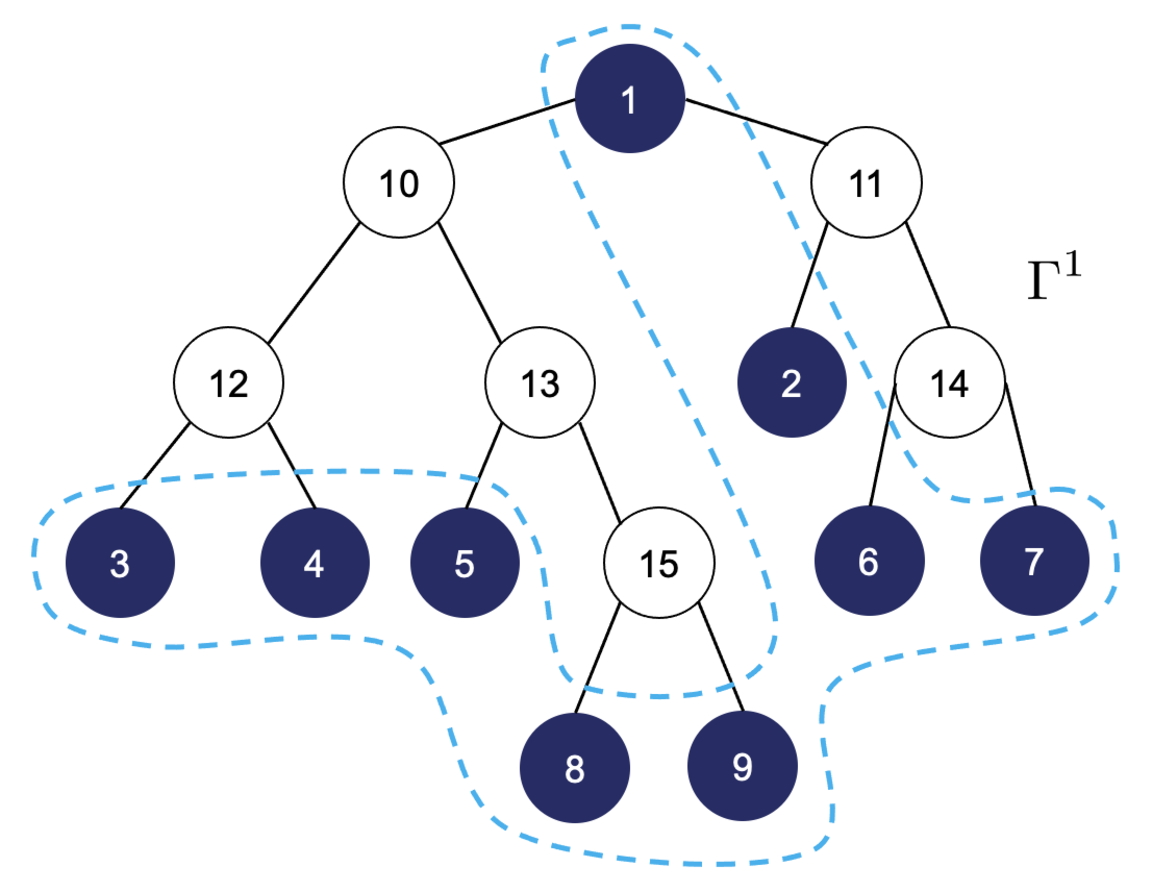}
	%\caption{fig1}
	\end{minipage}%
	}%
	\subfigure[Illustration of $\Gamma^{2}$]{
	\begin{minipage}[t]{0.32\linewidth}
	\centering
	\includegraphics[width=1.8in]{ite_num2.eps}
	%\caption{fig2}
	\end{minipage}%
	}%
	\subfigure[Illustration of $\Gamma^{3}$]{
	\begin{minipage}[t]{0.32\linewidth}
	\centering
	\includegraphics[width=1.8in]{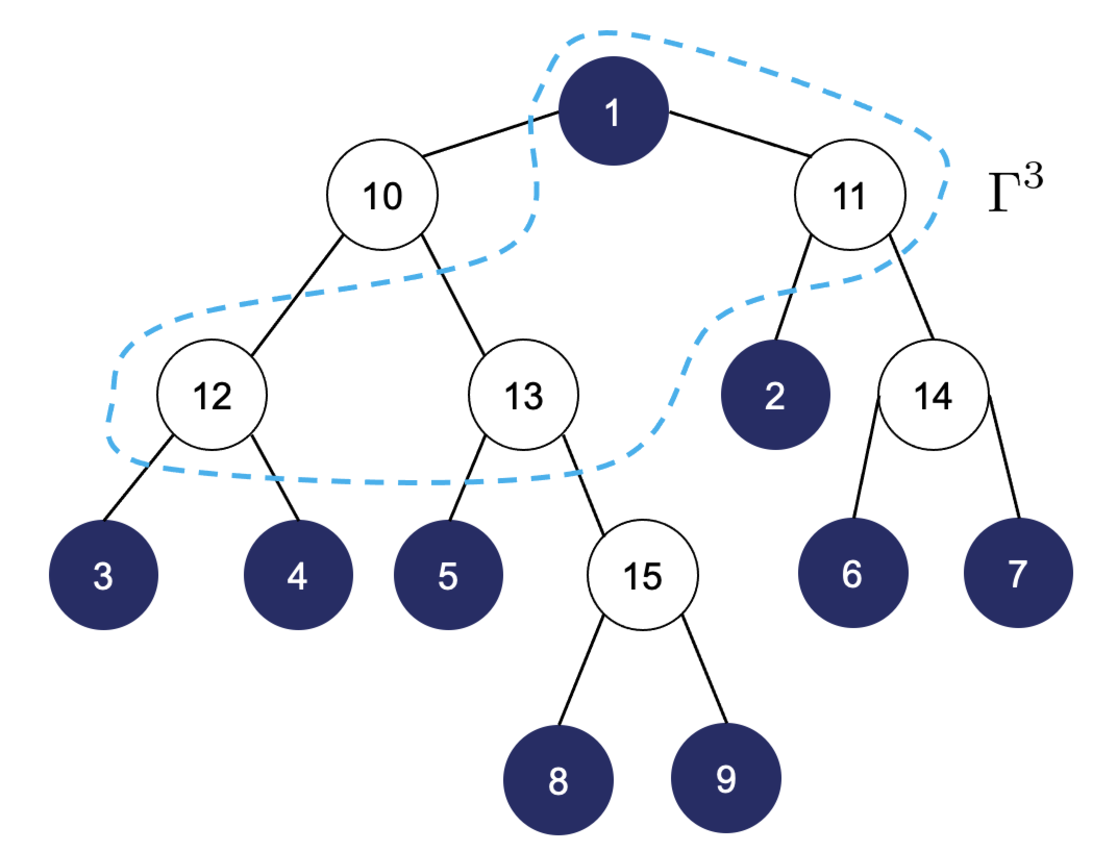}
	%\caption{fig2}
	\end{minipage}
	}%
	\label{fig:activeset}
	\centering
	\caption{Illustration of active sets.}
\end{figure}

\section{Pseudo-code of RRG in Section~\ref{subsec:rrg}}
\begin{algorithm}[H]
	\caption{\ac{rrg}}
	\textbf{Input:} Data matrix $\mathbf{X}$, corruption level $n_{1}$, threshold $\varepsilon$ \\
	\textbf{Output:} Adjacency matrix $\mathbf{A}$ \\
	\textbf{Procedure:}
	%\begin{multicols}{2}
	\begin{algorithmic}[1]\label{algo:rrg}
		\STATE Active set $\Gamma^{1} \leftarrow$ all the observed nodes\\
		\STATE Implement truncated inner product to compute $\hat{\mathrm{d}}(x_{i},x_{j}) \text{ for all } x_{i},x_{j}\in \mathcal{V}_{\mathrm{obs}}$.
		\WHILE{$|\Gamma^{i}|>2$}
			\STATE Update $\hat{\mathrm{d}}(x_{\mathrm{new}},x_{i}) \text{ for all } x_{i}\in \Gamma^{i}$ for all new hidden nodes. 
			\STATE Compute $\hat{\Phi}_{ijk}=\hat{\mathrm{d}}(x_{i},x_{k})-\hat{\mathrm{d}}(x_{j},x_{k})$ for all $x_{i}, x_{j}, x_{k}\in\Gamma^{i}$
			\FOR{all nodes $x_{i}$ and $x_{j}$ in $\Gamma^{i}$}
				\IF{$|\hat{\Phi}_{ijk}-\hat{\Phi}_{ijk^{\prime}}|<\varepsilon$ for all $x_{k},x_{k^{\prime}}\in\Gamma^{i}$}
					\IF{$|\hat{\Phi}_{ijk}-\hat{\mathrm{d}}(x_{i},x_{j})|<\varepsilon$ for all $x_{k}\in \Gamma^{i}$}
					\STATE $x_{j}$ is the parent of $x_{i}$.
					\STATE Eliminate $x_{i}$ from $\Gamma^{i}$
					\ELSE
					\STATE $x_{j}$ and $x_{i}$ are siblings.
					\STATE Create a hidden node $x_{\mathrm{new}}$ as the parent of $x_{j}$ and $x_{i}$
					\STATE Add $x_{\mathrm{new}}$ and eliminate $x_{j}$ and $x_{i}$ from $\Gamma^{i}$
					\ENDIF
				\ENDIF
			\ENDFOR
		\ENDWHILE
	\end{algorithmic}
	%\end{multicols}
\end{algorithm}

\section{Pseudo-code of RSNJ in Section~\ref{subsec:snjnj}}

\begin{algorithm}[H]
	\caption{\ac{rsnj}}
	\textbf{Input:} Data matrix $\mathbf{X}$, corruption level $n_{1}$ \\
	\textbf{Output:} Adjacent matrix $\mathbf{A}$ \\
	\textbf{Procedure:}
	\begin{algorithmic}[1]\label{algo:rsnj}
		\STATE Implement truncated inner product to compute $\hat{\mathrm{d}}(x_{i},x_{j}) \text{ for all } x_{i},x_{j}\in \mathcal{V}_{\mathrm{obs}}$. 
		\STATE Compute the symmetric affinity matrix $\hat{\mathbf{R}}$ as $\hat{\mathbf{R}}(i,j)=\exp(-\hat{\mathrm{d}}(x_{i},x_{j})) \text{ for all } x_{i},x_{j}\in \mathcal{V}_{\mathrm{obs}}$\\
		\STATE Set $B_{i}=\{x_{i}\} \text{ for all } x_{i}\in\Omega$ \\
		\STATE Compute the matrix $\mathbf{S}$ as $\hat{\mathbf{S}}(i,j)=\sigma_{2}(\hat{\mathbf{R}}^{B_{i}\cup B_{j}})$\\
		\WHILE{The number of $B_{i}$'s is larger than 3}
			\STATE Find $(\hat{i},\hat{j})=\mathop{\arg\min}_{i,j} \hat{\mathbf{S}}(i,j)$.\\
			\STATE Merge $B_{\hat{i}}$ and $B_{\hat{j}}$ as $B_{\hat{i}}=B_{\hat{i}}\cup B_{\hat{j}}$ and delete $B_{\hat{j}}$.\\
			\STATE Update $\hat{\mathbf{S}}(k,\hat{i})=\sigma_{2}(\hat{\mathbf{R}}^{B_{k}\cup B_{\hat{i}}})$.\\
		\ENDWHILE
	\end{algorithmic}
\end{algorithm}

\section{Pseudo-code of RCLRG in Section~\ref{subsec:rclrg}}

\begin{algorithm}[H]
	\caption{\ac{rclrg}}\label{algo:rclrg}
	\textbf{Input:} Data matrix $\mathbf{X}$, corruption level $n_{1}$, threshold $\varepsilon$ \\
	\textbf{Output:} Adjacency matrix $\mathbf{A}$ \\
	\textbf{Procedure:}
	\begin{algorithmic}[1]\label{algo:rclg}
		\STATE Construct a Chow-Liu tree with $\hat{\mathrm{d}}(x_{j},x_{k})$ for observed nodes $x_{j},x_{k}\in \mathcal{V}_{\mathrm{obs}}$\\
		\STATE Identify the set of internal nodes of the Chow-Liu tree\\
		\FOR{all internal nodes $x_{i}$ of the Chow-Liu tree}
			\STATE Implement \ac{rrg} algorithm on the closed neighborhood of $x_{i}$\\
			\STATE Replace the  closed neighborhood of $x_{i}$ with the output of \ac{rrg}\\
		\ENDFOR
	\end{algorithmic}
\end{algorithm}

\section{Illustrations of representative trees in Section~\ref{subsec:comp}}

\begin{figure}[H]
	\centering
	\subfigure[Double-binary tree]{
	\begin{minipage}[t]{0.45\linewidth}
	\centering
	\includegraphics[width=2.6in]{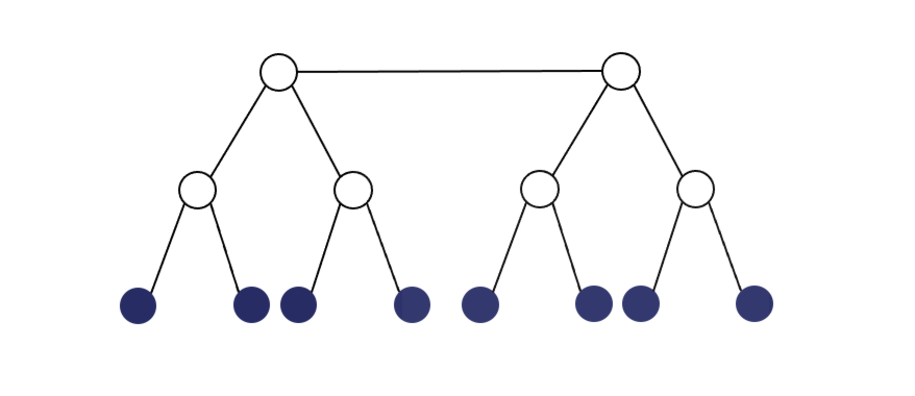}
	%\caption{fig1}
	\end{minipage}%
	}%
	\subfigure[\ac{hmm}]{
	\begin{minipage}[t]{0.45\linewidth}
	\centering
	\includegraphics[width=2.6in]{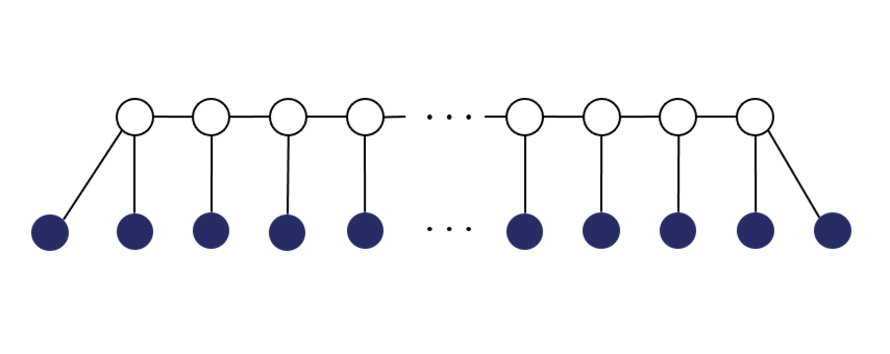}
	%\caption{fig2}
	\end{minipage}%
	}%

	\subfigure[Full $m$-tree, $m=3$]{
	\begin{minipage}[t]{0.45\linewidth}
	\centering
	\includegraphics[width=2.6in]{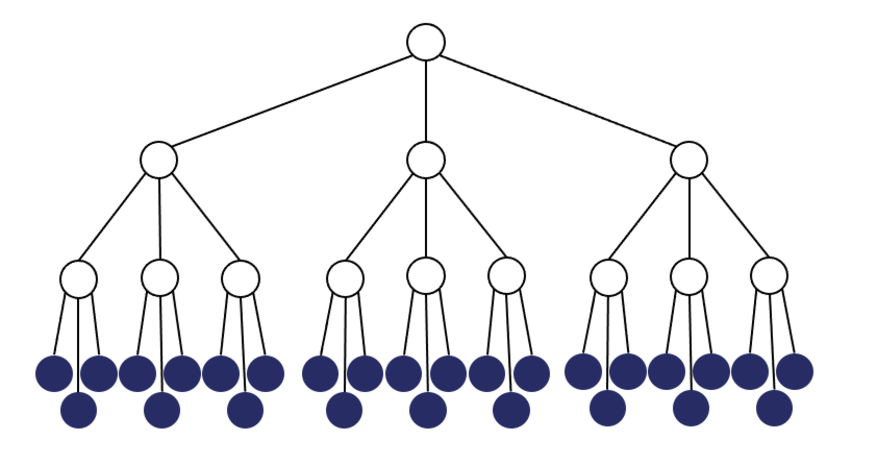}
	%\caption{fig2}
	\end{minipage}
	}%
	\subfigure[Double star]{
	\begin{minipage}[t]{0.45\linewidth}
	\centering
	\includegraphics[width=2.6in]{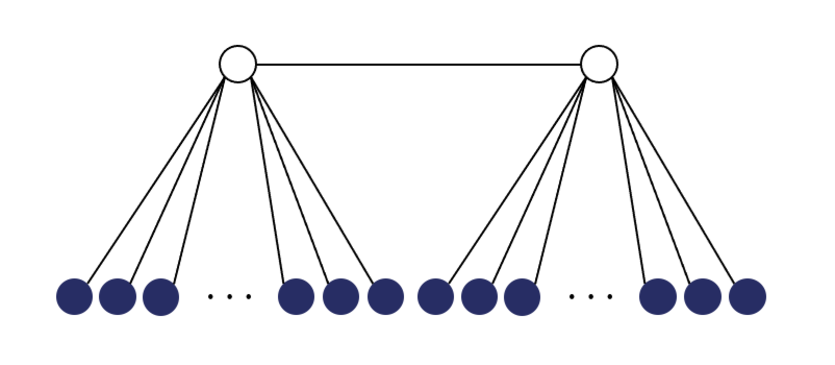}
	%\caption{fig2}
	\end{minipage}
	}%
\caption{Representative tree structures.}
\label{fig:comp}
\end{figure}

\section{Proofs of results in Section~\ref{subsec:infodistest}}\label{appendix:add_dist}

\begin{proof}[Proof of Proposition~\ref{prop:add}]
	For the sake of brevity, we prove the additivity property for paths of length~$2$. The proof for the general cases can be derived similarly. We consider 
	the case $x_{j}$ is on the path connected $x_{i}$ and $x_{k}$ and $x_{i},x_{j},x_{k}\in \mathcal{V}$.

	For any square matrix $\mathbf{A}\in \mathbb{R}^{n\times n}$, the determinant of $\mathbf{A}$ is denoted as $|\mathbf{A}| = \mathrm{det}(\mathbf{A})$.

	Then we can write information distance as
	\begin{align}\label{eqn:psedist}
		\mathrm{d}(x_{i},x_{k})=-\frac{1}{2}\log\big|\mathbf{\Sigma}_{ik}\mathbf{\Sigma}_{ik}^{\top}\big|+\frac{1}{4}\log\big|\mathbf{\Sigma}_{ii}\mathbf{\Sigma}_{ii}^{\top}\big|+\frac{1}{4}\log\big|\mathbf{\Sigma}_{kk}\mathbf{\Sigma}_{kk}^{\top}\big|
	\end{align}
	Note that $\mathbb{E}[\mathbf{x}_{i}|\mathbf{x}_{j}]=\mathbf{\Sigma}_{ij}\mathbf{\Sigma}_{jj}^{-1}\mathbf{x}_{j}$ and $\mathbf{\Sigma}_{ij}$ is of full rank by Assumption \ref{assupsing}, and
	\begin{align}
		\mathbf{A}_{i|j}=\mathbf{\Sigma}_{ij}\mathbf{\Sigma}_{jj}^{-1}
	\end{align}
	is also of full rank.

	Furthermore, we have
	\begin{align}
		\mathbf{\Sigma}_{ik}=\mathbf{A}_{i|j}\mathbf{\Sigma}_{jj}\mathbf{A}_{k|j}^{\top} \quad \text{and} \quad \mathbf{\Sigma}_{ik}\mathbf{\Sigma}_{ik}^{\top}=\mathbf{A}_{i|j}\mathbf{\Sigma}_{jj}\mathbf{A}_{k|j}^{\top}\mathbf{A}_{k|j}\mathbf{\Sigma}_{jj}\mathbf{A}_{i|j}^{\top}.
	\end{align}
	Then we have 
	\begin{align}\label{eqn:covdet1}
		\big|\mathbf{\Sigma}_{ik}\mathbf{\Sigma}_{ik}^{\top}\big|&=\big|\mathbf{A}_{i|j}\mathbf{\Sigma}_{jj}\mathbf{A}_{k|j}^{\top}\mathbf{A}_{k|j}\mathbf{\Sigma}_{jj}\mathbf{A}_{i|j}^{\top}\big| \\
		&=\big|\mathbf{A}_{i|j}^{\top}\mathbf{A}_{i|j}\mathbf{\Sigma}_{jj}\mathbf{A}_{k|j}^{\top}\mathbf{A}_{k|j}\mathbf{\Sigma}_{jj}\big|\\
		&=\frac{\big|\mathbf{\Sigma}_{jj}^{\top}\mathbf{A}_{i|j}^{\top}\mathbf{A}_{i|j}\mathbf{\Sigma}_{jj}\big|}{\big|\mathbf{\Sigma}_{jj}\big|}\frac{\big|\mathbf{\Sigma}_{jj}^{\top}\mathbf{A}_{k|j}^{\top}\mathbf{A}_{k|j}\mathbf{\Sigma}_{jj}\big|}{\big|\mathbf{\Sigma}_{jj}\big|}.
	\end{align}
	Furthermore, 
	\begin{align}\label{eqn:covdet2}
		\big|\mathbf{\Sigma}_{jj}^{\top}\mathbf{A}_{i|j}^{\top}\mathbf{A}_{i|j}\mathbf{\Sigma}_{jj}\big|&=\big|\mathbf{A}_{i|j}\mathbf{\Sigma}_{jj}\mathbf{\Sigma}_{jj}^{\top}\mathbf{A}_{i|j}^{\top}\big|=\big|\mathbf{\Sigma}_{ij}\mathbf{\Sigma}_{ij}^{\top}\big|,\\
		\big|\mathbf{\Sigma}_{jj}^{\top}\mathbf{A}_{k|j}^{\top}\mathbf{A}_{k|j}\mathbf{\Sigma}_{jj}\big|&=\big|\mathbf{\Sigma}_{kj}\mathbf{\Sigma}_{kj}^{\top}\big|.
	\end{align}

	Substituting \eqref{eqn:covdet1} and \eqref{eqn:covdet2} into \eqref{eqn:psedist}, we have
	\begin{align}
		\mathrm{d}(x_{i},x_{k})=&-\frac{1}{2}\log\big|\mathbf{\Sigma}_{ij}\mathbf{\Sigma}_{ij}^{\top}\big|+\frac{1}{4}\log\big|\mathbf{\Sigma}_{ii}\mathbf{\Sigma}_{ii}^{\top}\big|+\frac{1}{4}\log\big|\mathbf{\Sigma}_{jj}\mathbf{\Sigma}_{jj}^{\top}\big|\nonumber\\
		&\quad-\frac{1}{2}\log\big|\mathbf{\Sigma}_{kj}\mathbf{\Sigma}_{kj}^{\top}\big|+\frac{1}{4}\log\big|\mathbf{\Sigma}_{kk}\mathbf{\Sigma}_{kk}^{\top}\big|+\frac{1}{4}\log\big|\mathbf{\Sigma}_{jj}\mathbf{\Sigma}_{jj}^{\top}\big|\\
		&=\mathrm{d}(x_{i},x_{j})+\mathrm{d}(x_{j},x_{k}),
	\end{align}
	as desired.
\end{proof}
%\begin{definition}{\cite{vershynin2010introduction}}
%    For the sub-exponential random variable $X$, the sub-exponential norm %of it, denoted $\|X\|_{\psi_{1}}$ is defined as
%    \begin{align}
%        \|X\|_{\psi_{1}}=\sup_{p\geq 1} %p^{-1}\big(\mathbb{E}|X|^{p}\big)^{1/p}.
%    \end{align}
%\end{definition}
\begin{lemma}{(Bernstein-type inequality~\cite{vershynin2010introduction})}\label{subexp1}
	Let $X_{1},\ldots,X_{n}$ be $n$ centered sub-exponential random variables, and $K=\max_{1\le i\le n} \|X_{i}\|_{\psi_{1}}$, where $\|\cdot\|_{\psi_1}$ is  the sub-exponential norm  and is defined as 
	    \begin{align}
        \|X\|_{\psi_{1}}:=\sup_{p\geq 1}  p^{-1}\big(\mathbb{E}|X|^{p}\big)^{1/p}.
    \end{align}
	Then for every $a=(a_{1},\ldots,a_{n})\in \mathbb{R}^{n}$ and 
	every $t>0$, we have
	\begin{align}
		\mathbb{P}\bigg(\Big|\sum_{i=1}^{n} a_{i}X_{i}\Big|\geq t\bigg)\leq 2\exp\bigg[-c\min\Big\{\frac{t^{2}}{K^{2}\|a\|_{2}^{2}},\frac{t}{K\|a\|_{\infty}}\Big\}\bigg]
	\end{align}
\end{lemma}

\begin{lemma}\label{lemtrucate}
	Let the estimate of the covariance matrix $\mathbf{\Sigma}_{ij}$ based on the truncated inner product be $\hat{\mathbf{\Sigma}}_{ij}$. If $t_{2}<\kappa=\max\{\sigma_{\max}^{2},\rho_{\min}\}$, 
	we have
	\begin{align}
		\mathbb{P}\big(\|\hat{\mathbf{\Sigma}}_{ij}-\mathbf{\Sigma}_{ij}\|_{\infty,\infty}> t_{1}+t_{2}\big)\le 2l_{\max}^{2}e^{-\frac{3n_{2}}{16\kappa n_{1}} t_{1}}+l_{\max}^{2}e^{-c\frac{t_{2}^{2}n_{2}}{\kappa^{2}}}\quad \forall x_{i},x_{j}\in \mathcal{V}_{\mathrm{obs}}.
	\end{align}
\end{lemma}
\begin{proof}[Proof of Lemma~\ref{lemtrucate}]
	Let $I_{ij,1}^{st}$ be the set of indexes of the uncorrupted samples of $[\mathbf{x}_{i}]_{s}[\mathbf{x}_{j}]_{t}$. Without loss of generality, we assume that $|I_{ij,1}^{st}|=n_{2}$. Let $I_{ij,2}^{st}$ 
	and $I_{ij,3}^{st}$ be the sets of the indexes of truncated uncorrupted samples and the reserved corrupted samples, respectively.
	\begin{figure}[H]
		\centering\includegraphics[width=0.6\columnwidth,draft=false]{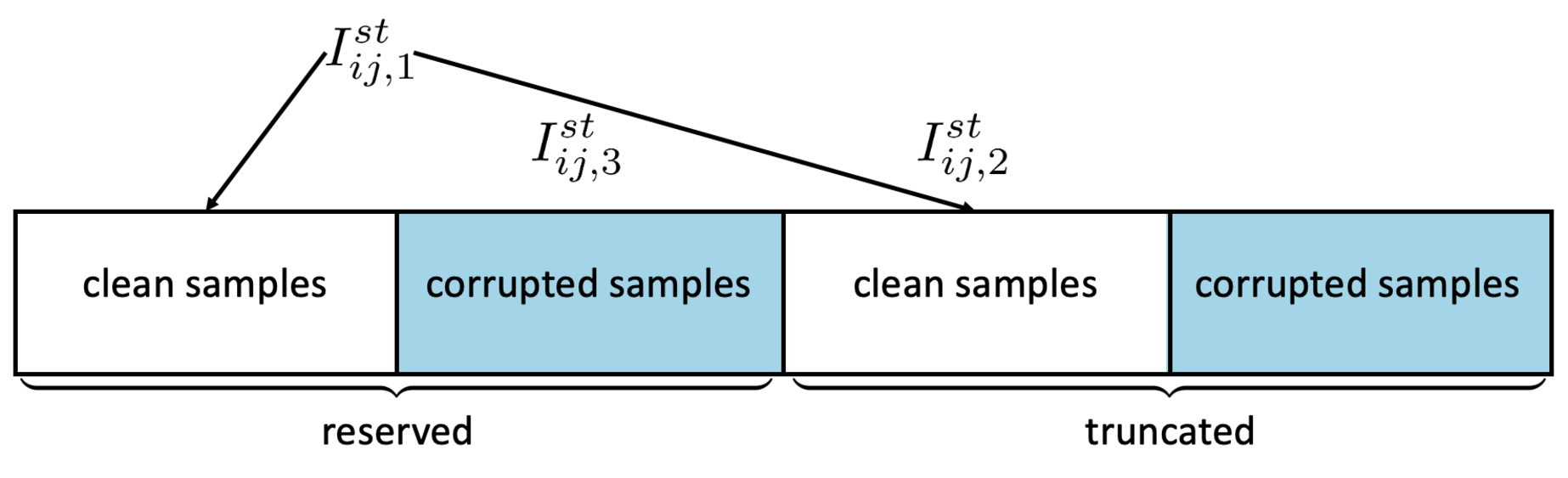}
		\caption{Illustration of the truncated inner product.}
		\label{fig:trucinnerprod}
	\end{figure}
	Then,
	\begin{align}
		[\hat{\mathbf{\Sigma}}_{ij}]_{st}=\frac{1}{n_{2}}\bigg(\sum_{m\in I_{ij,1}^{st}} [\mathbf{x}_{i}]_{s}^{(m)}[\mathbf{x}_{j}]_{t}^{(m)}-\sum_{m\in I_{ij,2}^{st}} [\mathbf{x}_{i}]_{s}^{(m)}[\mathbf{x}_{j}]_{t}^{(m)}+\sum_{m\in I_{ij,3}^{st}} [\mathbf{x}_{i}]_{s}^{(m)}[\mathbf{x}_{j}]_{t}^{(m)}\bigg)
	\end{align}
	The $(s,t)^{\mathrm{th}}$ entry of the error covariance matrix $\tilde{\mathbf{\Sigma}}_{ij}  = \hat{\mathbf{\Sigma}}_{ij}-\mathbf{\Sigma}_{ij}\in \mathbb{R}^{d\times d}$ is defined as 
	\begin{align}\label{truerror}
		[\tilde{\mathbf{\Sigma}}_{ij}]_{st}=\frac{1}{n_{2}}\bigg(-\sum_{m\in I_{ij,2}^{st}} [\mathbf{x}_{i}]_{s}^{(m)}[\mathbf{x}_{j}]_{t}^{(m)}+\sum_{m\in I_{ij,3}^{st}} [\mathbf{x}_{i}]_{s}^{(m)}[\mathbf{x}_{j}]_{t}^{(m)}\bigg).
	\end{align}
	From the definition of the truncated inner product, we can bound the right-hand side of \eqref{truerror} as
	\begin{align}\label{gaupro}
		\big|[\tilde{\mathbf{\Sigma}}_{ij}]_{st}\big|\leq \frac{2}{n_{2}}\sum_{m\in I_{ij,2}^{st}} \big|[\mathbf{x}_{i}]_{s}^{(m)}[\mathbf{x}_{j}]_{t}^{(m)}\big|.
	\end{align}
	Equipped with the expression of the moment-generating function of a chi-squared distribution, the moment-generating function of each term in the sum of \eqref{gaupro} can be upper bounded as
	\begin{align}
		\mathbb{E}\Big[e^{\lambda|[\mathbf{x}_{i}]_{s}^{(m)}[\mathbf{x}_{j}]_{t}^{(m)}|}\Big]&\leq \mathbb{E}\Big[e^{\lambda\frac{([\mathbf{x}_{i}]_{s}^{(m)})^{2}+([\mathbf{x}_{j}]_{t}^{(m)})^{2}}{2}}\Big]\\
		&\leq \sqrt{\mathbb{E}\big[e^{\lambda([\mathbf{x}_{i}]_{s}^{(m)})^{2}}\big]\mathbb{E}\big[e^{\lambda([\mathbf{x}_{j}]_{t}^{(m)})^{2}}\big]} \\
		&\leq \frac{1}{\sqrt{1-2\sigma_{\max}^{2}\lambda}}.
	\end{align}
	Using the power mean inequality, we have
	\begin{align}
		\Big(e^{\frac{2\lambda}{n_{2}}\sum_{m\in I_{ij,2}^{st}} |[\mathbf{x}_{i}]_{s}^{(m)}[\mathbf{x}_{j}]_{t}^{(m)}|}\Big)^{\frac{1}{|I_{ij,2}^{st}|}}&\leq \frac{\sum_{m\in I_{ij,2}^{st}}e^{\frac{2\lambda}{n_{2}}|[\mathbf{x}_{i}]_{s}^{(m)}[\mathbf{x}_{j}]_{t}^{(m)}| }}{|I_{ij,2}^{st}|}\\
		&\leq \bigg(\frac{\sum_{m\in I_{ij,2}^{st}}e^{\frac{2\lambda n_{1}}{n_{2}}|[\mathbf{x}_{i}]_{s}^{(m)}[\mathbf{x}_{j}]_{t}^{(m)}| }}{|I_{ij,2}^{st}|}\bigg)^{\frac{1}{n_{1}}}.
	\end{align}
	Thus, 
	\begin{align}
		\mathbb{E}\Big[e^{\lambda|[\tilde{\mathbf{\Sigma}}_{ij}]_{st}|}\Big]&\leq \mathbb{E}\Big[e^{\frac{2\lambda}{n_{2}}\sum_{m\in I_{ij,2}^{st}} \big|[\mathbf{x}_{i}]_{s}^{(m)}[\mathbf{x}_{j}]_{t}^{(m)}\big|}\Big]\leq \max_{m\in I_{ij,2}^{st}} \mathbb{E}\Big[e^{\frac{2\lambda n_{1}}{n_{2}}\big|[\mathbf{x}_{i}]_{s}^{(m)}[\mathbf{x}_{j}]_{t}^{(m)}\big|}\Big]  \\ 
		&\leq \frac{1}{\sqrt{1-\frac{4\sigma_{\max}^{2}n_{1}}{n_{2}}\lambda}}
	\end{align}
	and
	\begin{align}
		\mathbb{E}\Big[e^{\lambda\max_{s,t}\big|[\tilde{\mathbf{\Sigma}}_{ij}]_{st}\big|}\Big]&=\mathbb{E}\Big[\max_{s,t}e^{\lambda\big|[\tilde{\mathbf{\Sigma}}_{ij}]_{st}\big|}\Big]\leq l_{\max}^{2}\mathbb{E}\Big[e^{\lambda\big|[\tilde{\mathbf{\Sigma}}_{ij}]_{st}\big|}\Big] \leq \frac{l_{\max}^{2}}{\sqrt{1-\frac{4\sigma_{\max}^{2}n_{1}}{n_{2}}\lambda}}.
	\end{align}
	Thus,
	\begin{align}
		\mathbb{P}\big(\|\tilde{\mathbf{\Sigma}}_{ij}\|_{\infty,\infty}>t\big)&=\mathbb{P}\Big(\max_{s,t}\big|[\tilde{\mathbf{\Sigma}}_{ij}]_{st}\big|>t\Big)=\mathbb{P}\Big(e^{\lambda\max_{s,t}\big|[\tilde{\mathbf{\Sigma}}_{ij}]_{st}\big|}>e^{\lambda t}\Big)\\
		&\leq e^{-\lambda t}\mathbb{E}\Big[e^{\lambda\max_{s,t}\big|[\tilde{\mathbf{\Sigma}}_{ij}]_{st}\big|}\Big]\leq e^{-\lambda t}\frac{l_{\max}^{2}}{\sqrt{1-\frac{4\sigma_{\max}^{2}n_{1}}{n_{2}}\lambda}}.
	\end{align}
	Let $\lambda=\frac{3n_{2}}{16\sigma_{\max}^{2}n_{1}}$, then we have
	\begin{align}
		\mathbb{P}\big(\|\tilde{\mathbf{\Sigma}}_{ij}\|_{\infty,\infty}>t\big)\leq 2l_{\max}^{2}e^{-\frac{3n_{2}}{16\sigma_{\max}^{2}n_{1}} t}.
	\end{align}
	According to Lemma \ref{subexp1} (since the involved random variables are sub-exponential), we have
	\begin{align}
		\mathbb{P}\bigg(\Big|\frac{1}{n_{2}}\sum_{m\in I_{ij,1}^{st}} \big([\mathbf{x}_{i}]_{s}^{(m)}[\mathbf{x}_{j}]_{t}^{(m)}-[\mathbf{\Sigma}_{ij}]_{st}\big)\Big|>t\bigg)\leq \exp\bigg(-c\min\Big\{\frac{t^{2}n_{2}}{K^{2}},\frac{tn_{2}}{K}\Big\}\bigg),
	\end{align}
	where $K=\sigma_{\max}^{2}$.

	Thus, if $t<\kappa$, we have
	\begin{align}
		\mathbb{P}\big(\|\hat{\mathbf{\Sigma}}_{ij}-\mathbf{\Sigma}_{ij}\|_{\infty,\infty}> t_{1}+t_{2}\big)\le  2l_{\max}^{2}e^{-\frac{3n_{2}}{16\kappa n_{1}} t_{1}}+l_{\max}^{2}e^{-c\frac{t_{2}^{2}n_{2}}{\kappa^{2}}},
	\end{align}
	as desired.
\end{proof}

\begin{proof}[Proof of Proposition ~\ref{distconcen}]
	From the definition of the information distance, we have
	\begin{align}
		\mathrm{d}(x_{i},x_{j})=-\sum_{n=1}^{l_{\max}} \log \sigma_{n}\big(\mathbf{\Sigma}_{ij}\big)+\frac{1}{2}\log \det\big(\mathbf{\Sigma}_{ii}\big)+\frac{1}{2}\log \det\big(\mathbf{\Sigma}_{jj}\big).
	\end{align}
	According to the inequality $\|\mathbf{A}\|_{2}\leq \sqrt{\|\mathbf{A}\|_{1}\|\mathbf{A}\|_{\infty}}$ which holds for all   $\mathbf{A}\in \mathbb{R}^{n\times m}$ \cite{Gen:B13}, we have
	\begin{align}\label{sigbound}
		\big|\sigma_{k}(\hat{\mathbf{\Sigma}}_{ij})-\sigma_{k}(\mathbf{\Sigma}_{ij})\big|&\leq \|\hat{\mathbf{\Sigma}}_{ij}-\mathbf{\Sigma}_{ij}\|_{2}\\
		&\leq \sqrt{\|\hat{\mathbf{\Sigma}}_{ij}-\mathbf{\Sigma}_{ij}\|_{\infty}\|\hat{\mathbf{\Sigma}}_{ij}-\mathbf{\Sigma}_{ij}\|_{1}}\leq l_{\max}\|\hat{\mathbf{\Sigma}}_{ij}-\mathbf{\Sigma}_{ij}\|_{\infty,\infty}.
	\end{align}
	Using the triangle inequality, we arrive at
	\begin{align}
		\big|\hat{\mathrm{d}}(x_{i},x_{j})-\mathrm{d}(x_{i},x_{j})\big|\leq &\sum_{n=1}^{l_{\max}}\big|\log \sigma_{n}(\hat{\mathbf{\Sigma}}_{ij})-\log \sigma_{n}(\mathbf{\Sigma}_{ij})\big|+\frac{1}{2}\sum_{n=1}^{\mathrm{dim}(\mathbf{x}_{i})}\big|\log \sigma_{n}(\hat{\mathbf{\Sigma}}_{ii})-\log \sigma_{n}(\mathbf{\Sigma}_{ii})\big|\nonumber \\
		&+\frac{1}{2}\sum_{n=1}^{\mathrm{dim}(\mathbf{x}_{j})}\big|\log \sigma_{n}(\hat{\mathbf{\Sigma}}_{jj})-\log \sigma_{n}(\mathbf{\Sigma}_{jj})\big|.
	\end{align}
	Furthermore, since the singular value is lower bounded by $\gamma_{\min}$, using Taylor's theorem and \eqref{sigbound}, we obtain
	\begin{align}
		\big|\log \sigma_{n}(\hat{\mathbf{\Sigma}}_{ij})-\log \sigma_{n}(\mathbf{\Sigma}_{ij})\big|\leq \frac{1}{\gamma_{\min}} \big|\sigma_{n}(\hat{\mathbf{\Sigma}}_{ij})-\sigma_{n}(\mathbf{\Sigma}_{ij})\big|\leq \frac{l_{\max}}{\gamma_{\min}}\|\hat{\mathbf{\Sigma}}_{ij}-\mathbf{\Sigma}_{ij}\|_{\infty,\infty}.
	\end{align}
	Finally, 
	\begin{align}
		\big|\hat{\mathrm{d}}(x_{i},x_{j})-\mathrm{d}(x_{i},x_{j})\big|&\leq \Big(l_{\max}+\frac{\mathrm{dim}(\mathbf{x}_{i})+\mathrm{dim}(\mathbf{x}_{j})}{2}\Big)\frac{l_{\max}}{\gamma_{\min}}\|\hat{\mathbf{\Sigma}}_{ij}-\mathbf{\Sigma}_{ij}\|_{\infty,\infty}\nonumber\\
			&\leq \frac{2l_{\max}^{2}}{\gamma_{\min}}\|\hat{\mathbf{\Sigma}}_{ij}-\mathbf{\Sigma}_{ij}\|_{\infty,\infty}.
	\end{align}
	From Lemma \ref{lemtrucate}, the proposition is proved.
\end{proof}

\section{Proofs of results in Section~\ref{subsec:rrg}}

\begin{lemma}\label{optlem}
	Consider the optimization problem
	\begin{align}
		\mathscr{P}:\quad \max\limits_{ \{x_{i}\} } \quad &f(\mathbf{x})=\sum_{i=1}^{n} x_{i}(x_{i}-1) \notag \\
			{\rm s.t.} \quad & \sum_{i=1}^{N}x_{i}\leq N \quad  0\leq x_{i}\leq k \qquad i=1,\ldots,N.
	\end{align}
	Assume   $nk\geq N$. An optimal  solution is given by $x_{i}=k $ for all $i=1,\ldots,\lfloor\frac{N}{k}\rfloor$ and  $x_{\lfloor\frac{N}{k}\rfloor+1}=N-k\lfloor\frac{N}{k}\rfloor$, and $x_{i}=0$ for $i=\lfloor\frac{N}{k}\rfloor+2,\ldots,n$.
\end{lemma}
This lemma can be verified by direct calculation, and so we will omit the details.
\begin{proof}[Proof of Proposition~\ref{errorpropg}]
	We prove the proposition by induction.

	Proposition \ref{distconcen} and Eqn. \eqref{eq:simi} show that at the $0^{\mathrm{th}}$ layer \cite{stewart1998perturbation}
	\begin{align}
		\mathbb{P}(|\Delta_{ij}|>\varepsilon)<f(\varepsilon)=h^{(0)}(\varepsilon).
	\end{align}
	Now suppose that the distances related to the nodes created in the $(l-1)^{\mathrm{st}}$ iteration satisfy
	\begin{align}\label{eq:layer}
		\mathbb{P}\Big(\big|\hat{\mathrm{d}}(x_{i},x_{h})-\mathrm{d}(x_{i},x_{h})\big|>\varepsilon\Big)<h^{(l-1)}(\varepsilon).
	\end{align}
	Since $s>1$ and $m<1$, it is obvious that
	\begin{align}
		h^{(l)}(\varepsilon)\leq h^{(l+k)}(\varepsilon) \quad  \mbox{for all}\quad  l,k\in \mathbb{N} \quad\text{ and for all}\quad  \varepsilon>0.
	\end{align}
	Then we can deduce that
	\begin{align}
		\mathbb{P}(|\hat{\mathrm{d}}(x_{i},x_{j})-\mathrm{d}(x_{i},x_{j})|>\varepsilon)<h^{(l-1)}(\varepsilon) \quad \mbox{for all}\quad  x_{i},x_{j}\in \Gamma^{l}.
	\end{align}
	From the update equation of the distance in~\eqref{distup1}, we have
	\begin{align}
		\hat{\mathrm{d}}(x_{i},x_{h})&=\frac{1}{2(\big|\mathcal{C}(h)-1|\big)}\bigg(\sum_{j\in\mathcal{C}(h)}(\mathrm{d}(x_{i},x_{j})+\Delta_{ij} )+\frac{1}{|\mathcal{K}_{ij}|}\sum_{k\in\mathcal{K}_{ij}}(\Phi_{ijk}+\Delta_{ik}-\Delta_{jk})\bigg)
	\end{align}
	and
	\begin{align}
		\hat{\mathrm{d}}(x_{i},x_{h})&=\frac{1}{2(\big|\mathcal{C}(h)-1|\big)}\bigg(\sum_{j\in\mathcal{C}(h)}\Delta_{ij}+\frac{1}{|\mathcal{K}_{ij}|}\sum_{k\in\mathcal{K}_{ij}}(\Delta_{ik}-\Delta_{jk})\bigg)+\mathrm{d}(x_{i},x_{h}).
	\end{align}
	Using the union bound, we find that
	\begin{align}
		&\mathbb{P}\Big(\big|\hat{\mathrm{d}}(x_{i},x_{h})-\mathrm{d}(x_{i},x_{h})\big|>\varepsilon\Big)\nonumber\\
		&\leq \mathbb{P}\bigg(\bigcup_{j\in \mathcal{C}(h)}\Big\{\big|\Delta_{ij}+\frac{1}{|\mathcal{K}_{ij}|}\sum_{k\in\mathcal{K}_{ij}}(\Delta_{ik}-\Delta_{jk})\big|>2\varepsilon\Big\}\bigg) \\
		&\leq \sum_{j\in \mathcal{C}(h)} \mathbb{P}\bigg(\big|\Delta_{ij}+\frac{1}{|\mathcal{K}_{ij}|}\sum_{k\in\mathcal{K}_{ij}}(\Delta_{ik}-\Delta_{jk})\big|>2\varepsilon\bigg) \\
		&\leq \sum_{j\in \mathcal{C}(h)} \bigg[\mathbb{P}\Big(|\Delta_{ij}|>\frac{2}{3}\varepsilon\Big)+\sum_{k\in\mathcal{K}_{ij}}\mathbb{P}\Big( |\Delta_{ik}|>\frac{2}{3}\varepsilon \Big)+\mathbb{P}\Big( |\Delta_{jk}|>\frac{2}{3}\varepsilon \Big)\bigg].
	\end{align}
	The estimates of the distances related to the nodes in the $l^{\mathrm{th}}$ layer satisfy
	\begin{align}
		\mathbb{P}\Big(\big|\hat{\mathrm{d}}(x_{i},x_{h})-\mathrm{d}(x_{i},x_{h})\big|>\varepsilon\Big)&<|\mathcal{C}(h)|\big(1+2|\mathcal{K}_{ij}|\big)h^{(l-1)}\big(\frac{2}{3}\varepsilon\big)  \\
		&\leq d_{\max}(1+2N_{\tau})h^{(l-1)}\big(\frac{2}{3}\varepsilon\big).
	\end{align}
	Similarly, from \eqref{distup2}, we have
	\begin{align}
		&\mathbb{P}\Big(\big|\hat{\mathrm{d}}(x_{k},x_{h})-\mathrm{d}(x_{k},x_{h})\big|>\varepsilon\Big) \nonumber\\
		&\leq\left\{
			\begin{array}{lr}
			\sum_{i\in \mathcal{C}(h)}\mathbb{P}\big(|\Delta_{ik}|>\frac{1}{2}\varepsilon\big)+\mathbb{P}\big(|\Delta_{ih}|>\frac{1}{2}\varepsilon\big), & \text{if }k\in \mathcal{V}_{\mathrm{obs}} \\
			\sum_{(i,j)\in \mathcal{C}(h)\times \mathcal{C}(k)}\mathbb{P}\big(|\Delta_{ij}|>\frac{1}{3}\varepsilon\big)+\mathbb{P}\big(|\Delta_{ih}|>\frac{1}{3}\varepsilon\big)+\mathbb{P}\big(|\Delta_{jk}|>\frac{1}{3}\varepsilon\big), & \text{otherwise.}
			\end{array}
			\right.
	\end{align}
	Using the concentration bound at the $(l-1)^{\mathrm{st}}$ layer in inequality~\eqref{eq:layer}, we have
	\begin{align}
		\mathbb{P}\Big(\big|\hat{\mathrm{d}}(x_{k},x_{h})&-\mathrm{d}(x_{k},x_{h})\big|>\varepsilon\Big) \nonumber\\
		&\leq\left\{
			\begin{array}{lr}
			d_{\max}h^{(l-1)}(\frac{1}{2}\varepsilon)+d_{\max}^2(1+2N_{\tau})h^{(l-1)}(\frac{1}{3}\varepsilon), & \text{if }k\in \mathcal{V}_{\mathrm{obs}} \\
			d_{\max}^{2}h^{(l-1)}(\frac{2}{3}\varepsilon)+2d_{\max}^{3}(1+2N_{\tau})h^{(l-1)}(\frac{2}{9}\varepsilon), & \text{otherwise.}
			\end{array}
			\right.
	\end{align}
	Summarizing the above three concentration bounds, we have that for the nodes at the $l^{\mathrm{th}}$ layer, estimates of the information distances (based on the truncated inner product) satisfy
	\begin{align}
		\mathbb{P}\Big(\big|\hat{\mathrm{d}}(x_{k},x_{h})&-\mathrm{d}(x_{k},x_{h})\big|>\varepsilon\Big)< \big[d_{\max}^{2}+2d_{\max}^{3}(1+2N_{\tau})\big]h^{(l-1)}\Big(\frac{2}{9}\varepsilon\Big)=h^{(l)}(\varepsilon).
	\end{align}
\end{proof}

\begin{proposition}\label{propactset}
	The cardinalities of the active sets in $l^{\mathrm{th}}$ and $(l+1)^{\mathrm{st}}$ iterations admit following relationship
	\begin{align}
		\frac{|\Gamma^{l}|}{d_{\max}}\leq|\Gamma^{l+1}|\leq |\Gamma^{l}|-2.
	\end{align}
\end{proposition}
\begin{proof}[Proof of Proposition ~\ref{propactset}]
	Note that at the $l^{\mathrm{th}}$ iteration, the number of families is $|\Gamma^{l+1}|$, and thus we have
	\begin{align}
		\sum_{i=1}^{|\Gamma^{l+1}|} n_{i}=|\Gamma^{l}|,
	\end{align}
	where $n_{i}$ is the number of nodes in $\Gamma^{l}$ in each family. Since $1\leq n_{i}\leq d_{\max}$, we have $\frac{\Gamma^{l}}{d_{\max}}\leq|\Gamma^{l+1}|$.

	\begin{figure}[H]
		\centering\includegraphics[width=0.5\columnwidth,draft=false]{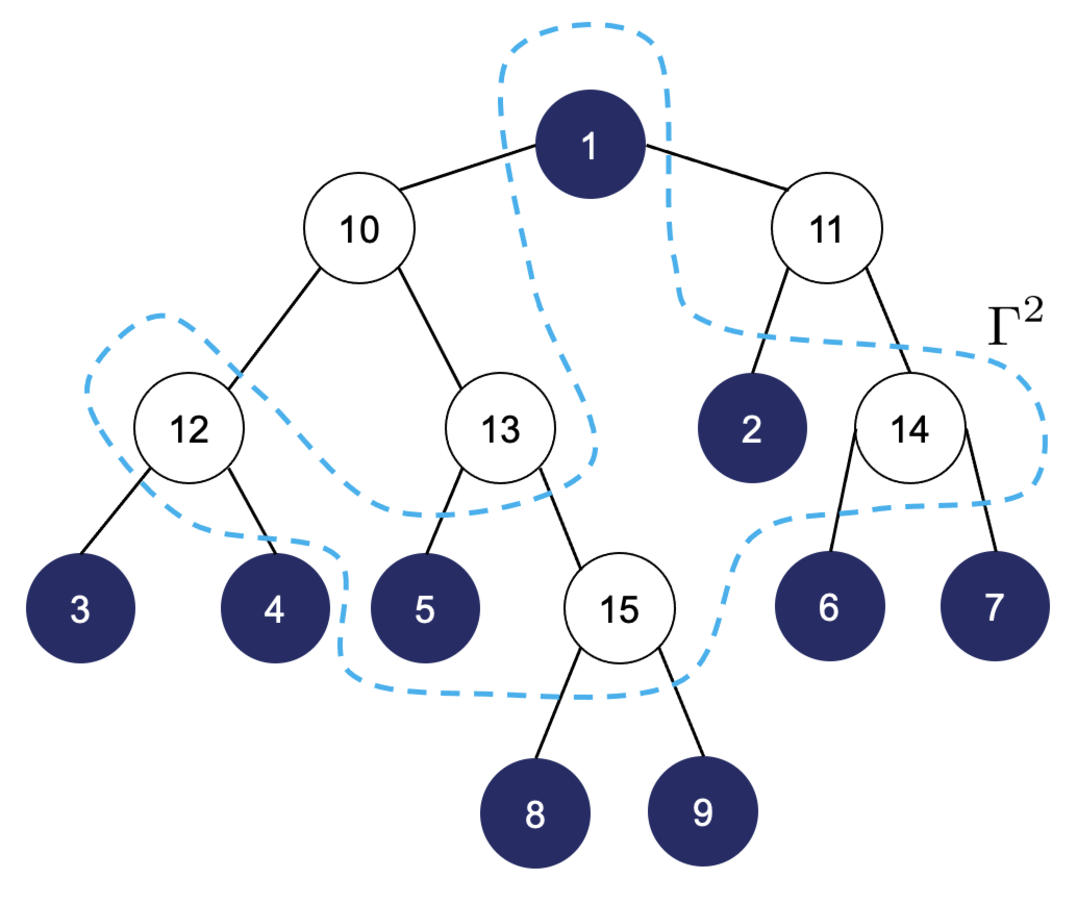}
		\caption{Illustration of \ac{rrg}. The shaded nodes are the observed nodes and the rest are hidden nodes. $\Gamma^{1}=\{x_{1},x_{2},\ldots,,x_{9}\}$, and $\Gamma^{2}$ is the nodes in 
		the dotted lines. If we delete the nodes in $\Gamma^{2}$, the remained unknown hidden nodes are $x_{10}$, $x_{11}$ and $x_{13}$. Nodes $x_{10}$ and $x_{13}$ are at the end of the chain 
		formed by these two nodes, and $x_{11}$ is at the end of the degenerate chain formed by itself.}
		\label{fig:ite_num2}
	\end{figure}
	
	We next prove that there are at least two of $n_{i}$'s not less than $2$. If we delete the nodes in active set $\Gamma^{l}$, the remaining hidden nodes form a 
	single tree or a forest. There will at least two nodes at the end of the chain, which means that they only have one neighbor in hidden nodes, as shown in Fig. \ref{fig:ite_num2}. Since they 
	at least have three neighbors, they have at least two neighbors in $\Gamma^{l}$. Thus, there are at least two of $n_{i}$'s not less than 2, and thus $|\Gamma^{l+1}|\leq |\Gamma^{l}|-2$.
\end{proof}

\begin{corollary}
	The maximum number of iterations of Algorithm \ref{algo:rrg}, $L_{\mathrm{R}}$, is bounded as
	\begin{align}
		\frac{\log \frac{|\mathcal{V}_{\mathrm{obs}}|}{2}}{\log d_{\max}}\leq L_{\mathrm{R}} \leq|\mathcal{V}_{\mathrm{obs}}|-2.
	\end{align}
\end{corollary}
\begin{proof}
	When Algorithm \ref{algo:rrg} terminates, $|\Gamma|\leq 2$. Combining Proposition \ref{propactset} and $|\Gamma|\leq 2$ proves the corollary.
\end{proof}

\begin{theorem}\label{theo:rrgsamplecomp2}
	Under Assumptions \ref{assupleng}--\ref{assupdist}, \ac{rrg} algorithm constructs the correct latent tree with probability at least $1-\eta$ if
	\begin{align}
		n_{2}&\geq \frac{64\lambda^{2} \kappa^{2}}{c\varepsilon^{2}}\big(\frac{9}{2}\big)^{2L_{\mathrm{R}}-2}\log\frac{17l_{\max}^{2}s^{L_{\mathrm{R}}-1}|\mathcal{V}_{\mathrm{obs}}|^{3}}{\eta}\label{ieq:rrg1}\\
		\frac{n_{2}}{n_{1}}&\geq \frac{128\lambda \kappa }{3\varepsilon}\big(\frac{9}{2}\big)^{L_{\mathrm{R}}-1}\log\frac{34l_{\max}^{2}s^{L_{\mathrm{R}}-1}|\mathcal{V}_{\mathrm{obs}}|^{3}}{\eta},\label{ieq:rrg2}
	\end{align}
	where 
	\begin{align}
		\lambda=\frac{2l_{\max}^{2}e^{\rho_{\max}/l_{\max}}}{\delta_{\min}^{1/l_{\max}}}\quad \kappa=\max\{\sigma_{\max}^{2},\rho_{\min}\} \quad s=d_{\max}^{2}+2d_{\max}^{3}(1+2N_{\tau})\quad \varepsilon=\frac{\rho_{\min}}{2},
	\end{align}
	$c$ is an absolute constant, and $L_{\mathrm{R}}$ is the number of iterations of \ac{rrg} needed to construct the tree.
\end{theorem}
\begin{proof}[Proof of Theorem ~\ref{theo:rrgsamplecomp2}]
	It is easy to see by substituting the constants $\lambda$, $\kappa$, $s$ and $\varepsilon$ into \eqref{ieq:rrg1} and \eqref{ieq:rrg2} that Theorem \ref{theo:rrgsamplecomp2} implies Theorem \ref{theo:rrgsamplecomp}, so we provide the proof of Theorem \ref{theo:rrgsamplecomp2} here.

	The error events of learning structure in the $l^{\mathrm{th}}$ layer of the latent tree (the $0^{\mathrm{th}}$ layer consists of the observed nodes, and the $(l+1)^{\mathrm{st}}$ layer is the active 
	set formed from $l^{\mathrm{th}}$ layer). The error events could be enumerated as: misclassification of families $\mathcal{E}^{l}_{\mathrm{f}}$, misclassification of non-families $\mathcal{E}^{l}_{\mathrm{nf}}$, misclassification of 
	parents $\mathcal{E}^{l}_{\mathrm{p}}$ and misclassification of siblings $\mathcal{E}^{l}_{\mathrm{s}}$. We will bound the probabilities of these four error events in the following.
	
	The event representing misclassification of families $\mathcal{E}^{l}_{\mathrm{f}}$ represents classifying the nodes that are not in the same family as a family. Suppose nodes $x_{i}$ and $x_{j}$ are in different families. The event that classifying 
	them to be in the same family $\mathcal{E}^{l}_{\mathrm{f},ij}$ at layer $l$ can be expressed as
	\begin{align}
		\mathcal{E}^{l}_{\mathrm{f},ij}=\big\{|\hat{\Phi}_{ijk}-\hat{\Phi}_{ijk^{\prime}}|<\varepsilon \quad\text{for all}\quad x_{k},x_{k^{\prime}}\in\Gamma^{l}\big\}.
	\end{align}
	We have
	\begin{align}
		\mathbb{P}(\mathcal{E}^{l}_{\mathrm{f},ij})&=\mathbb{P}\Big(\bigcap_{x_{k},x_{k^{\prime}}\in \Gamma}\big\{|\hat{\Phi}_{ijk}-\hat{\Phi}_{ijk^{\prime}}|<\varepsilon\big\}\Big)\leq \min_{x_{k},x_{k^{\prime}}\in \Gamma} \mathbb{P}\Big( |\hat{\Phi}_{ijk}-\hat{\Phi}_{ijk^{\prime}}|<\varepsilon \Big),\\
		\mathbb{P}(\mathcal{E}^{l}_{\mathrm{f}})&=\mathbb{P}\Big(\bigcup_{x_{i},x_{j} \text{not in same family}}\mathcal{E}^{l}_{\mathrm{f},ij}\Big)=\mathbb{P}\Big(\bigcup_{(x_{i},x_{j})\in \Gamma^{l}_{\mathrm{f}}}\mathcal{E}^{l}_{\mathrm{f},ij}\Big).
	\end{align}
	We enumerate all possible structural relationships between $x_{i}$, $x_{j}$, $x_{k}$ and $x_{k^{\prime}}$
	\begin{figure}[H]
		\centering\includegraphics[width=1\columnwidth,draft=false]{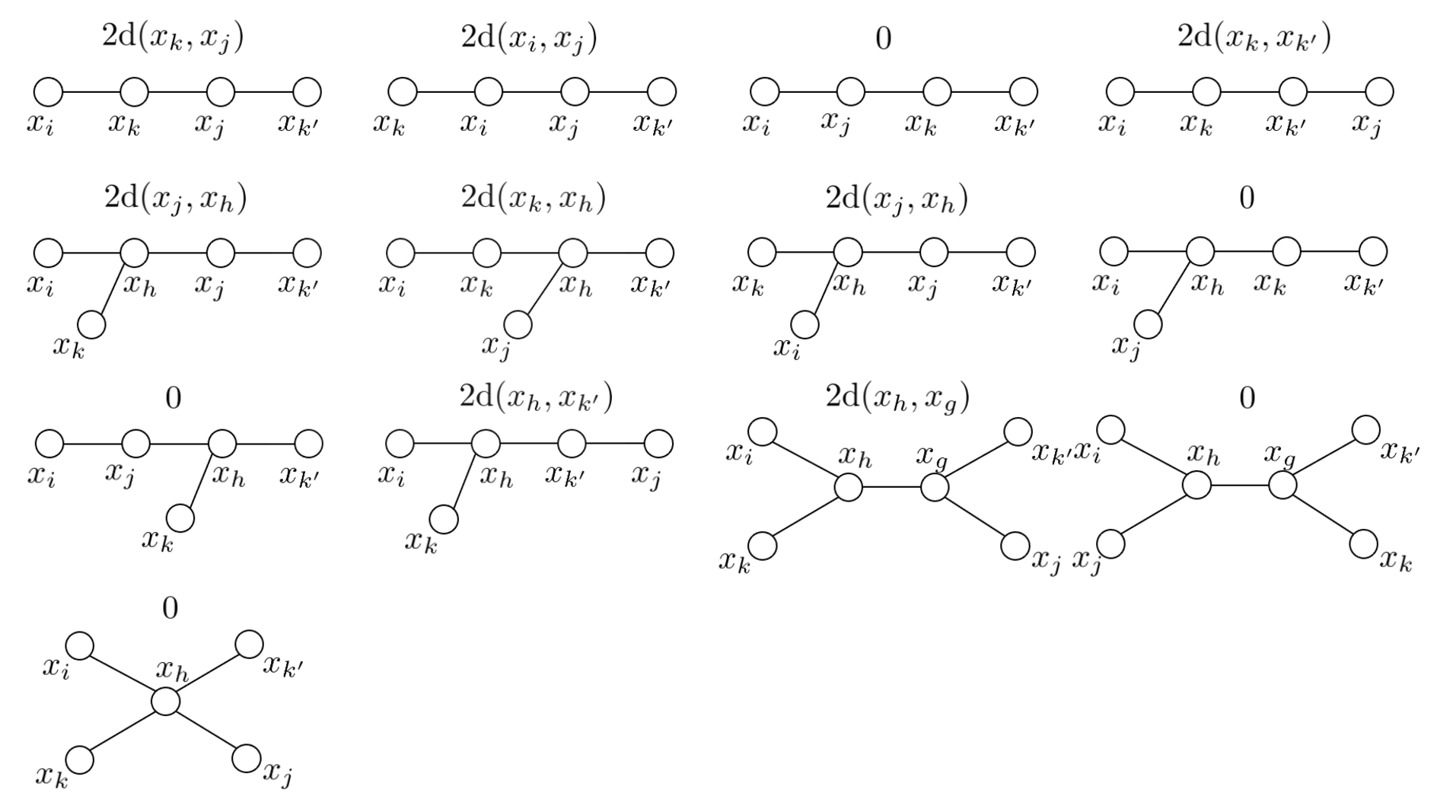}
		\caption{Enumerating of four-node topology and the corresponding $|\Phi_{ijk}-\Phi_{ijk^{\prime}}|$.}
		\label{fig:enumerate}
	\end{figure}
	Let $\varepsilon<2\rho_{\min}$, by decomposing the estimate of the information distance as $\hat{\mathrm{d}}(x_{i},x_{j})=\mathrm{d}(x_{i},x_{j})+\Delta_{ij}$, we have
	\begin{align}
		\mathbb{P}\Big(|\hat{\Phi}_{ijk}-\hat{\Phi}_{ijk^{\prime}}|<\varepsilon\Big)&=\mathbb{P}\Big(|\Phi_{ijk}-\Phi_{ijk^{\prime}}+\Delta_{ik}-\Delta_{jk}-\Delta_{ik^{\prime}}+\Delta_{jk^{\prime}}|<\varepsilon\Big)\nonumber\\
		&\leq \mathbb{P}\Big(\Delta_{ik}-\Delta_{jk}-\Delta_{ik^{\prime}}+\Delta_{jk^{\prime}}<\varepsilon-(\Phi_{ijk}-\Phi_{ijk^{\prime}})\Big)\nonumber \\
		&\leq \mathbb{P}\Big(\Delta_{ik}-\Delta_{jk}-\Delta_{ik^{\prime}}+\Delta_{jk^{\prime}}<\varepsilon-2\rho_{\min}\Big)\nonumber\\
		&\leq \mathbb{P}\Big(|\Delta_{ik}|>\frac{2\rho_{\min}-\varepsilon}{4}\Big)+\mathbb{P}\Big(|\Delta_{jk}|>\frac{2\rho_{\min}-\varepsilon}{4}\Big)\nonumber\\
		&\quad+\mathbb{P}\Big(\Delta_{jk^{\prime}}|>\frac{2\rho_{\min}-\varepsilon}{4}\Big)+\mathbb{P}\Big(|\Delta_{ik^{\prime}}|>\frac{2\rho_{\min}-\varepsilon}{4}\Big).
	\end{align}
	The event representing misclassification of the parents $\mathcal{E}^{l}_{\mathrm{p}}$ represents classifying a sibling relationship as a parent relationship. Following similar procedures, we have
	\begin{align}
		\mathbb{P}(\mathcal{E}^{l}_{\mathrm{p}})&=\mathbb{P}\Big(\bigcup_{x_{i},x_{j}\text{ are siblings}}\mathcal{E}^{l}_{\mathrm{p},ij}\Big)=\mathbb{P}\Big(\bigcup_{(x_{i},x_{j})\in \Gamma^{l}_{\mathrm{p}}}\mathcal{E}^{l}_{\mathrm{p},ij}\Big)\\
		\mathbb{P}(\mathcal{E}^{l}_{\mathrm{p},ij})&=\mathbb{P}\bigg(\bigcap_{x_{k}\in \Gamma^{l}}\Big\{\big|\hat{\Phi}_{ijk}-\hat{\mathrm{d}}(x_{i},x_{j})\big|<\varepsilon\Big\}\bigg)\leq \min_{x_{k}\in \Gamma^{l}} \mathbb{P}\Big(\big|\hat{\Phi}_{ijk}-\hat{\mathrm{d}}(x_{i},x_{j})\big|<\varepsilon\Big)\\
		&\leq \mathbb{P}\Big(|\Delta_{ij}|>\frac{2\rho_{\min}-\varepsilon}{3}\Big)+\mathbb{P}\Big(|\Delta_{ik}|>\frac{2\rho_{\min}-\varepsilon}{3}\Big)+\mathbb{P}\Big(|\Delta_{jk}|>\frac{2\rho_{\min}-\varepsilon}{3}\Big)
	\end{align}
	The event representing misclassification of non-families $\mathcal{E}^{l}_{\mathrm{nf}}$ represents classifying family members as non-family members. We have
	\begin{align}
		\mathbb{P}(\mathcal{E}^{l}_{\mathrm{nf}})&=\mathbb{P}\big(\bigcup_{x_{i},x_{j}\text{ in the same family}}\mathcal{E}^{l}_{\mathrm{nf},ij}\big)=\mathbb{P}\big(\bigcup_{(x_{i},x_{j})\in \Gamma^{l}_{\mathrm{nf}}}\mathcal{E}^{l}_{\mathrm{nf},ij}\big)\\
		\mathbb{P}(\mathcal{E}^{l}_{\mathrm{nf}})&=\mathbb{P}\Big(\bigcup_{x_{i},x_{j}\text{ in the same family}}\bigcup_{x_{k},x_{k^{\prime}}\in\Gamma}\big\{|\hat{\Phi}_{ijk}-\hat{\Phi}_{ijk^{\prime}}|> \varepsilon\big\}\Big)
	\end{align}
	and
	\begin{align}
		\mathbb{P} \Big(|\hat{\Phi}_{ijk}-\hat{\Phi}_{ijk^{\prime}}|\geq \varepsilon\Big) 
		&\leq \mathbb{P}\Big(|\Delta_{ik}|>\frac{\varepsilon}{4}\Big)+\mathbb{P}\Big(|\Delta_{jk}|>\frac{\varepsilon}{4}\Big)+\mathbb{P}\Big(|\Delta_{jk^{\prime}}|>\frac{\varepsilon}{4}\Big)+\mathbb{P}\Big(|\Delta_{ik^{\prime}}|>\frac{\varepsilon}{4}\Big)
	\end{align}
	The event representing misclassification of siblings $\mathcal{E}^{l}_{\mathrm{s}}$ represents classifying parent relationship as sibling relationship. Similarly, we have
	\begin{align}
		\mathbb{P}(\mathcal{E}^{l}_{\mathrm{s}})&=\mathbb{P}\bigg(\bigcup_{x_{i}\text{ is the parent of }x_{j}}\mathcal{E}^{l}_{\mathrm{s},ij}\Big)=\mathbb{P}\bigg(\bigcup_{(x_{i},x_{j})\in \Gamma^{l}_{\mathrm{s}}}\mathcal{E}^{l}_{\mathrm{s},ij}\bigg) \\
		\mathbb{P}(\mathcal{E}^{l}_{\mathrm{s},ij})&=\mathbb{P}\bigg(\bigcup_{x_{k}\in \Gamma}\Big\{\big|\hat{\Phi}_{jik}-\hat{\mathrm{d}}(x_{i},x_{j})\big|>\varepsilon\Big\}\bigg)
	\end{align}
	and
	\begin{align}
		\mathbb{P}\Big(\big|\hat{\Phi}_{jik}-\hat{\mathrm{d}}(x_{i},x_{j})\big|>\varepsilon\Big)\leq \mathbb{P}\Big(|\Delta_{ij}|>\frac{\varepsilon}{3}\Big)+\mathbb{P}\Big(|\Delta_{ik}|>\frac{\varepsilon}{3}\Big)+\mathbb{P}\Big(|\Delta_{jk}|>\frac{\varepsilon}{3}\Big)
	\end{align}
	
	To bound the probability of error event in $l^{\mathrm{th}}$ layer, we first analyze the cardinalities of $\Gamma^{l}_{\mathrm{f}}$, $\Gamma^{l}_{\mathrm{p}}$, $\Gamma^{l}_{\mathrm{nf}}$ and $\Gamma^{l}_{\mathrm{s}}$. Note 
	that the definitions of these four sets are
	\begin{align}
		\Gamma^{l}_{\mathrm{f}}&=\big\{(x_{i},x_{j}) \,:\, x_{i} \text{ and } x_{j} \text{ are not in the same family } x_{i},x_{j}\in \Gamma^{l}\big\} \\
		\Gamma^{l}_{\mathrm{p}}&=\big\{(x_{i},x_{j})\,:\,  x_{i} \text{ and } x_{j} \text{ are siblings } x_{i},x_{j}\in \Gamma^{l}\big\} \\
		\Gamma^{l}_{\mathrm{nf}}&=\big\{(x_{i},x_{j})\,:\,  x_{i} \text{ and } x_{j} \text{ are in the same family } x_{i},x_{j}\in \Gamma^{l}\big\} \\
		\Gamma^{l}_{\mathrm{s}}&=\big\{(x_{i},x_{j})\,:\, x_{i} \text{ and } x_{j} \text{ is the parent of } x_{i},x_{j}\in \Gamma^{l}\big\} .
	\end{align}
	
	Clearly, we have
	\begin{align}
		|\Gamma^{l}_{\mathrm{f}}|\leq \binom{|\Gamma^{l}|}{2} \quad\text{ and }\quad |\Gamma^{l}_{\mathrm{s}}|\leq |\Gamma^{l}|.
	\end{align}
	The cardinality of $\Gamma^{l}_{\mathrm{p}}$ can be bounded as
	\begin{align}
		|\Gamma^{l}_{\mathrm{p}}| \leq \sum_{i=1}^{|\Gamma^{l+1}|} \binom{n_{i}}{2}
	\end{align}
	where $n_{i}$ is the size of each family in $\Gamma^{l}$.
	
	From Lemma \ref{optlem}, we deduce that
	\begin{align}
		|\Gamma^{l}_{\mathrm{p}}|\leq \frac{1}{2}d_{\max}(d_{\max}-1)\frac{|\Gamma^{l}|}{d_{\max}}=\frac{1}{2}|\Gamma^{l}|(d_{\max}-1).
	\end{align}
	Similarly, we have
	\begin{align}
		|\Gamma^{l}_{\mathrm{nf}}|\leq \frac{1}{2}|\Gamma^{l}|(d_{\max}-1).
	\end{align}
	The probability of the error event in $l^{\mathrm{th}}$ layer can be bounded as
	\begin{align}
		\mathbb{P}(\mathcal{E}^{l})&=\mathbb{P}(\mathcal{E}^{l}_{\mathrm{f}}\cup\mathcal{E}^{l}_{\mathrm{p}}\cup\mathcal{E}^{l}_{\mathrm{nf}}\cup\mathcal{E}^{l}_{\mathrm{s}}) \nonumber\\
			&\leq \mathbb{P}(\mathcal{E}^{l}_{\mathrm{f}})+\mathbb{P}(\mathcal{E}^{l}_{\mathrm{p}})+\mathbb{P}(\mathcal{E}^{l}_{\mathrm{nf}})+\mathbb{P}(\mathcal{E}^{l}_{\mathrm{s}})\nonumber\\
			&\leq 4\binom{|\Gamma^{l}|}{2}h^{(l)}\Big(\frac{2\rho_{\min}-\varepsilon}{4}\Big)+\frac{3}{2}|\Gamma^{l}|(d_{\max}-1)h^{(l)}\Big(\frac{2\rho_{\min}-\varepsilon}{3}\Big)\nonumber\\
			&\quad +3|\Gamma^{l}|^{2}h^{(l)}\Big(\frac{\varepsilon}{3}\Big)+2|\Gamma^{l}|^{3}(d_{\max}-1)h^{(l)}\Big(\frac{\varepsilon}{4}\Big).
	\end{align}
	The probability of learning the wrong structure is
	\begin{align}
		\mathbb{P}(\mathcal{E})&=\mathbb{P}\bigg(\bigcup_{l}\mathcal{E}^{l}\bigg)\leq \sum_{l} \mathbb{P}(\mathcal{E}^{l})\\
		&\leq \sum_{l} 4\binom{|\Gamma^{l}|}{2}h^{(l)}\Big(\frac{2\rho_{\min}-\varepsilon}{4}\Big)+\frac{3}{2}|\Gamma^{l}|(d_{\max}-1)h^{(l)}\Big(\frac{2\rho_{\min}-\varepsilon}{3}\Big)\nonumber\\
		&\quad +3|\Gamma^{l}|^{2}h^{(l)}\Big(\frac{\varepsilon}{3}\Big)+2|\Gamma^{l}|^{3}(d_{\max}-1)h^{(l)}\Big(\frac{\varepsilon}{4}\Big)
	\end{align}
	With Proposition \ref{propactset}, we have
	\begin{align}
		\mathbb{P}(\mathcal{E})\leq & \sum_{l=0}^{L-1} 4\binom{|\mathcal{V}_{\mathrm{obs}}|-2l}{2}h^{(l)}\Big(\frac{2\rho_{\min}-\varepsilon}{4}\Big)+\frac{3}{2}(|\mathcal{V}_{\mathrm{obs}}|-2l)(d_{\max}-1)h^{(l)}\Big(\frac{2\rho_{\min}-\varepsilon}{3}\Big)\nonumber\\
		&\quad +3(|\mathcal{V}_{\mathrm{obs}}|-2l)^{2}h^{(l)}\Big(\frac{\varepsilon}{3}\Big)+2(|\mathcal{V}_{\mathrm{obs}}|-2l)^{3}(d_{\max}-1)h^{(l)}\Big(\frac{\varepsilon}{4}\Big),
	\end{align}
	where $L$ is the number of iterations of \ac{rrg}.
	
	We can separately bound the two parts of the first term in the summation $\binom{|\mathcal{V}_{\mathrm{obs}}|-2l}{2}h^{(l)}\Big(\frac{2\rho_{\min}-\varepsilon}{4}\Big)$ as
	\begin{align}
		\frac{4\binom{|\mathcal{V}_{\mathrm{obs}}|-2l}{2}s^{l}ae^{-wm^{l}x}}{4\binom{|\mathcal{V}_{\mathrm{obs}}|-2L}{2}s^{L-1}ae^{-wm^{L-1}x}}\leq \frac{\big(|\mathcal{V}_{\mathrm{obs}}|-2l\big)\big(|\mathcal{V}_{\mathrm{obs}}|-2l-1\big)}{2s^{L-1-l}}\leq \frac{|\mathcal{V}_{\mathrm{obs}}|^{2}}{2s^{L-1-l}} \text{ for }x>0\nonumber
	\end{align}
	and
	\begin{align}
		\frac{4\binom{|\mathcal{V}_{\mathrm{obs}}|-2l}{2}s^{l}be^{-um^{2l}x^{2}}}{4\binom{|\mathcal{V}_{\mathrm{obs}}|-2L}{2}s^{L-1}be^{-um^{2L-2}x^{2}}}\leq \frac{\big(|\mathcal{V}_{\mathrm{obs}}|-2l\big)\big(|\mathcal{V}_{\mathrm{obs}}|-2l-1\big)}{2s^{L-1-l}}\leq \frac{|\mathcal{V}_{\mathrm{obs}}|^{2}}{2s^{L-1-l}}.
	\end{align}
	These bounds imply that
	\begin{align}
		\sum_{l=0}^{L-1} 4\binom{|\mathcal{V}_{\mathrm{obs}}|-2l}{2}h^{(l)}(x)\leq \big[1+\frac{|\mathcal{V}_{\mathrm{obs}}|^{2}}{2}\sum_{i=1}^{\infty}\frac{1}{s^{i}}\big]4h^{(L-1)}(x)=4\Big(1+\frac{|\mathcal{V}_{\mathrm{obs}}|^{2}}{2(s-1)}\Big)4h^{(L-1)}(x) \text{ for }x>0 \nonumber
	\end{align}
	Similar procedures could be implemented on other terms, and we will obtain
	\begin{align}\label{eqn:conserrorp}
		\mathbb{P}(\mathcal{E})\leq &\bigg[4\Big(1+\frac{|\mathcal{V}_{\mathrm{obs}}|^{2}}{2(s-1)}\Big)+\frac{3}{2}(d_{\max}-1)\Big(1+\frac{|\mathcal{V}_{\mathrm{obs}}|}{2(s-1)}\Big)\bigg]h^{(L-1)}\Big(\frac{2\rho_{\min}-\varepsilon}{4}\Big) \nonumber\\
		&+\bigg[3\Big(1+\frac{|\mathcal{V}_{\mathrm{obs}}|^{2}}{4(s-1)}\Big)+2(d_{\max}-1)\Big(1+\frac{|\mathcal{V}_{\mathrm{obs}}|^3}{8(s-1)}\Big)\bigg]h^{(L-1)}\Big(\frac{\varepsilon}{4}\Big)\\
		=&\bigg[4\Big(1+\frac{|\mathcal{V}_{\mathrm{obs}}|^{2}}{2(s-1)}\Big)+\frac{3}{2}(d_{\max}-1)\Big(1+\frac{|\mathcal{V}_{\mathrm{obs}}|}{2(s-1)}\Big)\bigg]\Big(ae^{-wm^{L-1}\frac{2\rho_{\min}-\varepsilon}{4}}+be^{-um^{2L-2}(\frac{2\rho_{\min}-\varepsilon}{4})^{2}}\Big)\nonumber\\
		&+\bigg[3\Big(1+\frac{|\mathcal{V}_{\mathrm{obs}}|^{2}}{4(s-1)}\Big)+2(d_{\max}-1)\Big(1+\frac{|\mathcal{V}_{\mathrm{obs}}|^3}{8(s-1)}\Big)\bigg]\Big(ae^{-wm^{L-1}\frac{\varepsilon}{4}}+be^{-um^{2L-2}(\frac{\varepsilon}{4})^{2}}\Big)\leq \eta . \nonumber
	\end{align}
	Upper bounding each of the four  terms in inequality \eqref{eqn:conserrorp} by ${\eta}/{4}$, we obtain the following sufficient conditions of $n_{1}$ and $n_{2}$ to ensure that $\mathbb{P}(\mathcal{E})\leq \eta$:
	\begin{align}
		n_{2}\geq \max\Bigg\{&\frac{64\lambda^{2} \kappa^{2}}{c(2\rho_{\min}-\varepsilon)^{2}}\Big(\frac{9}{2}\Big)^{2L-2}\log\frac{4l_{\max}^{2}s^{L-1}\big[4(1+\frac{|\mathcal{V}_{\mathrm{obs}}|^{2}}{2(s-1)})+\frac{3}{2}(d_{\max}-1)(1+\frac{|\mathcal{V}_{\mathrm{obs}}|}{2(s-1)})\big]}{\eta},\nonumber\\
		&\frac{64\lambda^{2} \kappa^{2}}{c\varepsilon^{2}}\Big(\frac{9}{2}\Big)^{2L-2}\log\frac{4l_{\max}^{2}s^{L-1}\big[3(1+\frac{|\mathcal{V}_{\mathrm{obs}}|^{2}}{4(s-1)})+2(d_{\max}-1)(1+\frac{|\mathcal{V}_{\mathrm{obs}}|^3}{8(s-1)})\big]}{\eta}\Bigg\},\nonumber\\
		\frac{n_{2}}{n_{1}}\geq \max\Bigg\{&\frac{128\lambda \kappa }{3(2\rho_{\min}-\varepsilon)}\Big(\frac{9}{2}\Big)^{L-1}\log\frac{8l_{\max}^{2}s^{L-1}\big[4(1+\frac{|\mathcal{V}_{\mathrm{obs}}|^{2}}{2(s-1)})+\frac{3}{2}(d_{\max}-1)(1+\frac{|\mathcal{V}_{\mathrm{obs}}|}{2(s-1)})\big]}{\eta},\nonumber\\
		&\frac{128\lambda \kappa }{3\varepsilon}\Big(\frac{9}{2}\Big)^{L-1}\log\frac{8l_{\max}^{2}s^{L-1}\big[3(1+\frac{|\mathcal{V}_{\mathrm{obs}}|^{2}}{4(s-1)})+2(d_{\max}-1)(1+\frac{|\mathcal{V}_{\mathrm{obs}}|^3}{8(s-1)})\big]}{\eta}\Bigg\}.\nonumber
	\end{align}
	Note that
	\begin{align}
		&\max\Bigg\{ 4\Big(1+\frac{|\mathcal{V}_{\mathrm{obs}}|^{2}}{2(s-1)}\Big)+\frac{3}{2}(d_{\max}-1)\Big(1+\frac{|\mathcal{V}_{\mathrm{obs}}|}{2(s-1)}\Big), \nonumber\\
		&\qquad\qquad 3\Big(1+\frac{|\mathcal{V}_{\mathrm{obs}}|^{2}}{(s-1)}\Big)+2(d_{\max}-1)\Big(1+\frac{|\mathcal{V}_{\mathrm{obs}}|^3}{8(s-1)}\Big) \Bigg\} \nonumber  \\
		&<4\Big(1+\frac{|\mathcal{V}_{\mathrm{obs}}|^{2}}{2(s-1)}\Big)+2(d_{\max}-1)\Big(1+\frac{|\mathcal{V}_{\mathrm{obs}}|^3}{2(s-1)}\Big) \\
		&\leq 2(d_{\max}-1)\Big(2+\frac{|\mathcal{V}_{\mathrm{obs}}|^3+2|\mathcal{V}_{\mathrm{obs}}|^{2}}{2(s-1)}\Big)\\
		&<2(d_{\max}-1)\Big(2+\frac{|\mathcal{V}_{\mathrm{obs}}|^3+2|\mathcal{V}_{\mathrm{obs}}|^{2}}{s}\Big) \\
		&\overset{(a)}{<}2(d_{\max}-1)\frac{|\mathcal{V}_{\mathrm{obs}}|^3+2|\mathcal{V}_{\mathrm{obs}}|^{2}+7N_{\tau}|\mathcal{V}_{\mathrm{obs}}|^{3}}{s}\\
		&<17d_{\max}N_{\tau}\frac{|\mathcal{V}_{\mathrm{obs}}|^{3}}{s} \\
		&\overset{(b)}{<}\frac{17}{4}|\mathcal{V}_{\mathrm{obs}}|^{3},
	\end{align}
	where inequality $(a)$ and $(b)$ result from $s< 7N_{\tau}|\mathcal{V}_{\mathrm{obs}}|^{3}$ and $d_{\max}N_{\tau}<\frac{s}{4}$, respectively. Choosing $\varepsilon<\rho_{\min}$, we then can derive the sufficient conditions to ensure that $\mathbb{P}(\mathcal{E})\leq \eta$ as
	\begin{align}
		n_{2}&\geq \frac{64\lambda^{2} \kappa^{2}}{c\varepsilon^{2}}\Big(\frac{9}{2}\Big)^{2L-2}\log\frac{17l_{\max}^{2}s^{L-1}|\mathcal{V}_{\mathrm{obs}}|^{3}}{\eta},\\
		\frac{n_{2}}{n_{1}}&\geq \frac{128\lambda \kappa }{3\varepsilon}\Big(\frac{9}{2}\Big)^{L-1}\log\frac{34l_{\max}^{2}s^{L-1}|\mathcal{V}_{\mathrm{obs}}|^{3}}{\eta}.
	\end{align}
	In Theorem \ref{theo:rrgsamplecomp}, we choose $\varepsilon=\frac{\rho_{\min}}{2}$.
	
	Then the following conditions
	\begin{align}\label{eqn:sqrtn}
		n_{2}&\geq \frac{64\lambda^{2} \kappa^{2}}{c\varepsilon^{2}}\Big(\frac{9}{2}\Big)^{2L-2}\log\frac{17l_{\max}^{2}s^{L-1}|\mathcal{V}_{\mathrm{obs}}|^{3}}{\eta},\\
		n_{1}&=O\Big(\frac{\sqrt{n_{2}}}{\log n_{2}}\Big).
	\end{align}
	are sufficient to guarantee that $\mathbb{P}(\mathcal{E})\leq \eta$.
	
	We are going to prove that there exists $C^{\prime}>0$, such that
	\begin{align}\label{inq:sqrtn}
		C^{\prime}\frac{\sqrt{n_{2}}}{\log n_{2}}\leq \frac{n_{2}}{\frac{128\lambda \kappa }{3\varepsilon}(\frac{9}{2})^{L-1}\log\frac{34l_{\max}^{2}s^{L-1}|\mathcal{V}_{\mathrm{obs}}|^{3}}{\eta}},
	\end{align}
	which is equivalent to 
	\begin{align}
		C^{\prime}\frac{128\lambda \kappa }{3\varepsilon}\Big(\frac{9}{2}\Big)^{L-1}\log\frac{34l_{\max}^{2}s^{L-1}|\mathcal{V}_{\mathrm{obs}}|^{3}}{\eta}\leq \sqrt{n_{2}}\log n_{2}.
	\end{align}
	Since $n_{2}$ is lower bounded as in \eqref{eqn:sqrtn}, it is sufficient to show that there exists $C^{\prime}>0$, such that
	\begin{align}
		&(C^{\prime })^2\bigg(\frac{128\lambda \kappa }{3\varepsilon}\Big(\frac{9}{2}\Big)^{L-1}\log\frac{34l_{\max}^{2}s^{L-1}|\mathcal{V}_{\mathrm{obs}}|^{3}}{\eta}\bigg)^{2}\nonumber\\
		&\leq \frac{64\lambda^{2} \kappa^{2}}{c\varepsilon^{2}}\Big(\frac{9}{2}\Big)^{2L-2}\log\frac{17l_{\max}^{2}s^{L-1}|\mathcal{V}_{\mathrm{obs}}|^{3}}{\eta}\log\bigg[\frac{64\lambda^{2} \kappa^{2}}{c\varepsilon^{2}}\Big(\frac{9}{2}\Big)^{2L-2}\log\frac{17l_{\max}^{2}s^{L-1}|\mathcal{V}_{\mathrm{obs}}|^{3}}{\eta}\bigg], \nonumber
	\end{align}
	which is equivalent to
	\begin{align}
		(C^{\prime})^2\leq \frac{9}{256c}\frac{\log\frac{17l_{\max}^{2}s^{L-1}|\mathcal{V}_{\mathrm{obs}}|^{3}}{\eta}}{\log\frac{34l_{\max}^{2}s^{L-1}|\mathcal{V}_{\mathrm{obs}}|^{3}}{\eta}}\frac{\Big(\log\big(\frac{64\lambda^{2} \kappa^{2}}{c\varepsilon^{2}}(\frac{9}{2})^{2L-2}\big)+\log\log\frac{17l_{\max}^{2}s^{L-1}|\mathcal{V}_{\mathrm{obs}}|^{3}}{\eta}\Big)^{2}}{\log\frac{34l_{\max}^{2}s^{L-1}|\mathcal{V}_{\mathrm{obs}}|^{3}}{\eta}}.
	\end{align}
	We have
	\begin{align}
		\log\frac{17l_{\max}^{2}s^{L-1}|\mathcal{V}_{\mathrm{obs}}|^{3}}{\eta}/\log\frac{34l_{\max}^{2}s^{L-1}|\mathcal{V}_{\mathrm{obs}}|^{3}}{\eta}&>\frac{1}{2}\quad\text{and}\\
		\frac{\Big(\log\big(\frac{64\lambda^{2} \kappa^{2}}{c\varepsilon^{2}}(\frac{9}{2})^{2L-2}\big)+\log\log\frac{17l_{\max}^{2}s^{L-1}|\mathcal{V}_{\mathrm{obs}}|^{3}}{\eta}\Big)^{2}}{\log\frac{34l_{\max}^{2}s^{L-1}|\mathcal{V}_{\mathrm{obs}}|^{3}}{\eta}}&>\frac{\Big(\log\big(\frac{64\lambda^{2} \kappa^{2}}{c\varepsilon^{2}}(\frac{9}{2})^{2L-2}\big)\Big)^{2}}{\log\frac{34l_{\max}^{2}s^{L-1}|\mathcal{V}_{\mathrm{obs}}|^{3}}{\eta}}.
	\end{align}
	Since
	\begin{align}
		\lim_{L\to\infty}\frac{\Big(\log\big(\frac{64\lambda^{2} \kappa^{2}}{c\varepsilon^{2}}(\frac{9}{2})^{2L-2}\big)\Big)^{2}}{\log\frac{34l_{\max}^{2}s^{L-1}|\mathcal{V}_{\mathrm{obs}}|^{3}}{\eta}}=+\infty,
	\end{align}
	we can see that there exists $C^{\prime}>0$ that satisfies inequality \eqref{inq:sqrtn}.
\end{proof}

\section{Proofs of results in Section~\ref{subsec:snjnj}}
\begin{theorem}\label{theo:rnjsamplecomp2}
	If Assumptions \ref{assupleng} to \ref{assupdist} hold and all the nodes have exactly two children, \ac{rnj} constructs the correct latent tree with probability at least $1-\eta$ if
	\begin{align}
		n_{2}&>\frac{16\lambda^{2}\kappa^{2}}{c\rho_{\min}^{2}}\log\Big(\frac{2|\mathcal{V}_{\mathrm{obs}}|^{2}l_{\max}^{2}}{\eta}\Big)\label{ieq:rnj1}\\
		\frac{n_{2}}{n_{1}}&>\frac{64\lambda\kappa }{3\rho_{\min}}\log\Big(\frac{4|\mathcal{V}_{\mathrm{obs}}|^{2}l_{\max}^{2}}{\eta}\Big)\label{ieq:rnj2}
	\end{align}
	where 
	\begin{align}
		\lambda=\frac{2l_{\max}^{2}e^{\rho_{\max}/l_{\max}}}{\delta_{\min}^{1/l_{\max}}}\quad\text{ and }\quad \kappa=\max\{\sigma_{\max}^{2},\rho_{\min}\}, 
	\end{align}
	and $c$ is an absolute constant.
\end{theorem}
\begin{proof}[Proof of Theorem ~\ref{theo:rnjsamplecomp2}]
	It is easy to see by substituting the constants $\lambda$ and $\kappa$ into \eqref{ieq:rnj1} and \eqref{ieq:rnj2} that Theorem \ref{theo:rnjsamplecomp2} implies Theorem \ref{theo:rnjsamplecomp}, so we provide the proof of Theorem \ref{theo:rnjsamplecomp2} here.

	With the sufficient condition in Proposition \ref{prop:njsuff}, we can bound the probability of error event by the union bound as follows
	\begin{align}
		\mathbb{P}(\mathcal{E})&\le \mathbb{P}\Big(\max_{x_{i},x_{j}\in\mathcal{V}_{\mathrm{obs}}} \big|\hat{\mathrm{d}}(x_{i},x_{j})-\mathrm{d}(x_{i},x_{j})\big|> \frac{\rho_{\min}}{2}\Big) \\
		&\le |\mathcal{V}_{\mathrm{obs}}|^{2}\mathbb{P}\big(\big|\hat{\mathrm{d}}(x_{i},x_{j})-\mathrm{d}(x_{i},x_{j})\big|> \frac{\rho_{\min}}{2}\big).
	\end{align}
	We bound two terms in the tail probability separately as
	\begin{align}
		2l_{\max}^{2}e^{-\frac{3n_{2}}{64\lambda\kappa n_{1}} \rho_{\min}}&<\frac{\eta}{2|\mathcal{V}_{\mathrm{obs}}|^{2}}\\
		l_{\max}^{2}e^{-c\frac{n_{2}}{16\lambda^{2}\kappa^{2}}\rho_{\min}^{2}}&<\frac{\eta}{2|\mathcal{V}_{\mathrm{obs}}|^{2}}.
	\end{align}
	Then we have
	\begin{align}
		n_{2}&>\frac{16\lambda^{2}\kappa^{2}}{c\rho_{\min}^{2}}\log\Big(\frac{2|\mathcal{V}_{\mathrm{obs}}|^{2}l_{\max}^{2}}{\eta}\Big),\\
		\frac{n_{2}}{n_{1}}&>\frac{64\lambda\kappa }{3\rho_{\min}}\log\Big(\frac{4|\mathcal{V}_{\mathrm{obs}}|^{2}l_{\max}^{2}}{\eta}\Big).
	\end{align}
	The proof that $n_{1}=O(\sqrt{n_{2}}/\log n_{2})$ can be derived by following the similar procedures in the proof of Theorem \ref{theo:rrgsamplecomp}.
\end{proof}

\begin{proposition}\label{prop:spectralconcen}
	If Assumption \ref{assupleng} to \ref{assupdist} hold and the truncated inner product is adopted to estimate the information distances, 
	\begin{align}
		\mathbb{P}\big(\|\hat{\mathbf{R}}-\mathbf{R}\|_{2}>t\big)\leq |\mathcal{V}_{\mathrm{obs}}|^{2}f\Big(e^{\rho_{\min}}\frac{t}{|\mathcal{V}_{\mathrm{obs}}|}\Big),
	\end{align}
	where the function $f$ is defined as
	\begin{align}
		f(x)\triangleq 2l_{\max}^{2}e^{-\frac{3n_{2}}{32\lambda\kappa n_{1}} x}+l_{\max}^{2}e^{-c\frac{n_{2}}{4\lambda^{2}\kappa^{2}}x^{2}}=ae^{-wx}+be^{-ux^{2}},
	\end{align}
	with $\lambda=2l_{\max}^{2}e^{\rho_{\max}/l_{\max}}/\delta_{\min}^{1/l_{\max}}$, $w=\frac{3n_{2}}{32\lambda\kappa n_{1}}$, $u=c\frac{n_{2}}{4\lambda^{2}\kappa^{2}}$, $a=2l_{\max}^{2}$ and $b=l_{\max}^{2}$.
\end{proposition}
\begin{proof}[Proof of Proposition ~\ref{prop:spectralconcen}]
	Noting that $\mathbf{R}_{ij}=\exp\big(-\mathrm{d}(x_{i},x_{j})\big)$, we have
	\begin{align}
		\mathbb{P}\big(|\hat{\mathbf{R}}_{ij}-\mathbf{R}_{ij}|>t\big)&=\mathbb{P}\bigg(\Big|\exp\big(-\hat{\mathrm{d}}(x_{i},x_{j})\big)-\exp\big(-\mathrm{d}(x_{i},x_{j})\big)\Big|>t\bigg) \\
		&\overset{(a)}{\leq}\mathbb{P}\Big(\big|\hat{\mathrm{d}}(x_{i},x_{j})-\mathrm{d}(x_{i},x_{j})\big|>e^{\rho_{\min}}t\Big) \\
		&<f(e^{\rho_{\min}}t),
	\end{align}
	where inequality $(a)$ is derived from Taylor's Theorem.

	Since
	\begin{align}
		\|\hat{\mathbf{R}}-\mathbf{R}\|_{2}\leq |\mathcal{V}_{\mathrm{obs}}|\max_{i,j} |\hat{\mathbf{R}}_{ij}-\mathbf{R}_{ij}|,
	\end{align}
	we have
	\begin{align}
		\mathbb{P}\big(\|\hat{\mathbf{R}}-\mathbf{R}\|_{2}>t\big)\leq \mathbb{P}\Big(\max_{i,j} |\hat{\mathbf{R}}_{ij}-\mathbf{R}_{ij}|>\frac{t}{|\mathcal{V}_{\mathrm{obs}}|}\Big)\leq |\mathcal{V}_{\mathrm{obs}}|^{2}f\Big(e^{\rho_{\min}}\frac{t}{|\mathcal{V}_{\mathrm{obs}}|}\Big)
	\end{align}
	as desired.
\end{proof}

\begin{theorem}\label{theo:rsnjsamplecomp2}
	If Assumptions \ref{assupleng} to \ref{assupdist} hold and all the nodes have exactly two children, \ac{rsnj} constructs the correct latent tree with probability at least $1-\eta$ if
	\begin{align}
		n_{2}&\geq \frac{16\lambda^{2}\kappa^{2}|\mathcal{V}_{\mathrm{obs}}|^{2}}{ce^{2\rho_{\min}}g(|\mathcal{V}_{\mathrm{obs}}|,\rho_{\min},\rho_{\max})^{2}}\log\frac{2|\mathcal{V}_{\mathrm{obs}}|^{2}l_{\max}^{2}}{\eta}\label{ieq:rsnj1}\\
		\frac{n_{2}}{n_{1}}&\geq \frac{64\lambda\kappa |\mathcal{V}_{\mathrm{obs}}|}{3e^{\rho_{\min}}g(|\mathcal{V}_{\mathrm{obs}}|,\rho_{\min},\rho_{\max})}\log\frac{4|\mathcal{V}_{\mathrm{obs}}|^{2}l_{\max}^{2}}{\eta}\label{ieq:rsnj2}
	\end{align}
	where
	\begin{align}
		g(x,\rho_{\min},\rho_{\max})&=\left\{
			\begin{array}{lr}
			\frac{1}{2}(2e^{-\rho_{\max}})^{\log_{2}(x/2)}e^{-\rho_{\max}}(1-e^{-2\rho_{\min}}),&  \quad e^{-2\rho_{\max}}\leq 0.5\\
			e^{-3\rho_{\max}}(1-e^{-2\rho_{\min}}),& \quad e^{-2\rho_{\max}}> 0.5
			\end{array}
			\right.   \\
		\lambda&=\frac{2l_{\max}^{2}e^{\rho_{\max}/l_{\max}}}{\delta_{\min}^{1/l_{\max}}}\quad \kappa=\max\{\sigma_{\max}^{2},\rho_{\min}\},
	\end{align}
	and $c$ is an absolute constant.
\end{theorem}
\begin{proof}[Proof of Theorem ~\ref{theo:rsnjsamplecomp2}]
	It is easy to see by substituting the constants $\lambda$ and $\kappa$ into \eqref{ieq:rsnj1} and \eqref{ieq:rsnj2} that Theorem \ref{theo:rsnjsamplecomp2} implies Theorem \ref{theo:rsnjsamplecomp}, so we provide the proof of Theorem \ref{theo:rsnjsamplecomp2} here.

	Proposition \ref{prop:snjsuff} shows that the probability of learning the wrong tree $\mathbb{P}(\mathcal{E})$ could be bounded as
	\begin{align}\label{ieq:errp}
		\mathbb{P}(\mathcal{E})\leq \mathbb{P}\big(\|\hat{\mathbf{R}}-\mathbf{R}\|_{2}>g(|\mathcal{V}_{\mathrm{obs}}|,\rho_{\min},\rho_{\max})\big)\leq |\mathcal{V}_{\mathrm{obs}}|^{2}f\Big(e^{\rho_{\min}}\frac{g(|\mathcal{V}_{\mathrm{obs}}|,\rho_{\min},\rho_{\max})}{|\mathcal{V}_{\mathrm{obs}}|}\Big).
	\end{align}
	Substituting the expression of $f$ and bounding  the right-hand-side of inequality \eqref{ieq:errp} by $\eta$, we have
	\begin{align}
		n_{2}&\geq \frac{16\lambda^{2}\kappa^{2}|\mathcal{V}_{\mathrm{obs}}|^{2}}{ce^{2\rho_{\min}}g(|\mathcal{V}_{\mathrm{obs}}|,\rho_{\min},\rho_{\max})^{2}}\log\frac{2|\mathcal{V}_{\mathrm{obs}}|^{2}l_{\max}^{2}}{\eta} \quad\mbox{and}\\
		\frac{n_{2}}{n_{1}}&\geq \frac{64\lambda\kappa |\mathcal{V}_{\mathrm{obs}}|}{3e^{\rho_{\min}}g(|\mathcal{V}_{\mathrm{obs}}|,\rho_{\min},\rho_{\max})}\log\frac{4|\mathcal{V}_{\mathrm{obs}}|^{2}l_{\max}^{2}}{\eta}.
	\end{align}
	The proof that $n_{1}=O(\sqrt{n_{2}}/\log n_{2})$ can be derived by following the similar procedures in the proof of Theorem \ref{theo:rrgsamplecomp}.
\end{proof}

\section{Proofs of results in Section~\ref{subsec:rclrg}} \label{app:sec34}

\begin{lemma}\label{lem:mstprop}
	The \ac{mst} of a weighted graph $\mathbb{T}$ has the following properties:
	\begin{itemize}
		\item[(1)] For any cut $C$ of the graph, if the weight of an edge $e$ in the cut-set of $C$ is strictly smaller than the weights of all other edges of the cut-set of $C$, then this edge belongs to all \ac{mst}s of the graph.
		\item[(2)] If $\mathbb{T}^{\prime}$ is a tree of \ac{mst} edges, then we can contract $\mathbb{T}^{\prime}$ into a single vertex while maintaining the invariant that the \ac{mst} of the contracted graph plus $\mathbb{T}^{\prime}$ gives the \ac{mst} for the graph before contraction \cite{pettie2002optimal}.
	\end{itemize}
\end{lemma}

\begin{proof}[Proof of Proposition ~\ref{prop:corctmst}]
	We prove this argument by induction. Choosing any node as the root node, we first prove that the edges which are related to the observed nodes with the largest depth are identified or contracted correctly.

	Since we consider the edges which involve at least one observed node, we only need to discuss the edges formed by two observed nodes and one observed node and one hidden node. We first consider the 
	identification of the edges between two observed nodes.

	\begin{figure}[H]
		\centering\includegraphics[width=0.5\columnwidth,draft=false]{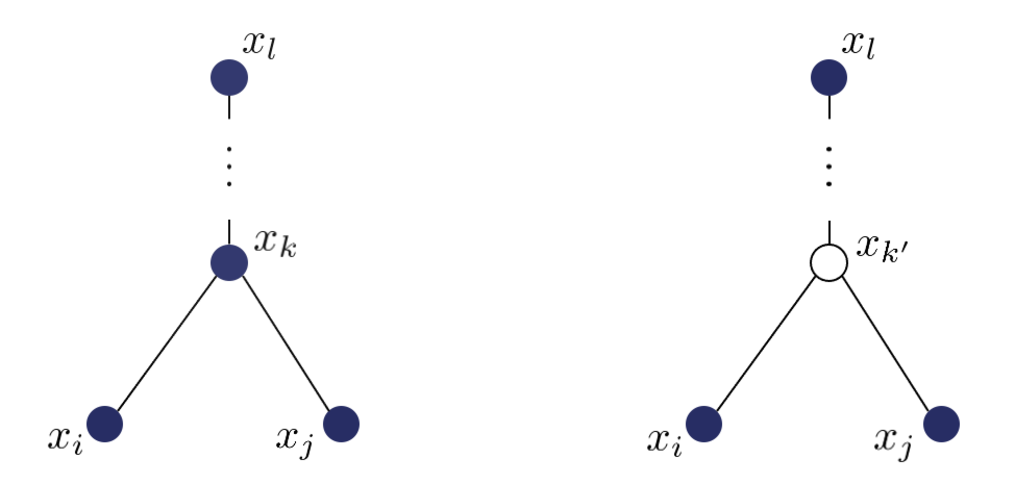}
		\caption{Two kinds of edges related at least one observed node.}
		\label{fig:contract_induc}
	\end{figure}

	To correctly identify the edge $(x_{k},x_{j})$ in Fig.~\ref{fig:contract_induc}, we consider the cut of the graph which splits the nodes into $\{x_{j}\}$ and all the other nodes. Lemma \ref{lem:mstprop} 
	says that the condition that 
	\begin{align}
		\hat{\mathrm{d}}(x_{k},x_{j})<\hat{\mathrm{d}}(x_{l},x_{j})\quad \forall x_{l}\in \mathcal{V}_{\mathrm{obs}},x_{l}\neq x_{k},x_{j}
	\end{align}
	is sufficient to guarantee that this edge is identified correctly. This condition is equivalent to
	\begin{align}
		\Delta_{kj}<\Delta_{lj}+\mathrm{d}(x_{l},x_{k})\quad \forall x_{l}\in \mathcal{V}_{\mathrm{obs}},x_{l}\neq x_{k},x_{j}, 
	\end{align}
	which is guaranteed by choosing $\Delta_{\mathrm{MST}}= \mathrm{d_{ct}}(x_{k};\mathbb{T},\mathcal{V}_{\mathrm{obs}})$.
	
	Furthermore, we need to guarantee that $x_{j}$ is not connected to other nodes except $x_{k}$. We consider the cut of the graph which split the nodes into $\{x_{j},x_{k}\}$ and all the other nodes. Lemma \ref{lem:mstprop} 
	says that the condition that 
	\begin{align}
		\hat{\mathrm{d}}(x_{l},x_{k})<\hat{\mathrm{d}}(x_{l},x_{j})\quad \forall x_{l}\in \mathcal{V}_{\mathrm{obs}},x_{l}\neq x_{k},x_{j}
	\end{align}
	is sufficient to guarantee $x_{j}$ is not connected to other nodes. This condition is equivalent to
	\begin{align}
		\Delta_{lk}<\Delta_{lj}+\mathrm{d}(x_{k},x_{j})\quad \forall x_{l}\in \mathcal{V}_{\mathrm{obs}},x_{l}\neq x_{k},x_{j},
	\end{align}
	which is guaranteed by choosing $\Delta_{\mathrm{MST}}= \mathrm{d_{ct}}(x_{k};\mathbb{T},\mathcal{V}_{\mathrm{obs}})$.

	A similar proof can be used to guarantee $(x_{i},x_{k})$ can be identified correctly. Then we can contract $x_{i},x_{j}$ to $x_{k}$ to form a super node in the subsequent edges identification for Lemma \ref{lem:mstprop}.

	Now we discuss the edges involving one observed node and one hidden node. There are two cases: (i) The hidden node $x_{k^{\prime}}$ should be contracted to either $x_{i}$ 
	or $x_{j}$. (ii) The hidden node $x_{k^{\prime}}$ should be contracted to $x_{l}\in \mathcal{V}_{\mathrm{obs}},x_{l}\neq x_{i},x_{j}$.

	We first consider the case (i). Without loss of generality, we assume that $x_{k^{\prime}}$ should be contracted to $x_{j}$. Contracting $x_{k^{\prime}}$ to $x_{j}$ is equivalent to that 
	$x_{i}$ is not connected to other nodes except $x_{j}$. Lemma \ref{lem:mstprop} shows that 
	\begin{align}\label{cond:hid}
		&\hat{\mathrm{d}}(x_{i},x_{j})<\hat{\mathrm{d}}(x_{i},x_{l})\quad \hat{\mathrm{d}}(x_{j},x_{l})<\hat{\mathrm{d}}(x_{i},x_{l})
		\qquad \forall\, x_{l}\in \mathcal{V}_{\mathrm{obs}},x_{l}\neq x_{i},x_{j}
	\end{align}
	is sufficient to achieve that $x_{i}$ is not connected to other nodes except $x_{j}$. This condition is equivalent to
	\begin{align}
		&\Delta_{ij}+\mathrm{d}(x_{k^{\prime}},x_{j})<\Delta_{lj}+\mathrm{d}(x_{k^{\prime}},x_{l})\quad\mbox{and}\quad \Delta_{jl}+\mathrm{d}(x_{k^{\prime}},x_{j})<\Delta_{il}+\mathrm{d}(x_{k^{\prime}},x_{i})  \nonumber\\
		&\qquad\qquad\qquad\forall x_{l}\in \mathcal{V}_{\mathrm{obs}},x_{l}\neq x_{k},x_{j},
	\end{align}
	which is guaranteed by choosing $\Delta_{\mathrm{MST}}= \mathrm{d_{ct}}(x_{k^{\prime}};\mathbb{T},\mathcal{V}_{\mathrm{obs}})$. Then 
	we can contract $x_{i}$ to $x_{j}$ to form a super node in the subsequent edges identification for Lemma \ref{lem:mstprop}.

	Then we consider the case (ii). Here we need to prove that $x_{k^{\prime}}$ will not be contracted to $x_{i}$ or $x_{j}$. Without loss of generality, we assume that $x_{k^{\prime}}$ 
	is contracted to $x_{l}$, which guaranteed by that there is no edge between $x_{i}$ and $x_{j}$. Lemma \ref{lem:mstprop} shows that 
	\begin{align}\label{cond:hid2}
		\hat{\mathrm{d}}(x_{i},x_{l})<\hat{\mathrm{d}}(x_{i},x_{j})\quad \hat{\mathrm{d}}(x_{j},x_{l})<\hat{\mathrm{d}}(x_{i},x_{j})
	\end{align}
	is sufficient to guarantee that there is no edge between $x_{i}$ and $x_{j}$. This condition is equivalent to
	\begin{align}
		\Delta_{il}+\mathrm{d}(x_{k^{\prime}},x_{l})<\Delta_{ij}+\mathrm{d}(x_{k^{\prime}},x_{j}) \quad \Delta_{jl}+\mathrm{d}(x_{k^{\prime}},x_{l})<\Delta_{ij}+\mathrm{d}(x_{k^{\prime}},x_{i}),
	\end{align}
	which is guaranteed by choosing $\Delta_{\mathrm{MST}}= \mathrm{d_{ct}}(x_{k^{\prime}};\mathbb{T},\mathcal{V}_{\mathrm{obs}})$.

	Assume that all the edges related to the nodes with depths larger than $l$ are identified or contracted correctly. We now consider the edges related 
	to the nodes with depths $l$. For the edges between two observed nodes and edges of case (i) and (ii), similar procedures can be adopted to prove the statements. 
	Here we discuss the case where $x_{l}$ should contract the hidden nodes which are its descendants. Contracting $x_{k^{\prime}}$ to $x_{l}$ is equivalent to that there are 
	edges between $(x_{l},x_{i})$ and $(x_{l},x_{j})$, and there is no other edges related to $x_{i}$ and $x_{j}$. Recall that condition \eqref{cond:hid2} is satisfied by 
	the induction hypothesis. Lemma \ref{lem:mstprop} shows that 
	\begin{align}
		&\hat{\mathrm{d}}(x_{i},x_{l})<\hat{\mathrm{d}}(x_{i},x_{k})\quad \hat{\mathrm{d}}(x_{k},x_{l})<\hat{\mathrm{d}}(x_{k},x_{i})\quad \forall x_{k}\in \mathcal{V}_{\mathrm{obs}},x_{k}\neq x_{i},x_{j},x_{l}\\
		&\hat{\mathrm{d}}(x_{j},x_{l})<\hat{\mathrm{d}}(x_{j},x_{k})\quad \hat{\mathrm{d}}(x_{k},x_{l})<\hat{\mathrm{d}}(x_{k},x_{j})\quad \forall x_{k}\in \mathcal{V}_{\mathrm{obs}},x_{k}\neq x_{i},x_{j},x_{l}
	\end{align}
	is sufficient to guarantee that $x_{k^{\prime}}$ is contracted to $x_{l}$. This condition is equivalent to
	\begin{align}
		&\Delta_{il}<\Delta_{ik}+\mathrm{d}(x_{k},x_{l})\quad \Delta_{kl}<\Delta_{ki}+\mathrm{d}(x_{i},x_{l})\quad \forall x_{k}\in \mathcal{V}_{\mathrm{obs}},x_{k}\neq x_{i},x_{j},x_{l}\\
		&\Delta_{jl}<\Delta_{jk}+\mathrm{d}(x_{k},x_{l})\quad \Delta_{kl}<\Delta_{kj}+\mathrm{d}(x_{j},x_{l})\quad \forall x_{k}\in \mathcal{V}_{\mathrm{obs}},x_{k}\neq x_{i},x_{j},x_{l}
	\end{align}
	which is guaranteed by choosing $\Delta_{\mathrm{MST}}= \mathrm{d_{ct}}(x_{l};\mathbb{T},\mathcal{V}_{\mathrm{obs}})$. Then 
	we can contract $x_{i}$ and $x_{j}$ to $x_{l}$ to form a super node in the subsequent edges identification for Lemma \ref{lem:mstprop}.

	Thus, the results are proved by induction. \end{proof}

\begin{theorem}\label{theo:clrgsamplcomp2}
	If Assumptions \ref{assupleng}--\ref{assupdegree}, \ac{rclrg} constructs the correct latent tree with probability 
	at least $1-\eta$ if
	\begin{align}
		n_{2}&\geq \max\bigg\{\frac{4}{\varepsilon^{2}}\Big(\frac{9}{2}\Big)^{2L_{\mathrm{C}}-2},\frac{1}{\Delta_{\mathrm{MST}}^{2}}\bigg\}\frac{16\lambda^{2} \kappa^{2}}{c}\log\frac{17l_{\max}^{2}s^{L_{\mathrm{C}}-1}|\mathcal{V}_{\mathrm{obs}}|^{3}+l_{\max}^{2}|\mathcal{V}_{\mathrm{obs}}|^{2}}{\eta},\label{ieq:rclrg1}\\
		\frac{n_{2}}{n_{1}}&\geq \max\bigg\{\frac{2}{\varepsilon}\Big(\frac{9}{2}\Big)^{L_{\mathrm{C}}-1},\frac{1}{\Delta_{\mathrm{MST}}}\bigg\}\frac{64\lambda \kappa }{3}\log\frac{34l_{\max}^{2}s^{L_{\mathrm{C}}-1}|\mathcal{V}_{\mathrm{obs}}|^{3}+2l_{\max}^{2}|\mathcal{V}_{\mathrm{obs}}|^{2}}{\eta},\label{ieq:rclrg2}
	\end{align}
	where 
	\begin{align}
		\lambda=\frac{2l_{\max}^{2}e^{\rho_{\max}/l_{\max}}}{\delta_{\min}^{1/l_{\max}}}\quad \kappa=\max\{\sigma_{\max}^{2},\rho_{\min}\} \quad s=d_{\max}^{2}+2d_{\max}^{3}(1+2N_{\tau})\quad \varepsilon=\frac{\rho_{\min}}{2}, 
	\end{align}
	$c$ is an absolute constant, and $L_{\mathrm{C}}$ is the number of iterations of \ac{rclrg} needed to construct the tree.
\end{theorem}
\begin{proof}[Proof of Theorem ~\ref{theo:clrgsamplcomp2}]
	It is easy to see by substituting the constants $\lambda$, $\kappa$, $s$ and $\varepsilon$ into \eqref{ieq:rclrg1} and \eqref{ieq:rclrg2} that Theorem \ref{theo:clrgsamplcomp2} implies Theorem \ref{theo:clrgsamplcomp}, so we provide the proof of Theorem \ref{theo:clrgsamplcomp2} here.

	The \ac{rclrg} algorithm consists of two stages: Calculation of \ac{mst} and implementation of \ac{rrg} on internal nodes. The probability of error of \ac{rclrg} could be decomposed as
	\begin{align}
		\mathbb{P}(\mathcal{E})=\mathbb{P}\big(\mathcal{E}_{\mathrm{MST}}\cup(\mathcal{E}_{\mathrm{MST}}^{c}\cap \mathcal{E}_{\mathrm{RRG}})\big)=\mathbb{P}(\mathcal{E}_{\mathrm{MST}})+\mathbb{P}(\mathcal{E}_{\mathrm{MST}}^{c}\cap \mathcal{E}_{\mathrm{RRG}})\leq \mathbb{P}(\mathcal{E}_{\mathrm{MST}})+\mathbb{P}(\mathcal{E}_{\mathrm{RRG}})\nonumber
	\end{align}
	We define the correct event of calculation of the \ac{mst} as
	\begin{align}
		\mathcal{C}_{\mathrm{MST}}=\bigcap_{x_{i},x_{j}\in \mathcal{V}_{\mathrm{obs}}} \Big\{\big|\hat{\mathrm{d}}(x_{i},x_{j})-\mathrm{d}(x_{i},x_{j})\big|<\frac{\Delta_{\mathrm{MST}}}{2}\Big\}=\bigcap_{x_{i},x_{j}\in \mathcal{V}_{\mathrm{obs}}} \mathcal{C}_{ij}
	\end{align}
	Proposition \ref{prop:corctmst} shows that
	\begin{align}
		\mathbb{P}(\mathcal{E}_{\mathrm{MST}})&\leq 1- \mathbb{P}(\mathcal{C}_{\mathrm{MST}})=\mathbb{P}\big((\bigcap_{x_{i},x_{j}\in \mathcal{V}_{\mathrm{obs}}} \mathcal{C}_{ij})^{c}\big)\\
		&=\mathbb{P}\bigg(\bigcup_{x_{i},x_{j}\in \mathcal{V}_{\mathrm{obs}}} \mathcal{C}_{ij}^{c}\bigg)\leq \sum_{x_{i},x_{j}\in \mathcal{V}_{\mathrm{obs}}}\mathbb{P}(\mathcal{C}_{ij}^{c})\leq \binom{|\mathcal{V}_{\mathrm{obs}}|}{2} f\Big(\frac{\Delta_{\mathrm{MST}}}{2}\Big)
	\end{align}
	We define the event that \ac{rrg} yields the correct subtree based on $\mathrm{nbd}[x_{i},\mathbb{T}]$
	\begin{align}
		\mathcal{C}_{\mathrm{RRG}}&=\bigcap_{x_{i}\in \mathrm{Int}(\mathrm{MST}(\mathcal{V}_{\mathrm{obs}};\mathbf{\hat{D}}))} \{\text{Output of RRG is correct with input }\mathrm{nbd}[x_{i},\mathbb{T}]\} \\
		&=\bigcap_{x_{i}\in \mathrm{Int}(\mathrm{MST}(\mathcal{V}_{\mathrm{obs}};\mathbf{\hat{D}}))} \mathcal{C}_{i}
	\end{align}
	Then we have
	\begin{align}
		\!\mathbb{P}(\mathcal{E}_{\mathrm{RRG}}) =1\!-\!\mathbb{P}(\mathcal{C}_{\mathrm{RRG}}) \! = \! \mathbb{P}\bigg(\Big(\bigcap_{x_{i}\in \mathrm{Int}(\mathrm{MST}(\mathcal{V}_{\mathrm{obs}};\mathbf{\hat{D}}))}  \! \mathcal{C}_{i}\Big)^{c}\bigg)=\mathbb{P}\bigg(\bigcup_{x_{i}\in \mathrm{Int}(\mathrm{MST}(\mathcal{V}_{\mathrm{obs}};\mathbf{\hat{D}}))} \! \mathcal{C}_{i}^{c}\bigg).
	\end{align}

	By defining $L_{\mathrm{C}}=\lceil \frac{\mathrm{Deg}(\mathrm{MST}(\mathcal{V}_{\mathrm{obs}};\mathbf{\hat{D}}))}{2}-1\rceil$, we have
	\begin{align}
		\mathbb{P}(\mathcal{E})&\leq \mathbb{P}(\mathcal{E}_{\mathrm{MST}})+\mathbb{P}(\mathcal{E}_{\mathrm{RRG}}) \\
		&\leq\binom{|\mathcal{V}_{\mathrm{obs}}|}{2} f\Big(\frac{\Delta_{\mathrm{MST}}}{2}\Big)+\big(|\mathcal{V}_{\mathrm{obs}}|-2\big)\Bigg\{\bigg[4\Big(1+\frac{|\mathcal{V}_{\mathrm{obs}}|^{2}}{2(s-1)}\Big)+\frac{3}{2}(d_{\max}-1)\Big(1+\frac{|\mathcal{V}_{\mathrm{obs}}|}{2(s-1)}\Big)\bigg]\nonumber\\
		&\times h^{(L_{\mathrm{C}}-1)}\Big(\frac{2\rho_{\min}-\varepsilon}{4}\Big)+\bigg[3\Big(1+\frac{|\mathcal{V}_{\mathrm{obs}}|^{2}}{4(s-1)}\Big)+2(d_{\max}-1)\Big(1+\frac{|\mathcal{V}_{\mathrm{obs}}|^3}{8(s-1)}\Big)\bigg]h^{(L_{\mathrm{C}}-1)}(\frac{\varepsilon}{4})\Bigg\} 
	\end{align}
	To derive the sufficient conditions of $\mathbb{P}(\mathcal{E})\leq \eta$, we consider the following conditions
	\begin{align}
		\mathbb{P}(\mathcal{E}_{\mathrm{MST}})\leq (1-r)\eta \quad\mbox{and}\quad \mathbb{P}(\mathcal{E}_{\mathrm{RRG}})\leq r\eta \quad\mbox{for some}\quad r\in(0,1)
	\end{align}
	Following the same calculations with inequalities \eqref{eqn:sampcomplex}, we have
	\begin{align}
		n_{2}&\geq \max\bigg\{\frac{64\lambda^{2} \kappa^{2}}{c\varepsilon^{2}}\Big(\frac{9}{2}\Big)^{2L_{\mathrm{C}}-2}\log\frac{17l_{\max}^{2}s^{L_{\mathrm{C}}-1}|\mathcal{V}_{\mathrm{obs}}|^{3}}{r\eta},\frac{16\lambda^{2}\kappa^{2}}{c\Delta_{\mathrm{MST}}^{2}}\log\frac{l_{\max}^{2}|\mathcal{V}_{\mathrm{obs}}|^{2}}{(1-r)\eta}\bigg\}\\
		\frac{n_{2}}{n_{1}}&\geq \max\bigg\{\frac{128\lambda \kappa }{3\varepsilon}\Big(\frac{9}{2}\Big)^{L_{\mathrm{C}}-1}\log\frac{34l_{\max}^{2}s^{L_{\mathrm{C}}-1}|\mathcal{V}_{\mathrm{obs}}|^{3}}{r\eta},\frac{64\lambda \kappa }{3\Delta_{\mathrm{MST}}}\log\frac{2l_{\max}^{2}|\mathcal{V}_{\mathrm{obs}}|^{2}}{(1-r)\eta}\bigg\}
	\end{align}
	By choosing $r=\frac{17s^{L_{\mathrm{C}}-1}|\mathcal{V}_{\mathrm{obs}}|^{3}}{17s^{L_{\mathrm{C}}-1}|\mathcal{V}_{\mathrm{obs}}|^{3}+|\mathcal{V}_{\mathrm{obs}}|^{2}}$, we have
	\begin{align}
		n_{2}&\geq \max\bigg\{\frac{4}{\varepsilon^{2}}\Big(\frac{9}{2}\Big)^{2L_{\mathrm{C}}-2},\frac{1}{\Delta_{\mathrm{MST}}^{2}}\bigg\}\frac{16\lambda^{2} \kappa^{2}}{c}\log\frac{17l_{\max}^{2}s^{L_{\mathrm{C}}-1}|\mathcal{V}_{\mathrm{obs}}|^{3}+l_{\max}^{2}|\mathcal{V}_{\mathrm{obs}}|^{2}}{\eta},\\
		\frac{n_{2}}{n_{1}}&\geq \max\bigg\{\frac{2}{\varepsilon}\Big(\frac{9}{2}\Big)^{L_{\mathrm{C}}-1},\frac{1}{\Delta_{\mathrm{MST}}}\bigg\}\frac{64\lambda \kappa }{3}\log\frac{34l_{\max}^{2}s^{L_{\mathrm{C}}-1}|\mathcal{V}_{\mathrm{obs}}|^{3}+2l_{\max}^{2}|\mathcal{V}_{\mathrm{obs}}|^{2}}{\eta}
	\end{align}

	Following a similar proof as that for \ac{rrg}, we claim that
	\begin{align}
		n_{2}&\geq \max\bigg\{\frac{4}{\varepsilon^{2}}\Big(\frac{9}{2}\Big)^{2L_{\mathrm{C}}-2},\frac{1}{\Delta_{\mathrm{MST}}^{2}}\bigg\}\frac{16\lambda^{2} \kappa^{2}}{c}\log\frac{17l_{\max}^{2}s^{L_{\mathrm{C}}-1}|\mathcal{V}_{\mathrm{obs}}|^{3}+l_{\max}^{2}|\mathcal{V}_{\mathrm{obs}}|^{2}}{\eta},\\
		n_{1}&=O\Big(\frac{\sqrt{n_{2}}}{\log n_{2}}\Big)
	\end{align}
	are sufficient to guarantee $\mathbb{P}(\mathcal{E})\leq\eta$.
\end{proof}

\section{Discussions and Proofs of results in Section~\ref{subsec:comp}} \label{app:table}
In this section, we provide more discussions of the results in Table \ref{table:samplecompare}. We also provide the proofs of results listed in Table \ref{table:samplecompare}.

The sample complexities of \ac{rrg} and \ac{rclrg}  are achieved w.h.p., since the number of iterations $L_{\mathrm{R}}$ and $L_{\mathrm{C}}$ depend on the 
quality of the estimates of the information distances. The parameter $t$ for \ac{rsnj} scales as $O(\frac{1}{l_{\max}}+\log |\mathcal{V}_{\mathrm{obs}}|)$. 
For the dependence on $\mathrm{Diam}(\mathbb{T})$, \ac{rrg} and \ac{rsnj} have the worst performance. This is because \ac{rrg} constructs new hidden nodes and estimates the information distances related to them in each iteration (or layer), which results in more severe error propagation on larger and deeper graphs. In contrast, our impossibility result in Theorem \ref{theo:converse} suggests that \ac{rnj} has the optimal dependence on $\mathrm{Diam}(\mathbb{T})$. \ac{rclrg} also has the optimal dependence on the diameter of 
graphs on \ac{hmm}, which demonstrates that the Chow-Liu initialization procedure greatly reduces the sample complexity from $O\big((\frac{9}{2})^{\mathrm{Diam}(\mathbb{T})}\big)$ to $O\big(\log\mathrm{Diam}(\mathbb{T})\big)$. 
Since the dependence on $\rho_{\max}$ only relies on the parameters, the dependence of $\rho_{\max}$ of all these algorithms remains the same for graphical models with different underlying structures. 
\ac{rrg}, \ac{rclrg} and \ac{rnj} have the same dependence $O(e^{2\frac{\rho_{\max}}{l_{\max}}})$, while \ac{rsnj} has a worse dependence on $\rho_{\max}$. 
\subsection{Proofs of entries in Table~\ref{table:samplecompare}}
\paragraph{Double-binary tree}
For \ac{rrg}, the number of iterations  needed to construct the tree $L_{\mathrm{R}}=\frac{1}{2}(\mathrm{Diam} (\mathbb{T} ) -1)$. Thus, the sample complexity of \ac{rrg} is $O\big(e^{2\frac{\rho_{\max}}{l_{\max}}}(\frac{9}{2})^{\mathrm{Diam}(\mathbb{T})}\big)$.

For \ac{rclrg}, as mentioned previously, the \ac{mst} can be obtained by contracting the hidden nodes to its closest observed node. For example, the \ac{mst} of the double-binary tree with $\mathrm{Diam}(\mathbb{T})=5$ could be derived by 
contracting hidden nodes as Fig.~\ref{fig:clrgdb}. Then $L_{\mathrm{C}}=\lceil \frac{\mathrm{Diam}(\mathbb{T})+1}{4}\rceil-1$, and the number of observed nodes is $|\mathcal{V}_{\mathrm{obs}}|=2^{\frac{\mathrm{Diam}(\mathbb{T})+1}{2}}$. 
Thus, the sample complexity is $O\big(e^{2\frac{\rho_{\max}}{l_{\max}}}(\frac{9}{2})^{\frac{\mathrm{Diam}(\mathbb{T})}{2}}\big)$.
\begin{figure}[H]
	\centering\includegraphics[width=0.5\columnwidth,draft=false]{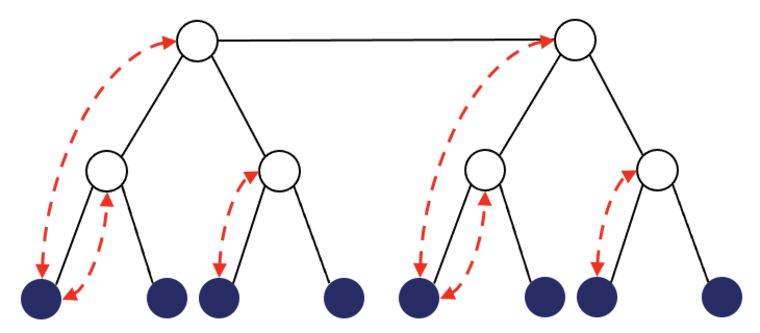}
	\caption{The contraction of hidden nodes in double-binary trees.}
	\label{fig:clrgdb}
\end{figure}

For \ac{rsnj}, the number of observed nodes is $|\mathcal{V}_{\mathrm{obs}}|=2^{\frac{\mathrm{Diam}(\mathbb{T})+1}{2}}$, so the sample complexity is $O\big(e^{2t\rho_{\max}}\mathrm{Diam}(\mathbb{T})\big)$.

For \ac{rnj}, the number of observed nodes is $|\mathcal{V}_{\mathrm{obs}}|=2^{\frac{\mathrm{Diam}(\mathbb{T})+1}{2}}$, so the sample complexity is $O\big(e^{2\frac{\rho_{\max}}{l_{\max}}}\mathrm{Diam}(\mathbb{T})\big)$.

\paragraph{HMM}
For \ac{rrg},   the number of iterations  needed to construct the tree $L_{\mathrm{R}}=\lceil\frac{\mathrm{Diam}{\mathbb{T}}}{2}-1\rceil$. Thus, the sample complexity of \ac{rrg} is $O\big(e^{2\frac{\rho_{\max}}{l_{\max}}}(\frac{9}{2})^{\mathrm{Diam}(\mathbb{T})}\big)$.

For \ac{rclrg}, \ac{mst} could be derived as contracting hidden nodes as shown in Fig. \ref{fig:clrghmm}. Then $L_{\mathrm{C}}=1$ and $|\mathcal{V}_{\mathrm{obs}}|=\mathrm{Diam}(\mathbb{T})+1$. The sample complexity 
is thus $O\big(e^{2\frac{\rho_{\max}}{l_{\max}}}\log\mathrm{Diam}(\mathbb{T})\big)$.
\begin{figure}[H]
	\centering\includegraphics[width=0.5\columnwidth,draft=false]{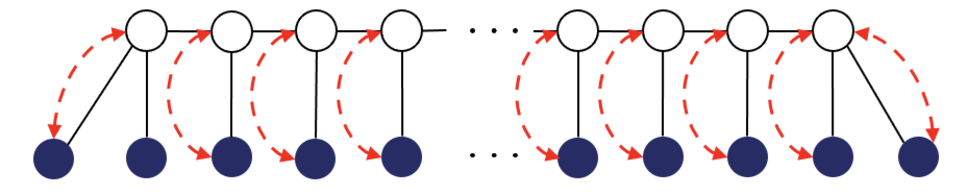}
	\caption{The contraction of hidden nodes in HMMs.}
	\label{fig:clrghmm}
\end{figure}

For \ac{rsnj}, the number of observed nodes is $|\mathcal{V}_{\mathrm{obs}}|=\mathrm{Diam}(\mathbb{T})+1$, so the sample complexity is $O\big(e^{2t\rho_{\max}}\log \mathrm{Diam}(\mathbb{T})$.

For \ac{rnj}, the number of observed nodes is $|\mathcal{V}_{\mathrm{obs}}|=\mathrm{Diam}(\mathbb{T})+1$, so the sample complexity is $O\big(e^{2\frac{\rho_{\max}}{l_{\max}}}\log\mathrm{Diam}(\mathbb{T})\big)$.

\paragraph{Full $m$-tree}
For \ac{rrg}, the number of iterations  needed to construct the tree $L_{\mathrm{R}}=\frac{1}{2} \mathrm{Diam}(\mathbb{T})$. Thus, the sample complexity of \ac{rrg} is $O\big(e^{2\frac{\rho_{\max}}{l_{\max}}}(\frac{9}{2})^{\mathrm{Diam}(\mathbb{T})}\big)$.

For \ac{rclrg}, the \ac{mst} can be derived by contracting hidden nodes as shown in Fig.~\ref{fig:clrgfull}. Then $L_{\mathrm{C}}=2$ and $|\mathcal{V}_{\mathrm{obs}}|=m^{\mathrm{Diam}(\mathbb{T})/2}$. 
Thus, its sample complexity is $O\big(e^{2\frac{\rho_{\max}}{l_{\max}}}\mathrm{Diam}(\mathbb{T})\big)$.
\begin{figure}[H]
	\centering\includegraphics[width=0.5\columnwidth,draft=false]{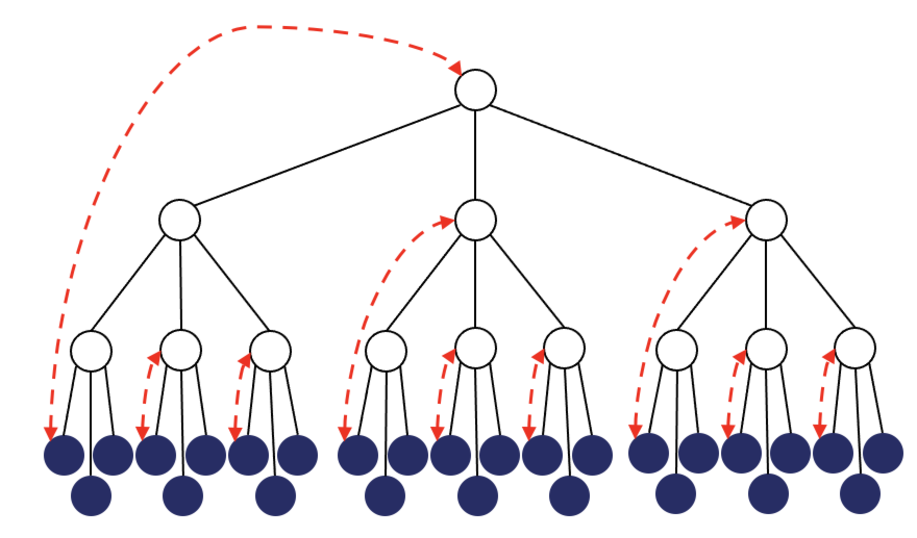}
	\caption{The contraction of hidden nodes in full $m$-trees.}
	\label{fig:clrgfull}
\end{figure}

\paragraph{Double star}
For \ac{rrg}, the number of iterations  needed to construct the tree $L_{\mathrm{R}}=1$. Thus, the sample complexity of \ac{rrg} is $O(e^{2\frac{\rho_{\max}}{l_{\max}}})$.

For \ac{rclrg}, the maximum number of iterations over each \ac{rrg} step (over each internal node of the constructed Chow-Liu tree)  in \ac{rclrg}   is. $L_{\mathrm{C}}=1$ and $|\mathcal{V}_{\mathrm{obs}}|=2d_{\max}$, so the sample complexity of \ac{rclrg} is   $O\big(e^{2\frac{\rho_{\max}}{l_{\max}}}\log d_{\max}\big)$.

%\section{More simulation results complementing those in Section~\ref{sec:simu}}\label{app:sim}
\section{Additional numerical details and results}\label{app:num}
\subsection{Standard deviations of results in Fig.~\ref{fig:simu}}\label{app:stddev}

 We first report the standard deviations of the results presented in Fig.~\ref{fig:simu} in the main paper.  All results are averaged over 100 independent runs.

Constant magnitude corruptions (Fig.~\ref{fig:simu}(a))
\begin{table}[H]
	\scriptsize
	\centering
	%\resizebox{\textwidth}{15mm}{
	\begin{tabular}{|c|c|c|c|c|c|c|c|}
	\hline
	\diagbox{Algorithm}{$\sigma$/($\sigma$/AVG)$\times 100$}{\# Samples}&500 & 1000	& 1500 & 2000 & 5000 & 10000 & 20000\\
	\hline
	\ac{rrg}&9.7/9.3 & 4.4/5.0 & 3.7/4.5 & 4.0/5.0 & 5.0/6.8 & 14.4/24.1 & 21.0/75.0\\
	\hline
	\ac{rsnj}&3.3/3.8 & 3.0/7.0 & 3.9/28.8 & 0.3/703.5 & 0.0/0.0 & 0.0/0.0 & 0.0/0.0\\
	\hline
	\ac{rclrg}&2.1/52.0& 0.5/229.1&0.0/0.0&0.0/0.0 &0.0/0.0&0.0/0.0&0.0/0.0\\
	\hline
	\ac{rnj}&5.6/4.3& 9.0/7.1 & 12.3/10.0 & 17.1/15.1 & 28.4/28.3& 35.4/47.9 & 32.2/68.1\\
	\hline
	\ac{rg}&9.2/9.5& 8.8/8.8 & 8.3/8.3 & 7.8/7.8 & 5.7/6.2 & 4.1/4.9 & 1.9/2.3\\
	\hline
	\ac{snj}&0.4/0.3 & 0.6/0.4 & 1.4/0.9 & 2.7/1.8 & 3.2/2.6& 3.6/3.7& 3.2/5.9\\
	\hline
	\ac{clrg}&3.0/2.2&	3.5/2.6& 3.4/2.5& 4.0/3.0& 11.2/19.9& 4.7/22.4 & 2.1/43.5\\
	\hline
	\ac{nj}&1.8/1.3&2.2/1.6&2.2/1.5&3.1/2.3&6.0/4.8&11.5/9.8&17.9/16.35\\
	\hline
	\end{tabular}
	%}
	\caption{The standard deviations and standard deviations divided by the means of  the Robinson-Foulds distances for different algorithms}
	\label{table:sd_const}
\end{table}

Uniform corruptions (Fig.~\ref{fig:simu}(b))
\begin{table}[H]
	\scriptsize
	\centering
	%\resizebox{\textwidth}{15mm}{
	\begin{tabular}{|c|c|c|c|c|c|c|c|}
	\hline
	\diagbox{Algorithm}{$\sigma$/($\sigma$/AVG)$\times 100$}{\# Samples}&500 & 1000	& 1500 & 2000 & 5000 & 10000 & 20000\\
	\hline
	\ac{rrg}&4.5/5.0 & 3.3/4.0 & 3.8/4.6 & 3.1/4.0 & 4.3/5.9 & 10.9/17.4 & 23.0/103.2\\
	\hline
	\ac{rsnj}&3.3/4.0 & 2.9/6.7 & 5.0/30.1 & 0.7/230.3 & 0.0/0.0 & 0.0/0.0 & 0.0/0.0\\
	\hline
	\ac{rclrg}&4.6/9.7&	2.5/35.8&0.6/197.1&0.1/1971.0&0.0/0.0&0.0/0.0&0.0/0.0\\
	\hline
	\ac{rnj}&9.2/6.9& 11.5/9.4 & 16.4/14.1 & 18.7/16.6 & 31.1/35.0& 31.4/50.9 & 33.7/74.5\\
	\hline
	\ac{rg}&9.2/9.0& 9.8/9.6 & 8.0/7.8 & 8.1/8.0 & 9.0/9.0 & 7.8/7.4 & 5.9/6.1\\
	\hline
	\ac{snj}&0.0/0.0 & 0.0/0.0 & 0.0/0.0 & 0.2/0.1 & 0.5/0.3& 4.4/3.0& 3.5/3.0\\
	\hline
	\ac{clrg}&3.3/2.4&	3.4/2.5& 3.3/2.4& 3.5/2.5& 3.0/2.2& 6.0/4.5 & 8.0/17.5\\
	\hline
	\ac{nj}&1.7/1.2&1.9/1.3&2.0/1.4&2.0/1.4&2.1/1.5&3.9/2.8&6.1/4.8\\
	\hline
	\end{tabular}
	%}
	\caption{The standard deviations and standard deviations divided by the means of  the Robinson-Foulds  distances for different algorithms}
	\label{table:sd_indeplarge}
\end{table}
\ac{hmm} corruptions  (Fig.~\ref{fig:simu}(c))
\begin{table}[H]
	\scriptsize
	\centering
	%\resizebox{\textwidth}{15mm}{
	\begin{tabular}{|c|c|c|c|c|c|c|c|}
	\hline
	\diagbox{Algorithm}{$\sigma$/($\sigma$/AVG)$\times 100$}{\# Samples}&500 & 1000	& 1500 & 2000 & 5000 & 10000 & 20000\\
	\hline
	\ac{rrg}&4.0/4.5 & 5.3/5.8 & 3.8/4.5 & 3.3/4.0 & 3.4/4.6 & 7.5/10.9 & 21.1/49.6\\
	\hline
	\ac{rsnj}&5.9/6.0 & 3.5/6.4 & 3.4/9.7 & 3.6/35.2 & 0.0/0.0 & 0.0/0.0 & 0.0/0.0\\
	\hline
	\ac{rclrg}&13.1/17.4& 4.9/14.1 & 3.4/29.1 & 1.6/54.6 & 0.0/0.0 & 0.0/0.0 & 0.0/0.0\\
	\hline
	\ac{rnj}&6.7/4.5 & 11.3/8.9 & 12.7/10.5 & 19.2/16.7 & 29.3/30.7& 38.0/48.7 & 32.3/65.4\\
	\hline
	\ac{rg}&9.3/9.1& 8.4/8.3 & 8.7/8.5 & 8.6/8.3 & 9.0/8.8 & 8.6/8.2 & 5.5/5.7\\
	\hline
	\ac{snj}&0.3/0.2 & 0.4/0.3 & 0.4/0.3 & 0.5/0.3 & 2.0/1.2 & 4.8/3.4 & 3.9/3.2\\
	\hline
	\ac{clrg}& 3.4/2.5 & 3.3/2.4 & 3.2/2.3 & 3.1/2.3 & 3.5/2.6& 15.0/12.7 & 8.2/13.7\\
	\hline
	\ac{nj}& 1.8/1.3 & 1.6/1.2 & 2.0/1.4 & 1.9/1.3 & 2.8/2.0 & 4.5/3.3 & 5.7/4.4\\
	\hline
	\end{tabular}
	%}
	\caption{The standard deviations and standard deviations divided by the means of  the Robinson-Foulds  distances for different algorithms}
	\label{table:sd_hmm}
\end{table}
We note that most of the standard deviations (relative to the means) are   reasonably small. However, some entries in Tables~\ref{table:sd_const}--\ref{table:sd_hmm} appear to be rather large, for example $0.5/229.1$. The reason is that the mean value of the errors are already quite small in these cases, so any deviation from the small means result in large standard deviations. This, however, seems unavoidable.

\subsection{More simulation results complementing those in Section~\ref{sec:simu}}\label{app:sim}

In the following more extensive simulations, we consider eight corruption patterns:
\begin{itemize}
    \item Uniform corruptions: Uniform corruptions are independent additive noises in $[-2A,2A]$ and distributed randomly in the data matrix $\mathbf{X}_1^n$.
    \item Constant magnitude corruptions: Constant magnitude corruptions are independent additive noises but taking values in $\{-A,+A\}$ with probability $0.5$ and distributed randomly in $\mathbf{X}_1^n$.
    \item Gaussian corruptions: Gaussian corruptions are independent additive Gaussian noises $\mathcal{N}(0,A^{2})$ and distributed randomly in $\mathbf{X}_1^n$.
    \item \ac{hmm} corruptions: \ac{hmm} corruptions are generated by a \ac{hmm} which shares the same structure as the original \ac{hmm}  but has different parameters. They replace the entries in $\mathbf{X}_1^n$ with the samples generated by the variables in the same positions.
    \item Double binary corruptions: Double binary corruptions are generated by a double binary tree-structured graphical model which shares the same structure as the original double binary graphical model but has different parameters. They replace the entries in $\mathbf{X}_1^n$ with the samples generated by the variables in the same positions.
    \item Gaussian outliers: Gaussian outliers are outliers that  are generated by independent Gaussian random variables distributed as $\mathcal{N}(0,A^{2})$. 
    \item \ac{hmm} outliers: \ac{hmm} outliers are outliers that  are generated by a \ac{hmm} that shares the same structure as the original \ac{hmm}  but has different parameters.
    \item Double binary outliers: Double binary outliers are outliers that  are generated by a double binary tree-structured graphical model which shares the same structure as the original \ac{hmm}  but has different parameters.
\end{itemize}
In all our experiments, the parameter $A$ is set to $60$ and the number of corruptions $n_{1}$ is set to  $100$.

Samples are generated from two graphical models: \ac{hmm} (Fig.~\ref{fig:comp}(b)) and double binary tree (Fig.~\ref{fig:comp}(a)). The dimensions of the random vectors at each node are $l_{\max}=3$. The Robinson-Foulds distance~\cite{robinson1981comparison} between the nominal tree and the estimate and the error rate (zero-one loss) are adopted to measure the performance of learning algorithms. These are computed based on $100$ independent trials. We use the code for \ac{rg} and \ac{clrg} provided by  Choi et al.~\cite{choi2011learning}. All our experiments are run on an  Intel(R) Xeon(R) CPU
E5-2697 v4 @ 2.30 GHz.
\subsubsection{HMM}
Just as in the experiments in  Choi et al.~\cite{choi2011learning}, the diameter of the \ac{hmm} (Fig.~\ref{fig:comp}(b)) is chosen to be $\mathrm{Diam}(\mathbb{T})=80$. The  matrices $(\mathbf{A},\mathbf{\Sigma}_{\mathrm{r}},\mathbf{\Sigma}_{\mathrm{n}})$ are chosen so that the condition in Proposition \ref{prop:homo} are satisfied with $\alpha=1$, and we set $\mathbf{A}$ commutable with $\mathbf{\Sigma}_{\mathrm{r}}$. The information distances between neighboring nodes are chosen to be the same value $0.24$, which implies that $\rho_{\min}=0.24$ and $\rho_{\max}=0.24\cdot\mathrm{Diam}(\mathbb{T})=19.2$.
\begin{figure}[H]
\centering
\subfigure[Robinson-Foulds distances]{
\begin{minipage}[t]{0.48\linewidth}
\centering
\includegraphics[width=2.6in]{hmm_const1_rf.eps}
%\caption{fig2}
\end{minipage}%
}%
\subfigure[Structure recovery error rate]{
\begin{minipage}[t]{0.48\linewidth}
\centering
\includegraphics[width=2.6in]{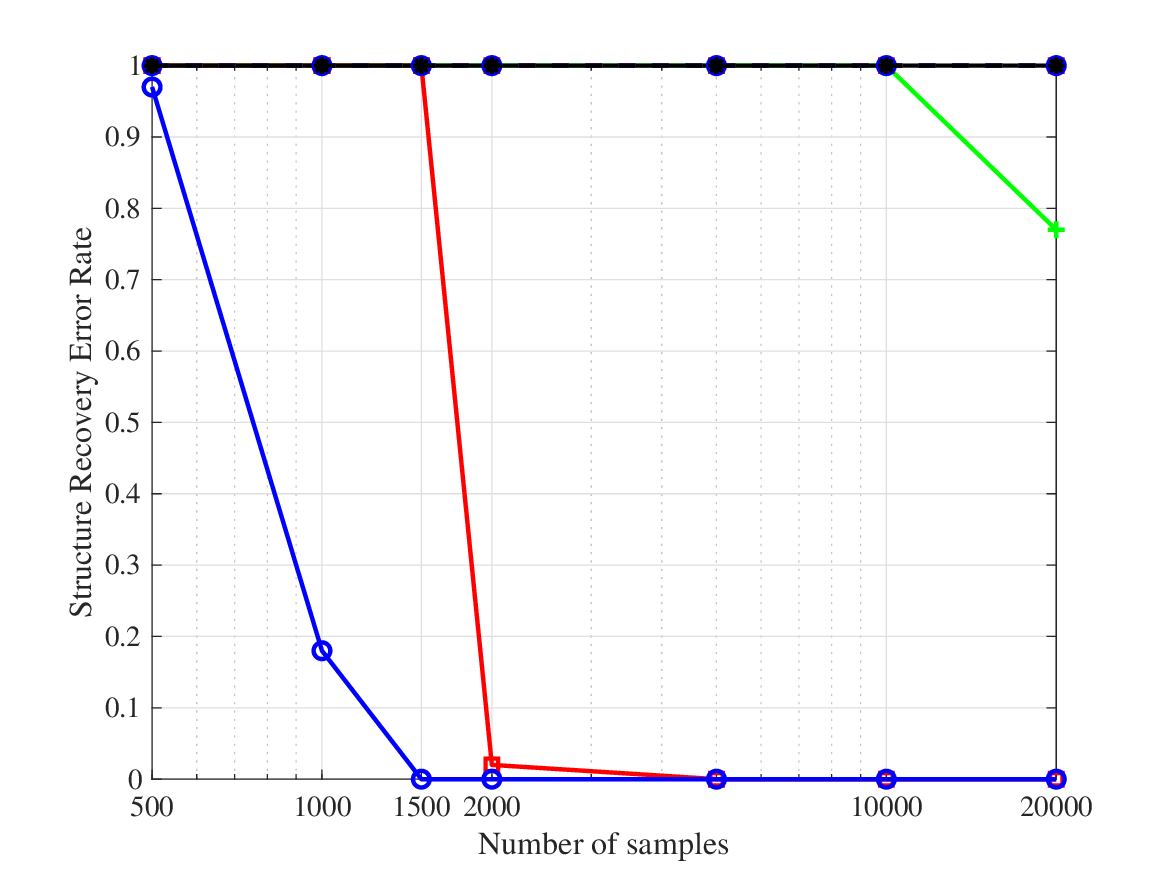}
%\caption{fig1}
\end{minipage}%
\label{fig:hmm_const}
}%
\centering
\caption{Performances of robustified and original learning algorithms with constant magnitude corruptions}
\end{figure}

\begin{figure}[H]
\centering
\subfigure[Robinson-Foulds distances]{
\begin{minipage}[t]{0.48\linewidth}
\centering
\includegraphics[width=2.6in]{hmm_indeplarge1_rf.eps}
%\caption{fig2}
\end{minipage}%
}%
\subfigure[Structure recovery error rate]{
\begin{minipage}[t]{0.48\linewidth}
\centering
\includegraphics[width=2.6in]{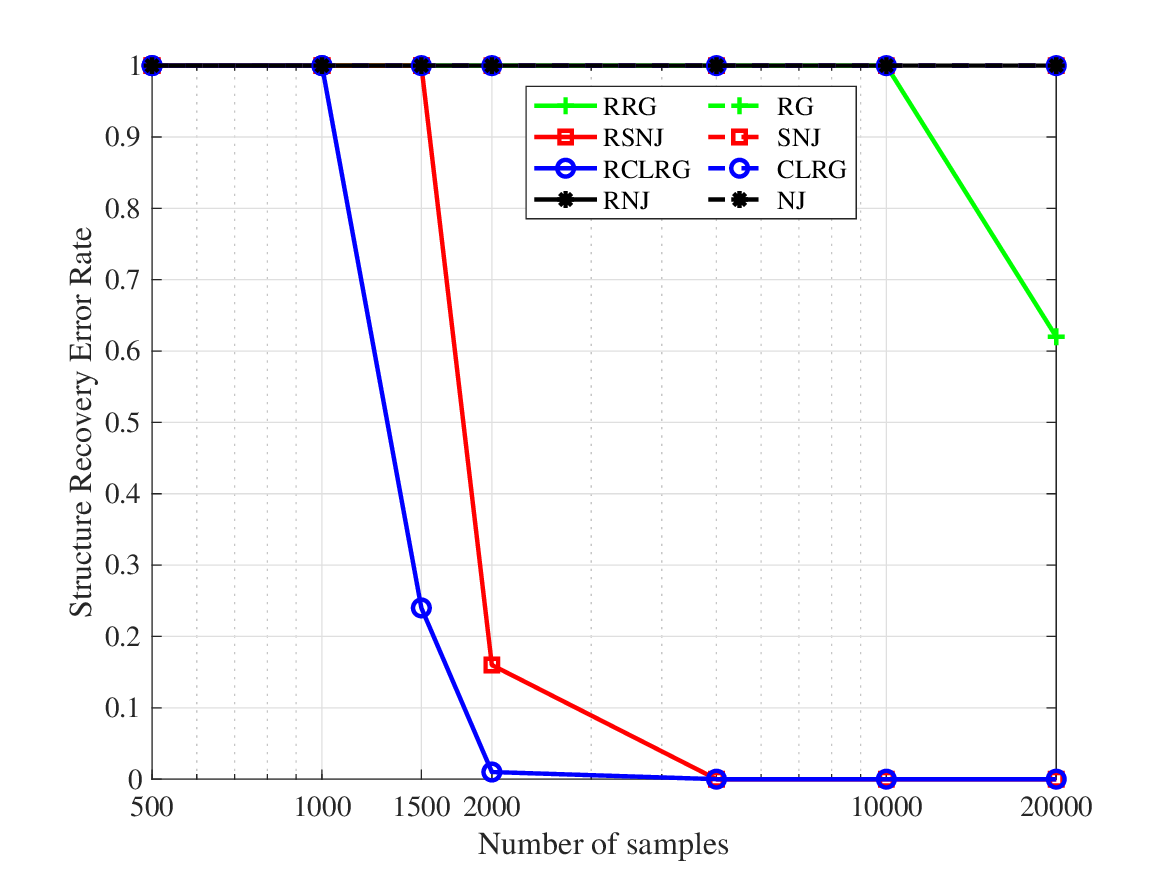}
%\caption{fig1}
\end{minipage}%
}%
\centering
\caption{Performances of robustified and original learning algorithms with uniform corruptions}
\end{figure}

\begin{figure}[H]
\centering
\subfigure[Robinson-Foulds distances]{
\begin{minipage}[t]{0.48\linewidth}
\centering
\includegraphics[width=2.6in]{hmm_hmm2_rf.eps}
%\caption{fig2}
\end{minipage}%
}%
\subfigure[Structure recovery error rate]{
\begin{minipage}[t]{0.48\linewidth}
\centering
\includegraphics[width=2.6in]{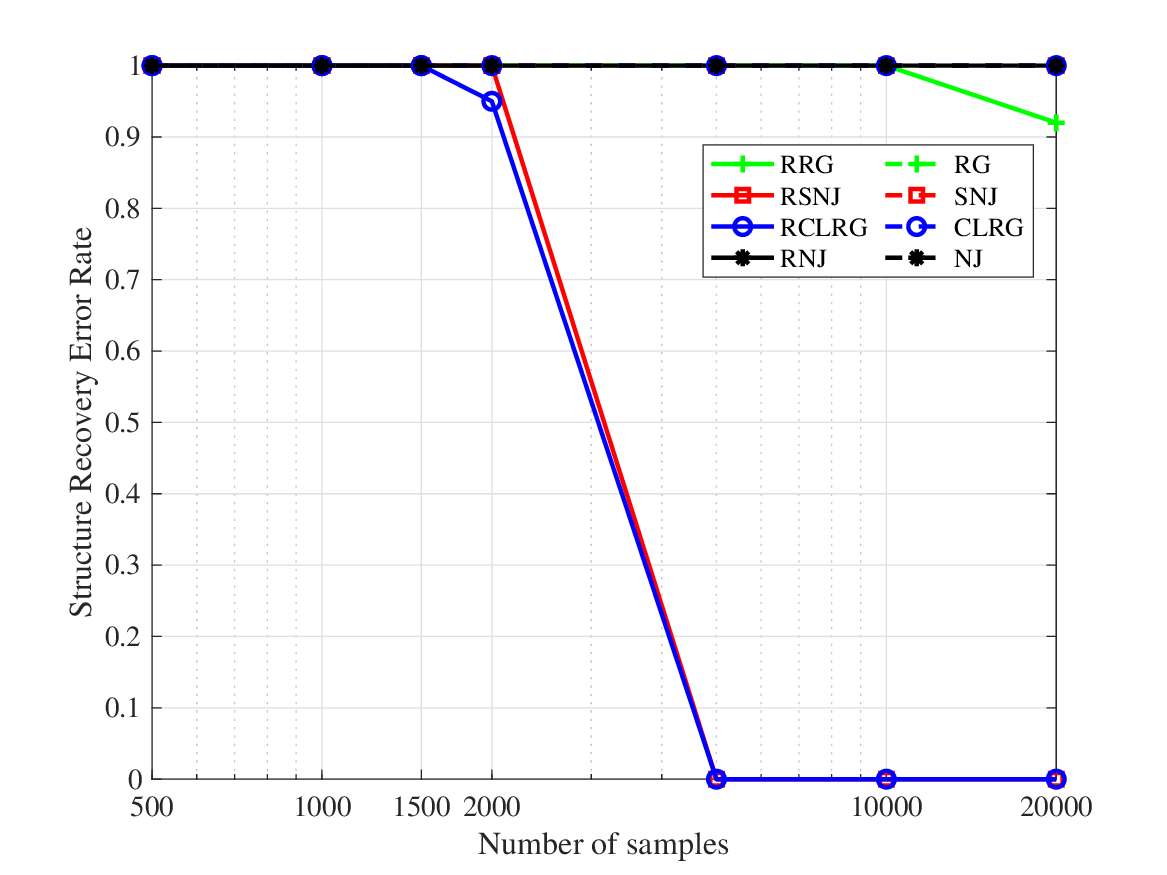}
%\caption{fig1}
\end{minipage}%
}%
\centering
\caption{Performances of robustified and original learning algorithms with \ac{hmm} corruptions}
\end{figure}

\begin{figure}[H]
\centering
\subfigure[Robinson-Foulds distances]{
\begin{minipage}[t]{0.48\linewidth}
\centering
\includegraphics[width=2.6in]{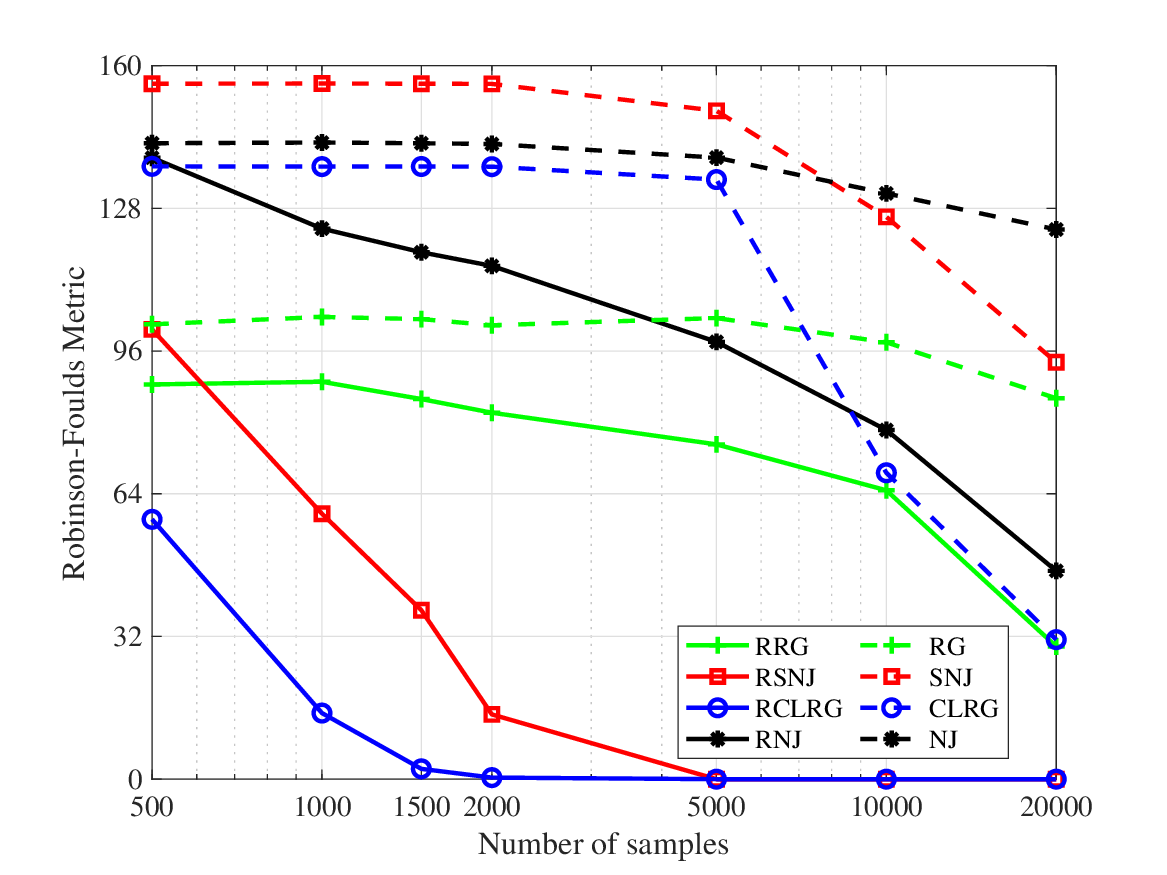}
%\caption{fig2}
\end{minipage}%
}%
\subfigure[Structure recovery error rate]{
\begin{minipage}[t]{0.48\linewidth}
\centering
\includegraphics[width=2.6in]{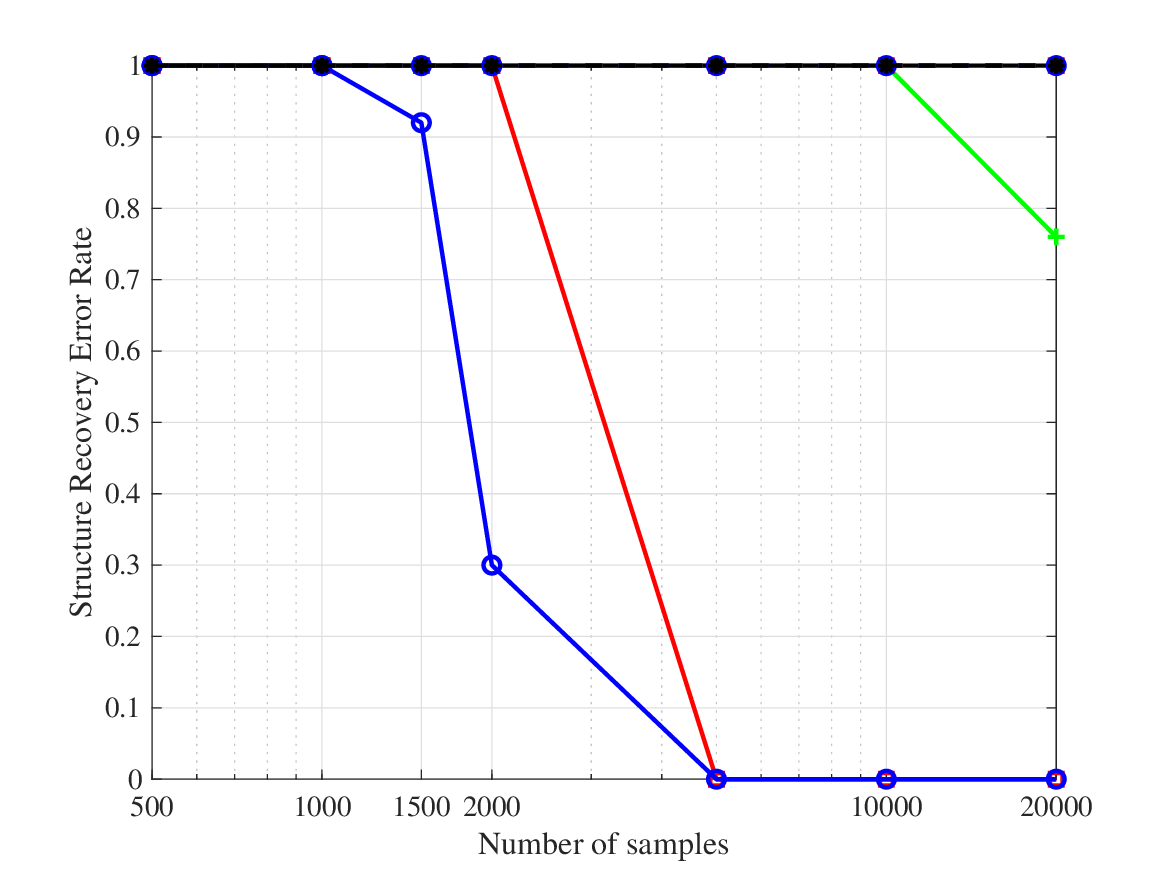}
%\caption{fig1}
\end{minipage}%
}%
\centering
\caption{Performances of robustified and original learning algorithms with Gaussian corruptions}
\end{figure}

\begin{figure}[H]
\centering
\subfigure[Robinson-Foulds distances]{
\begin{minipage}[t]{0.48\linewidth}
\centering
\includegraphics[width=2.6in]{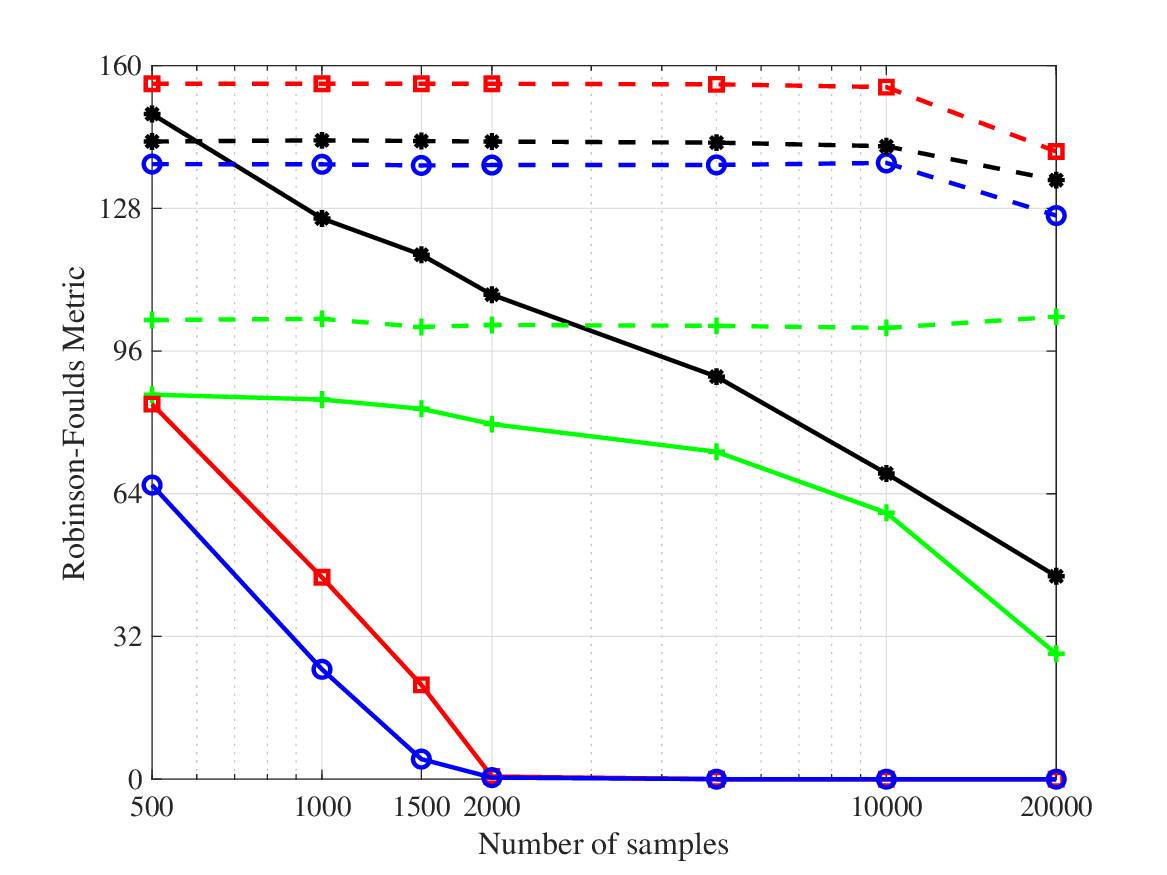}
%\caption{fig2}
\end{minipage}%
}%
\subfigure[Structure recovery error rate]{
\begin{minipage}[t]{0.48\linewidth}
\centering
\includegraphics[width=2.6in]{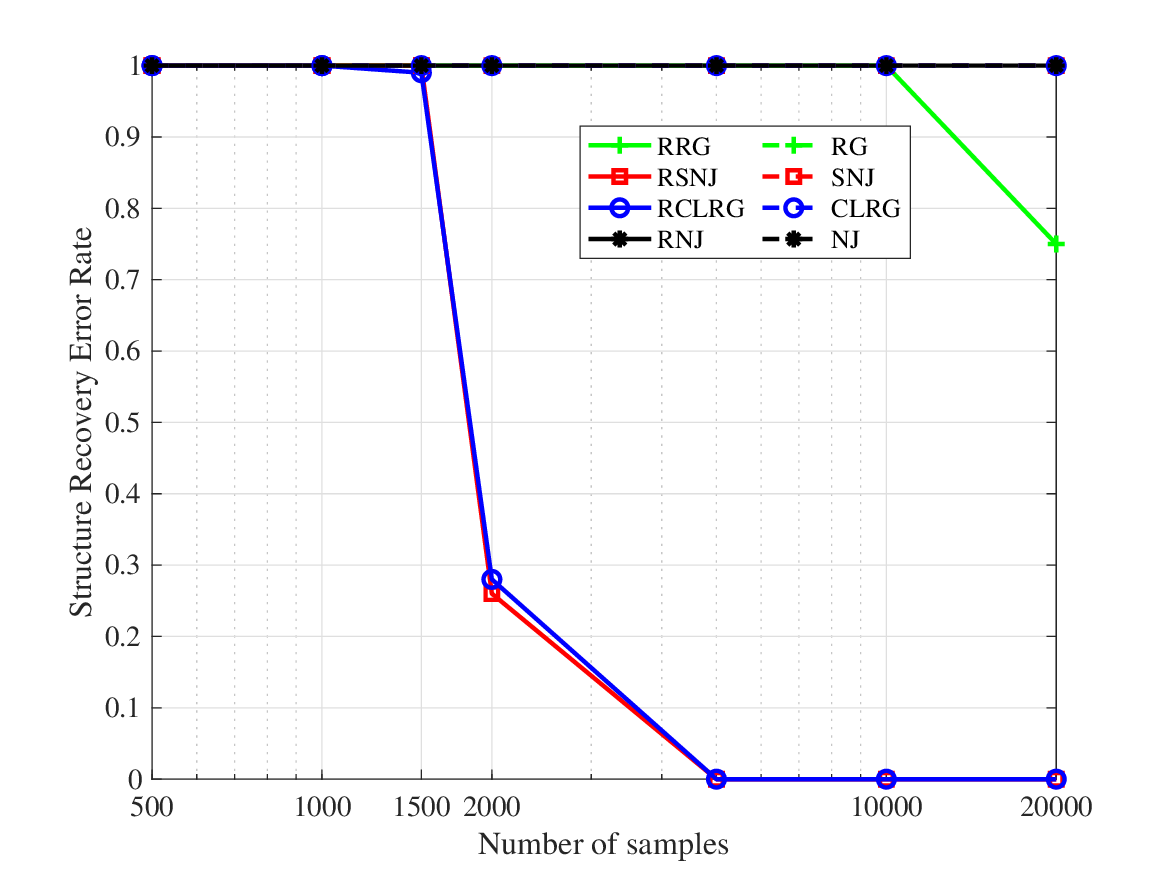}
%\caption{fig1}
\end{minipage}%
}%
\centering
\caption{Performances of robustified and original learning algorithms with double binary corruptions}
\end{figure}

\begin{figure}[H]
\centering
\subfigure[Robinson-Foulds distances]{
\begin{minipage}[t]{0.48\linewidth}
\centering
\includegraphics[width=2.6in]{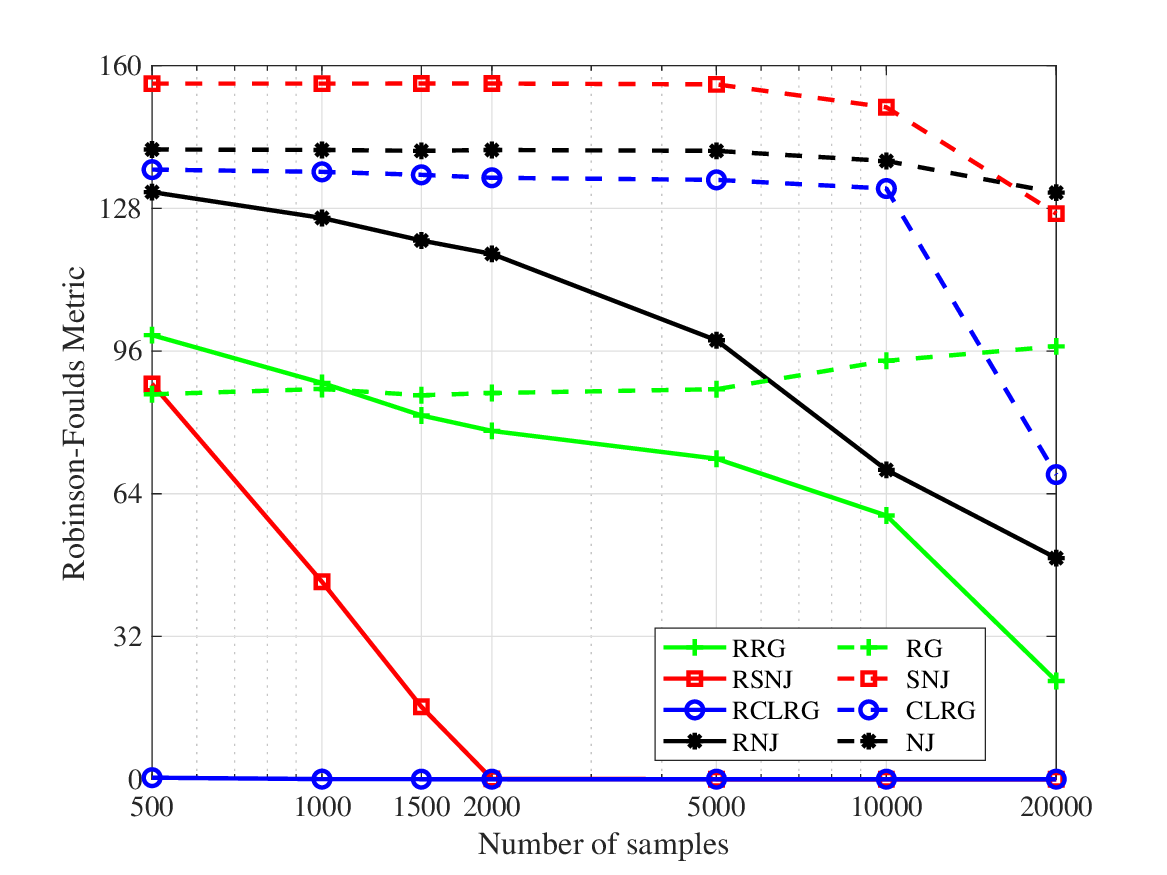}
%\caption{fig2}
\end{minipage}%
}%
\subfigure[Structure recovery error rate]{
\begin{minipage}[t]{0.48\linewidth}
\centering
\includegraphics[width=2.6in]{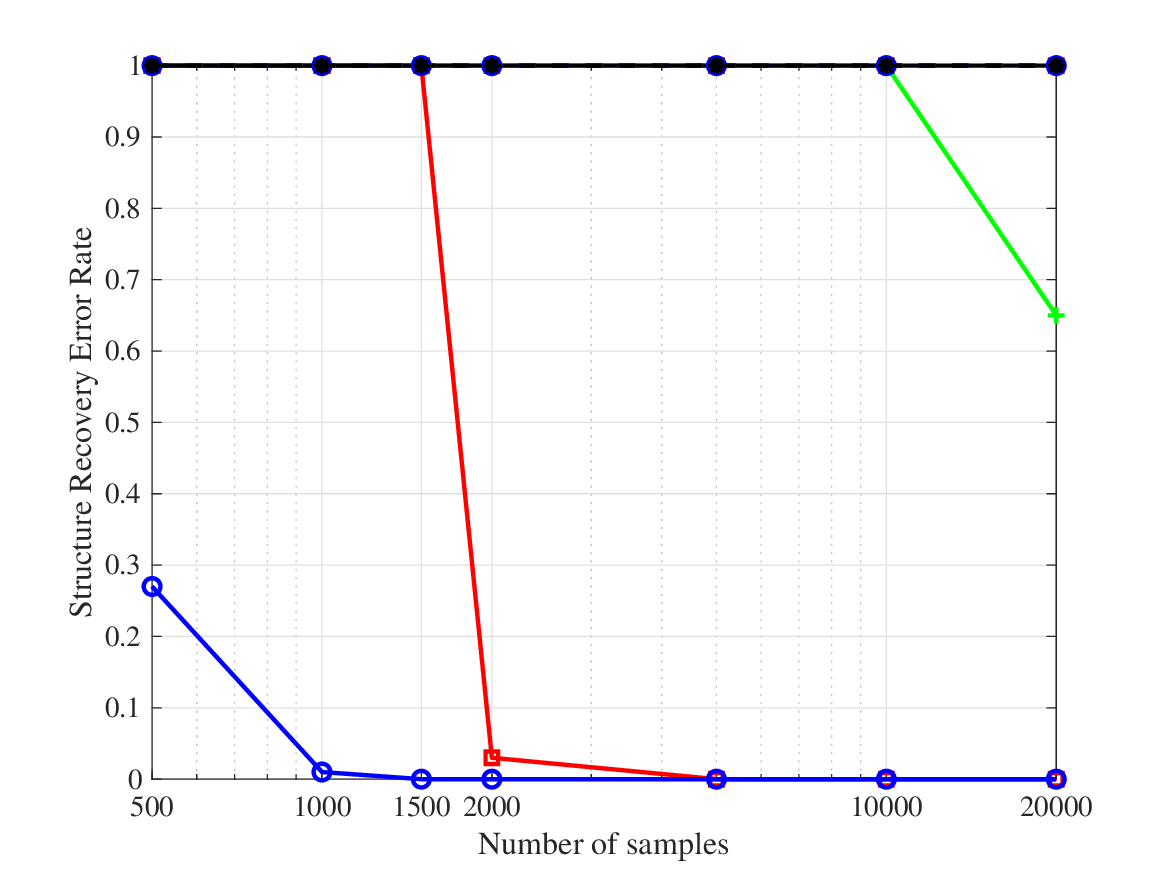}
%\caption{fig1}
\end{minipage}%
}%
\centering
\caption{Performances of robustified and original learning algorithms with Gaussian outliers}
\end{figure}

\begin{figure}[H]
\centering
\subfigure[Robinson-Foulds distances]{
\begin{minipage}[t]{0.48\linewidth}
\centering
\includegraphics[width=2.6in]{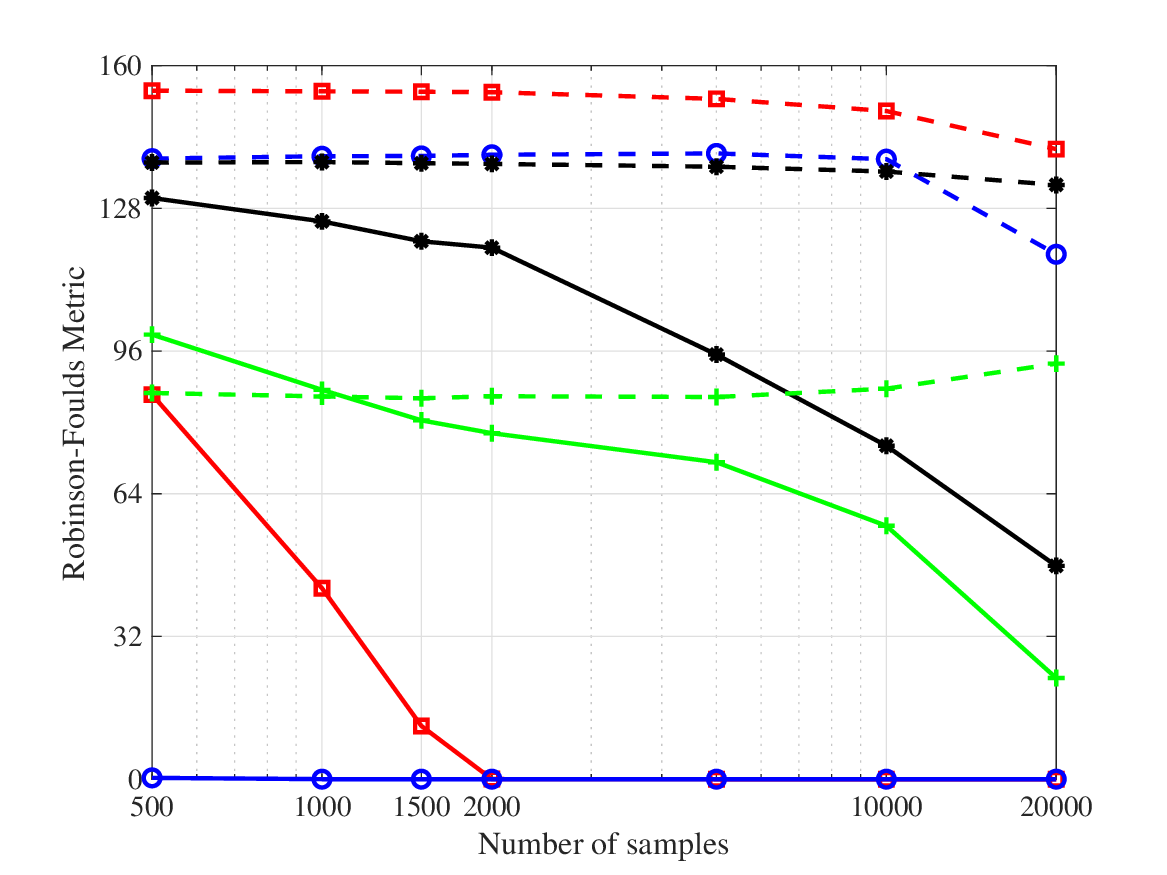}
%\caption{fig2}
\end{minipage}%
}%
\subfigure[Structure recovery error rate]{
\begin{minipage}[t]{0.48\linewidth}
\centering
\includegraphics[width=2.6in]{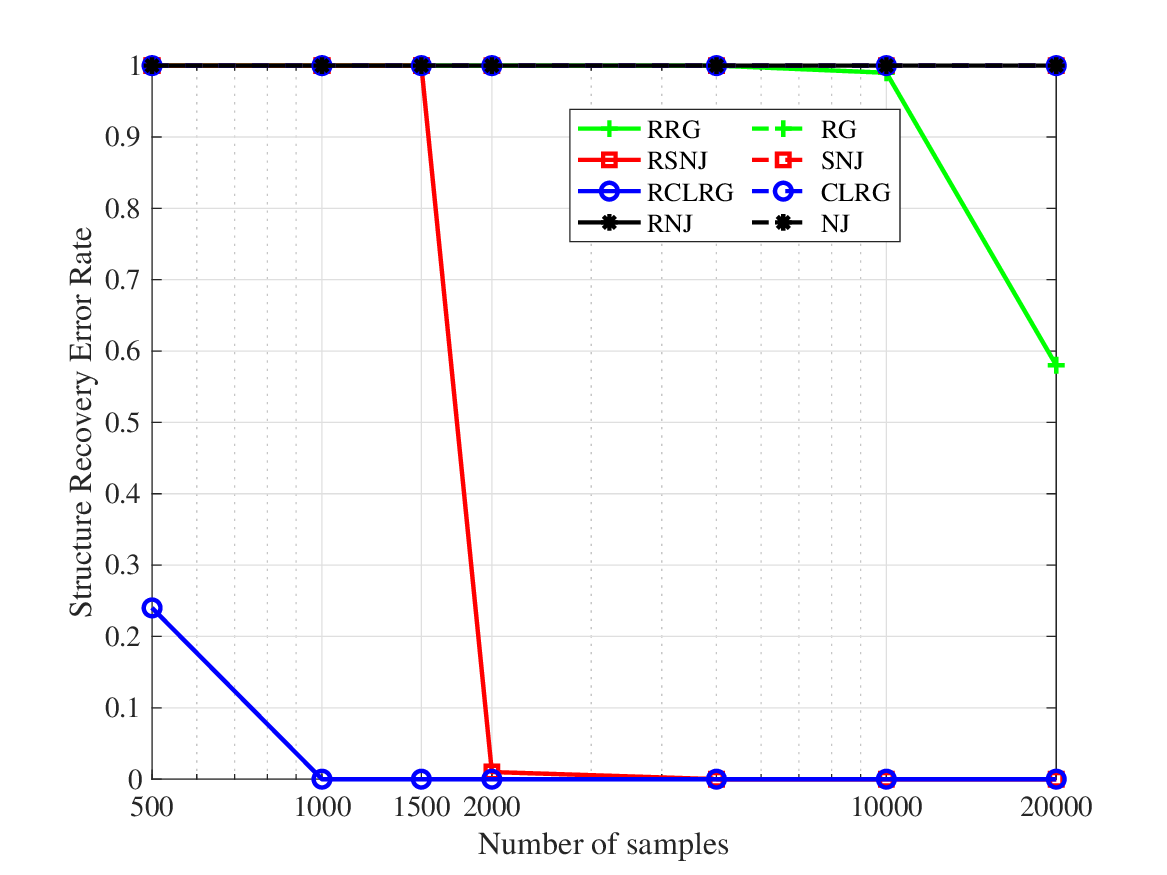}
%\caption{fig1}
\end{minipage}%
}%
\centering
\caption{Performances of robustified and original learning algorithms with \ac{hmm} outliers}
\end{figure}

\begin{figure}[H]
\centering
\subfigure[Robinson-Foulds distances]{
\begin{minipage}[t]{0.48\linewidth}
\centering
\includegraphics[width=2.6in]{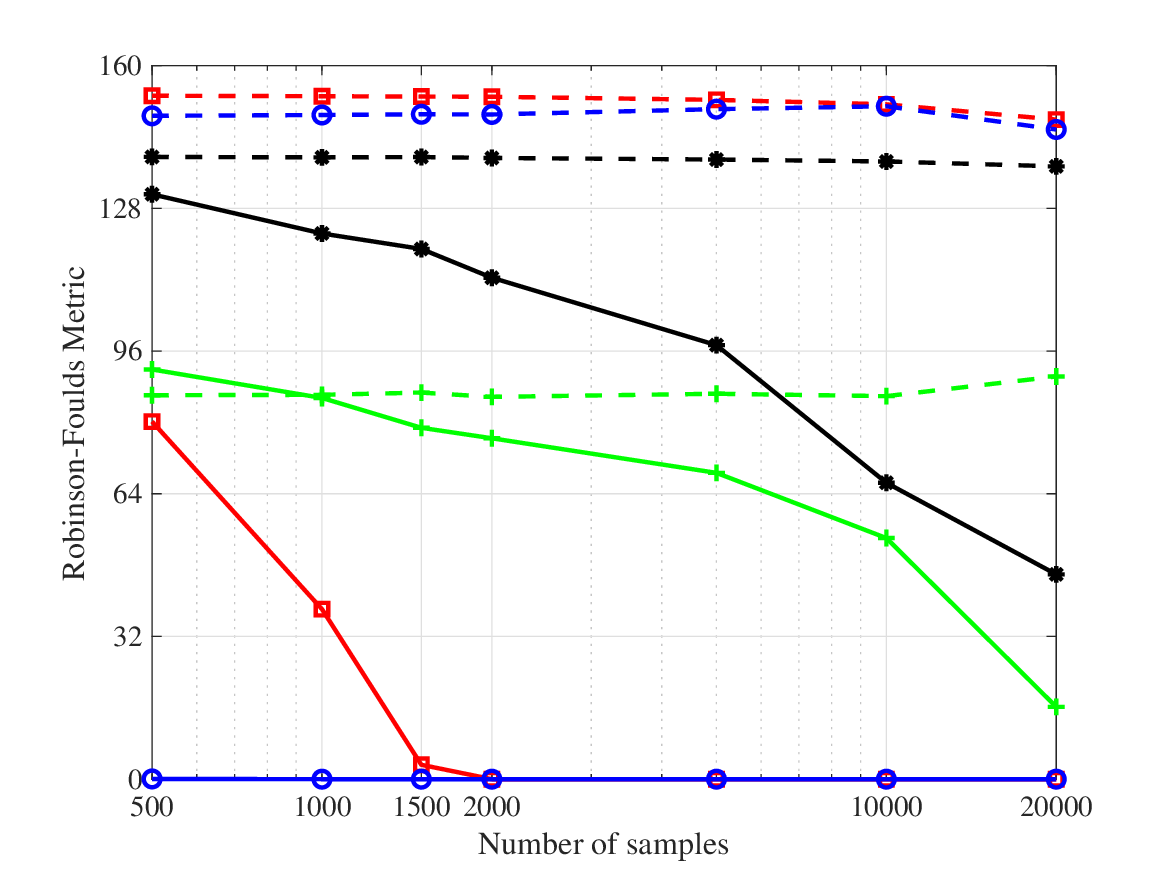}
%\caption{fig2}
\end{minipage}%
}%
\subfigure[Structure recovery error rate]{
\begin{minipage}[t]{0.48\linewidth}
\centering
\includegraphics[width=2.6in]{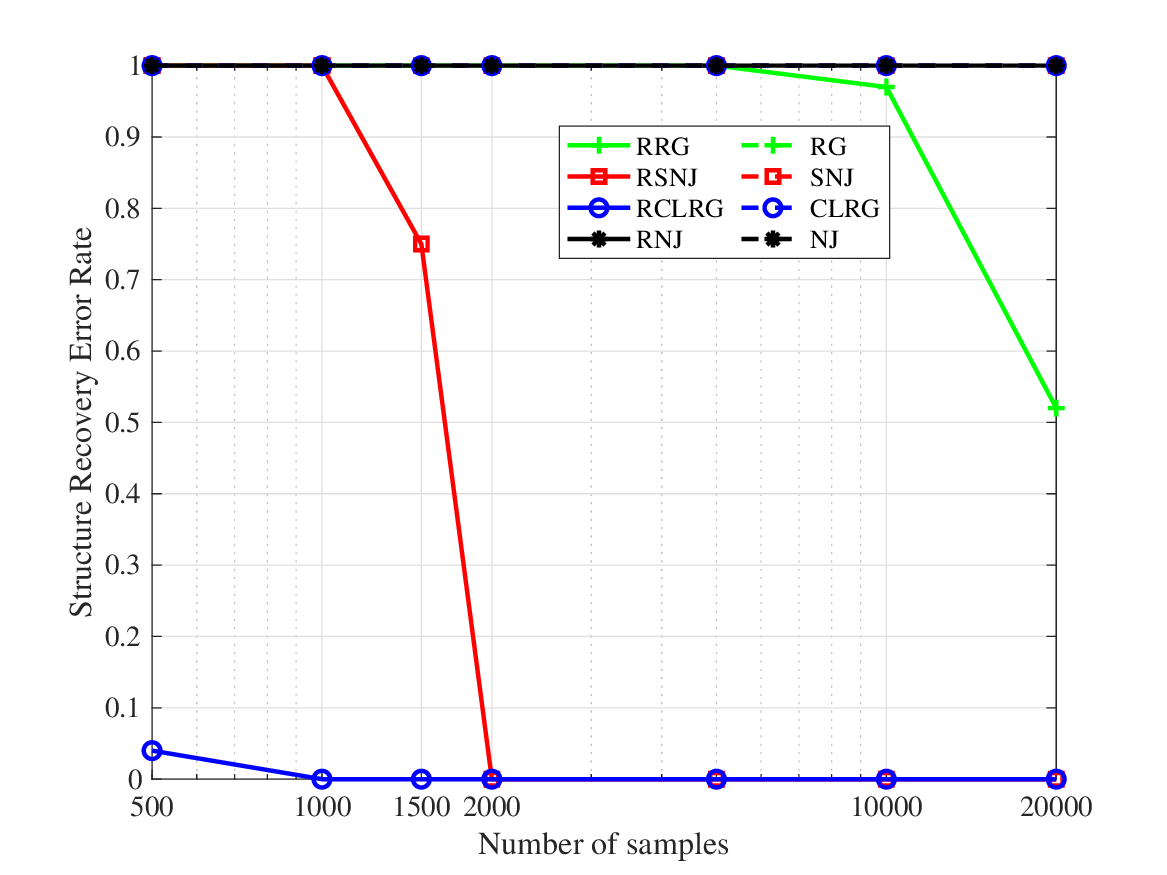}
%\caption{fig1}
\end{minipage}%
}%
\centering
\caption{Performances of robustified and original learning algorithms with double binary outliers}
\end{figure}

These figures show that for the \ac{hmm}, \ac{rclrg} performs best among all these algorithms. The reason is that the Chow-Liu initialization greatly reduces the effective depth of the original tree, which mitigates the error propagation.   These simulation results also corroborate the effectiveness of the truncated inner product in combating any form of corruptions. We observe that the errors of robustified algorithms are significantly less that those of original algorithms.  %\st{However, as shown in Fig.~\ref{fig:hmm_const}, when the number of clean samples is not large, the performance of \ac{rrg} is worse than that of \ac{rg} since the estimation is too large to make the correct inference.}  \red{Can we remove this sentence since it's only a small part of Fig.~15(b) that has this anomaly? Furthermore, your explanation is not clear.}

Table~\ref{table:samplecompare} shows that for the \ac{hmm}, \ac{rclrg} and \ac{rnj} both have optimal dependence on the diameter of the tree. In fact, by changing the parameters $\rho_{\min}$ and $\rho_{\max}$, we find that \ac{rnj} can sometimes perform better than \ac{rclrg} when $\rho_{\min}$ and $\rho_{\max}$ are both very small. In the experiments shown above, the parameters favor \ac{rclrg}.

Finally, it is also instructive to observe the effect of the different corruption patterns. By comparing the simulation results of \ac{hmm} (resp.\ Gaussian and double binary) corruptions and \ac{hmm} (resp.\ Gaussian and double binary) outliers, we can see that the algorithms perform worse in the presence of \ac{hmm} (resp.\ Gaussian and double binary) corruptions. Since the truncated inner product truncates the samples with large absolute values, if corruptions appear in the {\em same} positions for all the samples, i.e., they appear as outliers, it is easier for the truncated inner product to identify these outliers and truncate them, resulting in higher quality estimates.

\subsubsection{Double binary tree}
The diameter of the double binary tree (Fig.~\ref{fig:comp}(a)) is $\mathrm{Diam}(\mathbb{T})=11$. The  matrices $(\mathbf{A},\mathbf{\Sigma}_{\mathrm{r}},\mathbf{\Sigma}_{\mathrm{n}})$ are chosen so that the condition in Proposition \ref{prop:homo} are satisfied with $\alpha=1$, and we set $\mathbf{A}$ commutable with $\mathbf{\Sigma}_{\mathrm{r}}$. The information distance between neighboring nodes is $1$, which implies that $\rho_{\min}=1$ and $\rho_{\max}=\mathrm{Diam}(\mathbb{T})=11$.
\begin{figure}[H]
\centering
\subfigure[Robinson-Foulds distances]{
\begin{minipage}[t]{0.48\linewidth}
\centering
\includegraphics[width=2.6in]{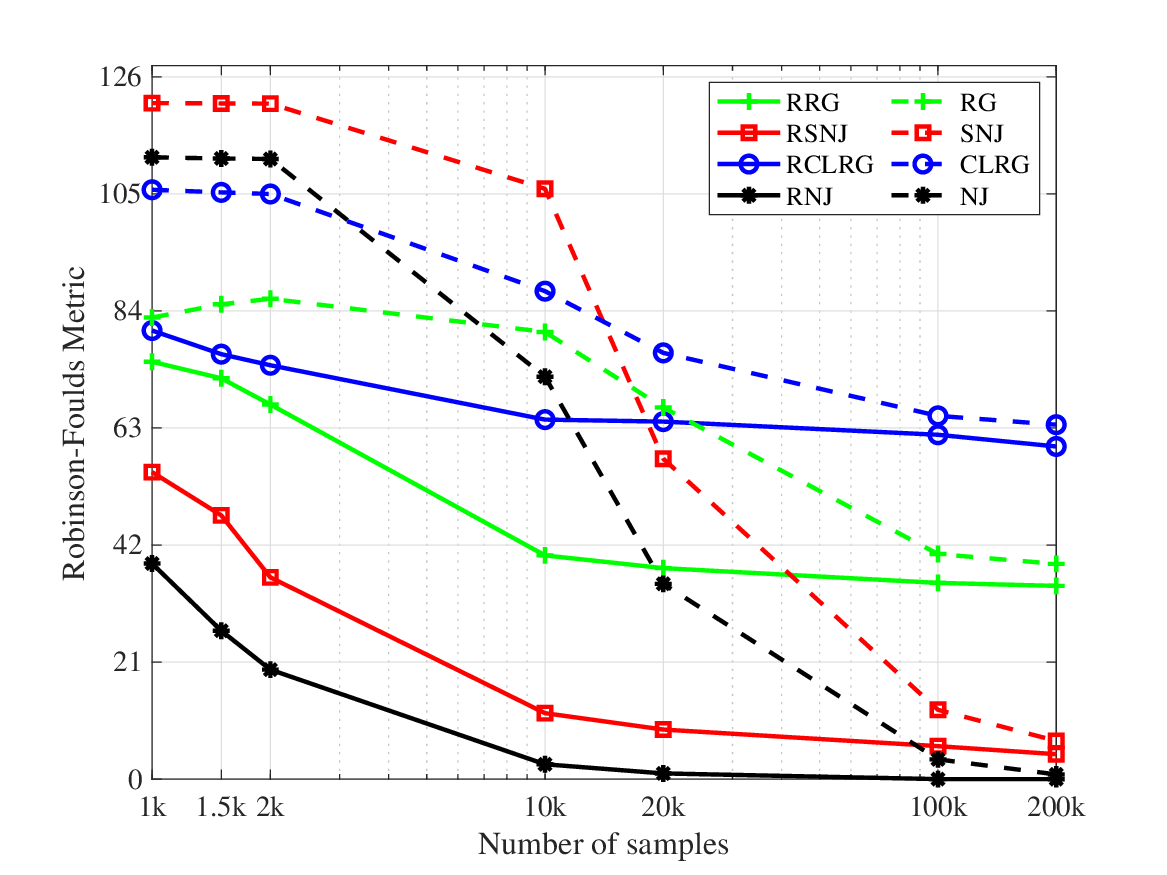}
%\caption{fig2}
\end{minipage}%
}%
\subfigure[Structure recovery error rate]{
\begin{minipage}[t]{0.48\linewidth}
\centering
\includegraphics[width=2.6in]{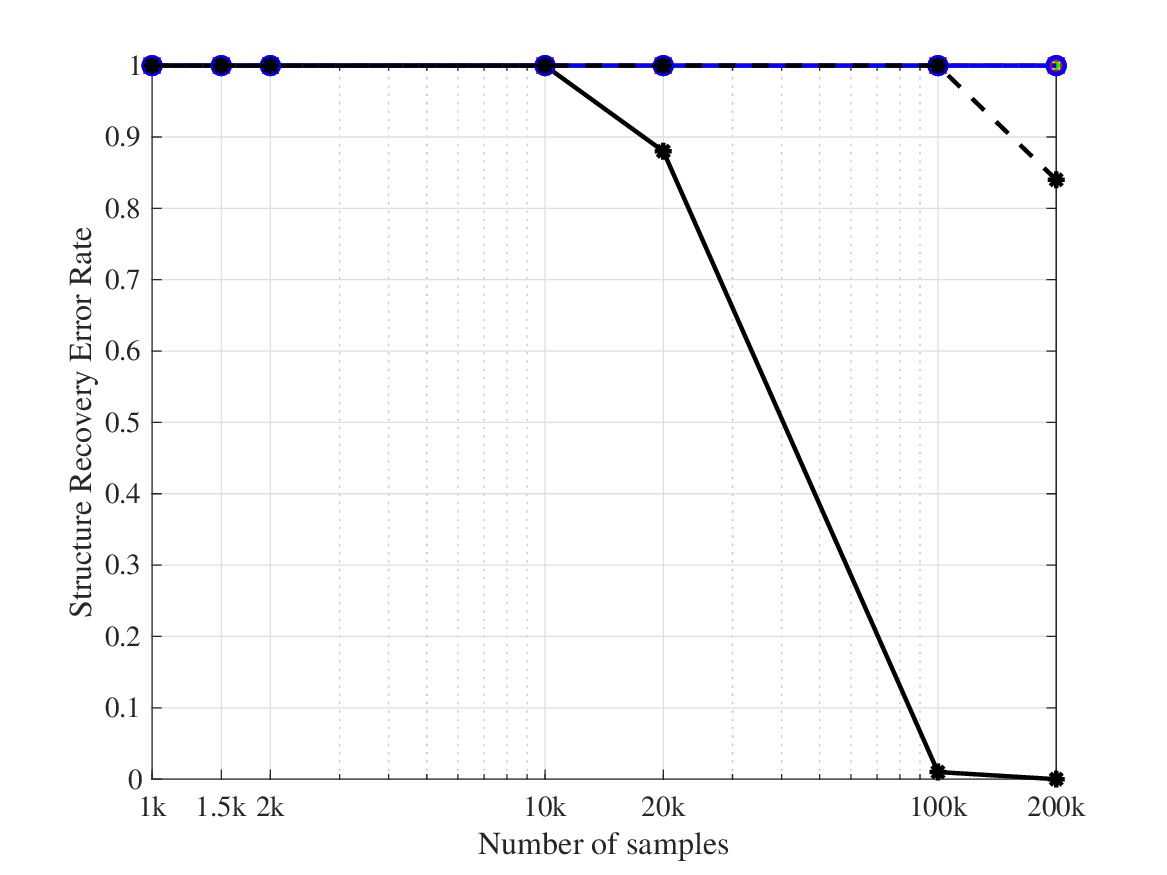}
%\caption{fig1}
\end{minipage}%
}%
\centering
\caption{Performances of robustified and original learning algorithms with constant magnitude corruptions}
\end{figure}

\begin{figure}[H]
\centering
\subfigure[Robinson-Foulds distances]{
\begin{minipage}[t]{0.48\linewidth}
\centering
\includegraphics[width=2.6in]{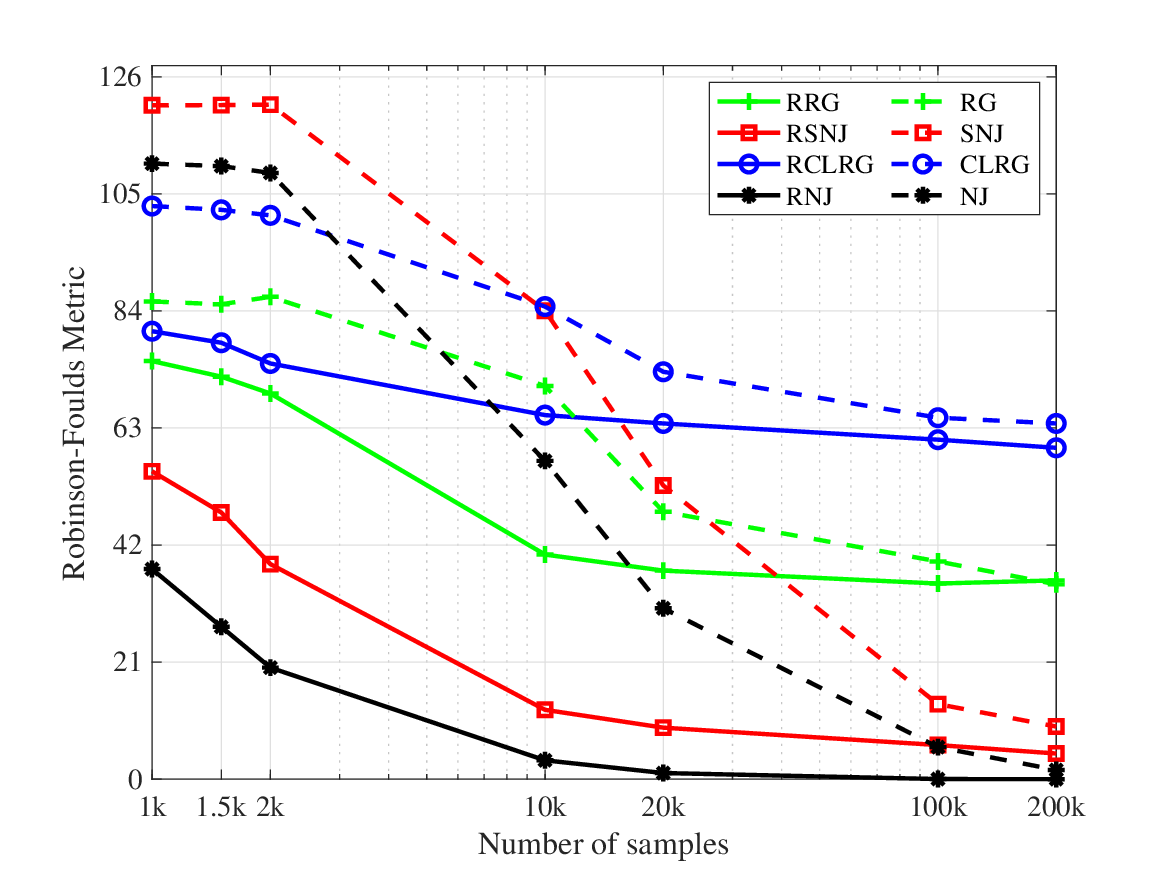}
%\caption{fig2}
\end{minipage}%
}%
\subfigure[Structure recovery error rate]{
\begin{minipage}[t]{0.48\linewidth}
\centering
\includegraphics[width=2.6in]{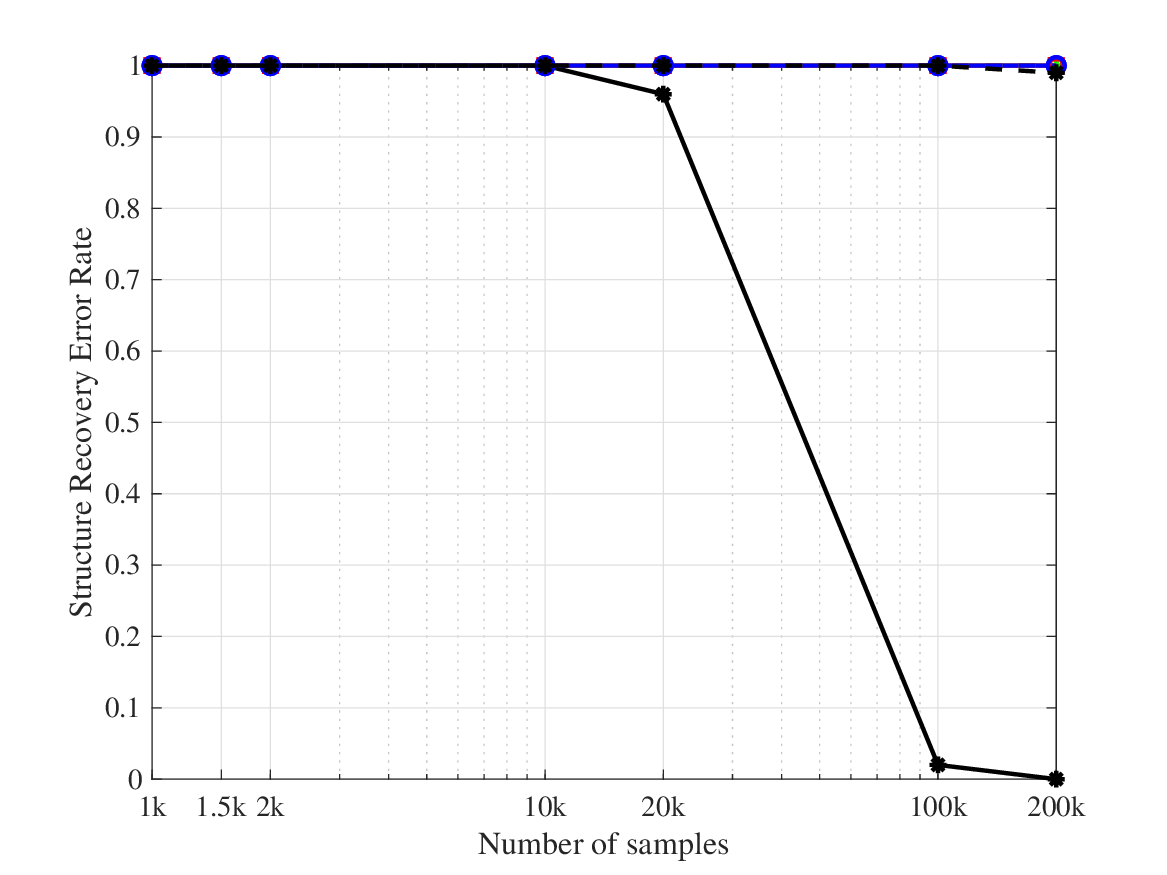}
%\caption{fig1}
\end{minipage}%
}%
\centering
\caption{Performances of robustified and original learning algorithms with uniform corruptions}
\end{figure}

\begin{figure}[H]
\centering
\subfigure[Robinson-Foulds distances]{
\begin{minipage}[t]{0.48\linewidth}
\centering
\includegraphics[width=2.6in]{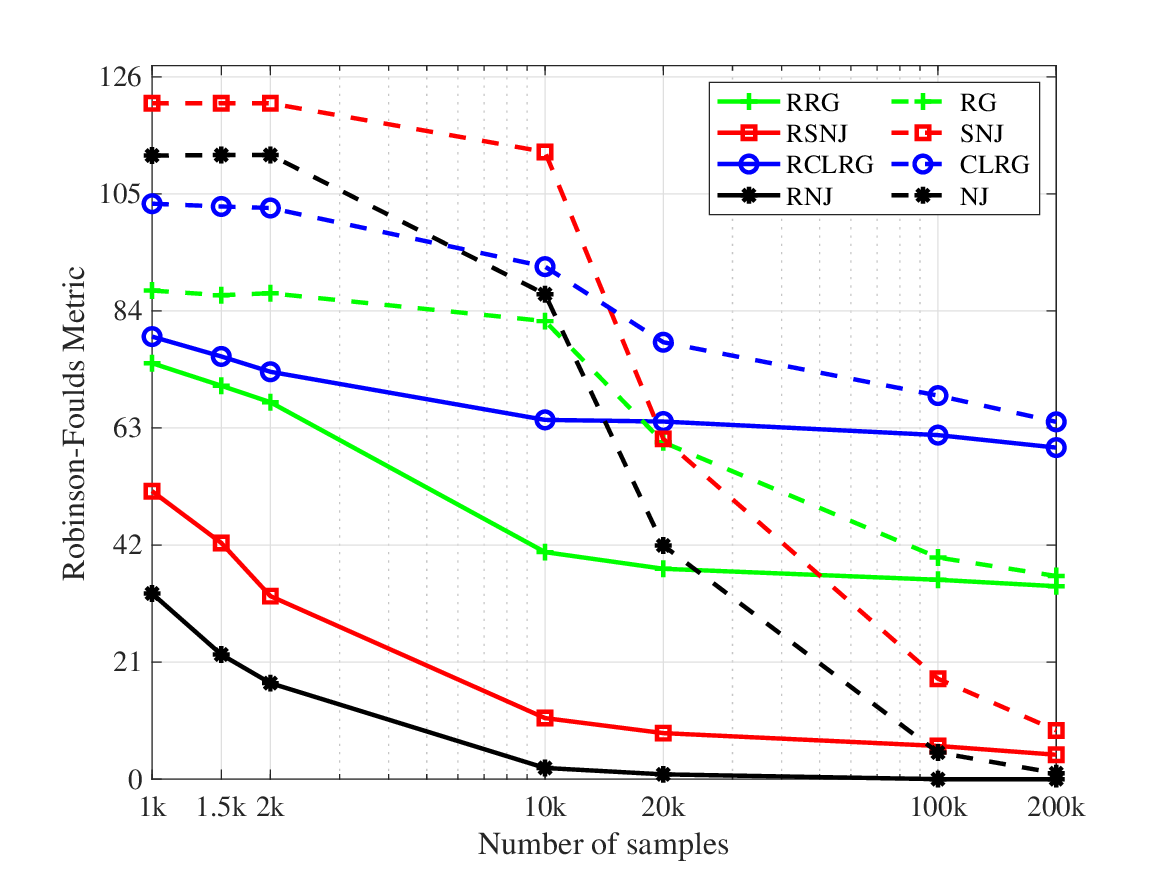}
%\caption{fig2}
\end{minipage}%
}%
\subfigure[Structure recovery error rate]{
\begin{minipage}[t]{0.48\linewidth}
\centering
\includegraphics[width=2.6in]{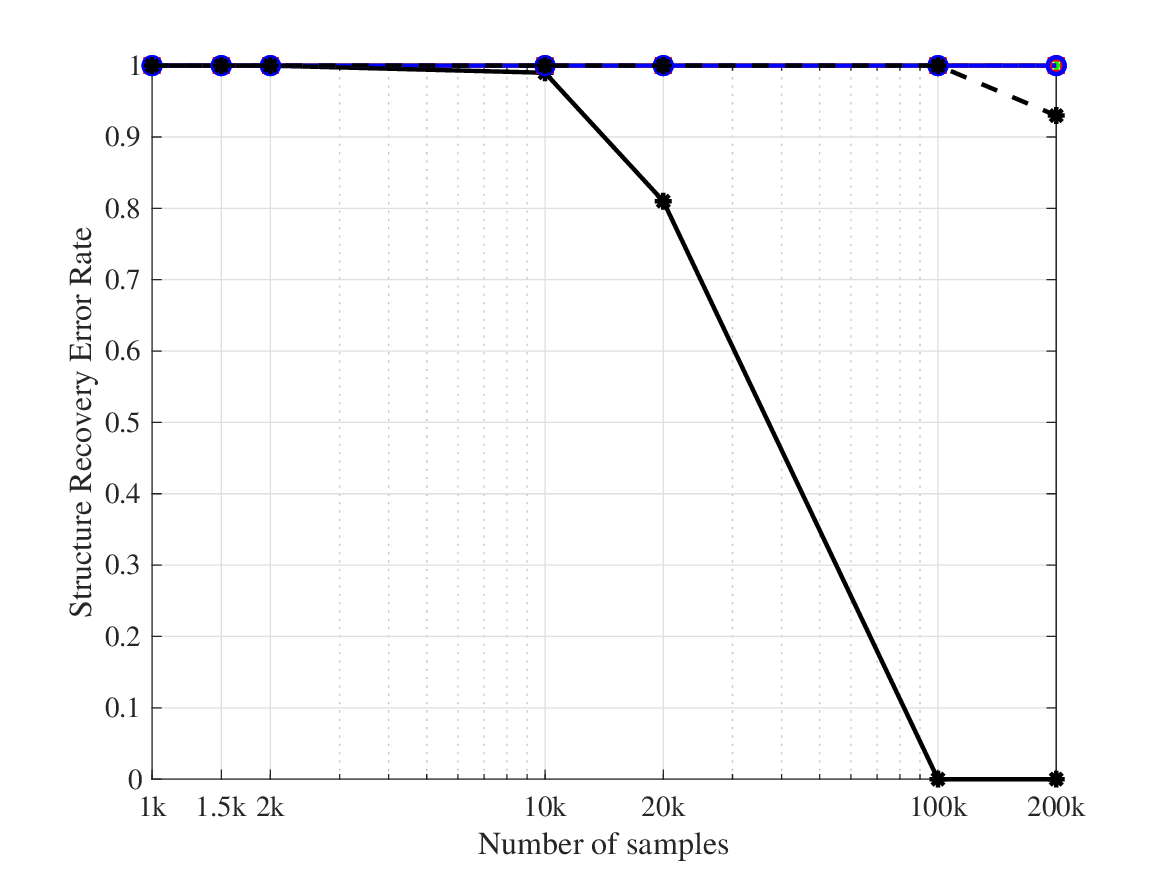}
%\caption{fig1}
\end{minipage}%
}%
\centering
\caption{Performances of robustified and original learning algorithms with Gaussian corruptions}
\end{figure}

\begin{figure}[H]
\centering
\subfigure[Robinson-Foulds distances]{
\begin{minipage}[t]{0.48\linewidth}
\centering
\includegraphics[width=2.6in]{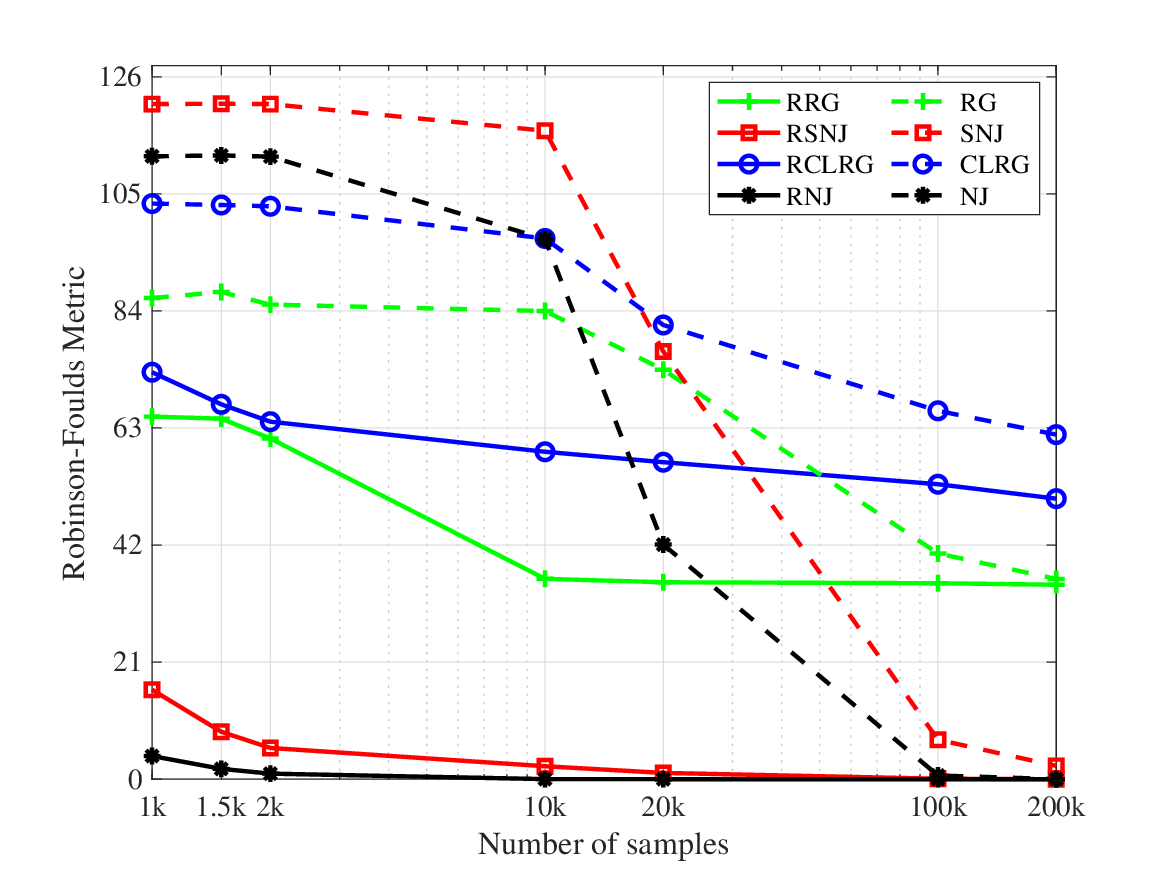}
%\caption{fig2}
\end{minipage}%
}%
\subfigure[Structure recovery error rate]{
\begin{minipage}[t]{0.48\linewidth}
\centering
\includegraphics[width=2.6in]{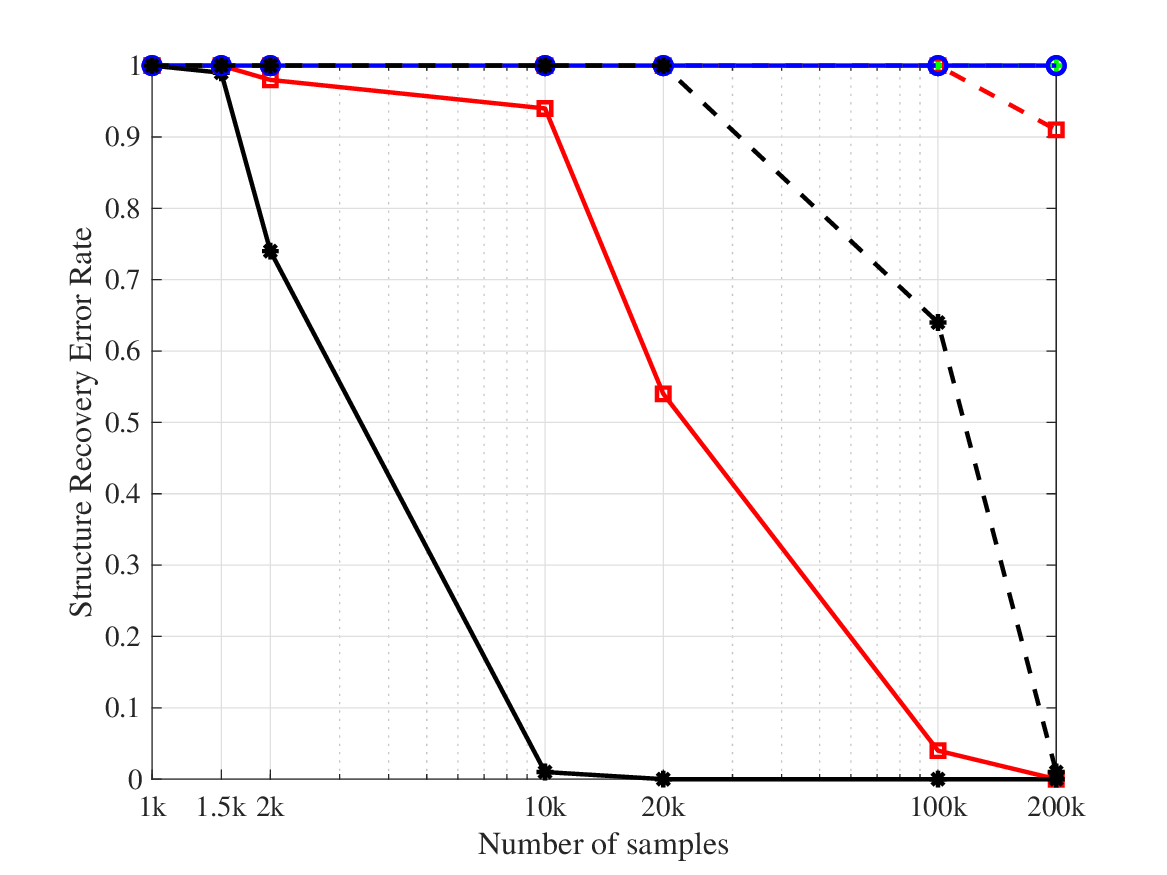}
%\caption{fig1}
\end{minipage}%
}%
\centering
\caption{Performances of robustified and original learning algorithms with \ac{hmm} corruptions}
\end{figure}

\begin{figure}[H]
\centering
\subfigure[Robinson-Foulds distances]{
\begin{minipage}[t]{0.48\linewidth}
\centering
\includegraphics[width=2.6in]{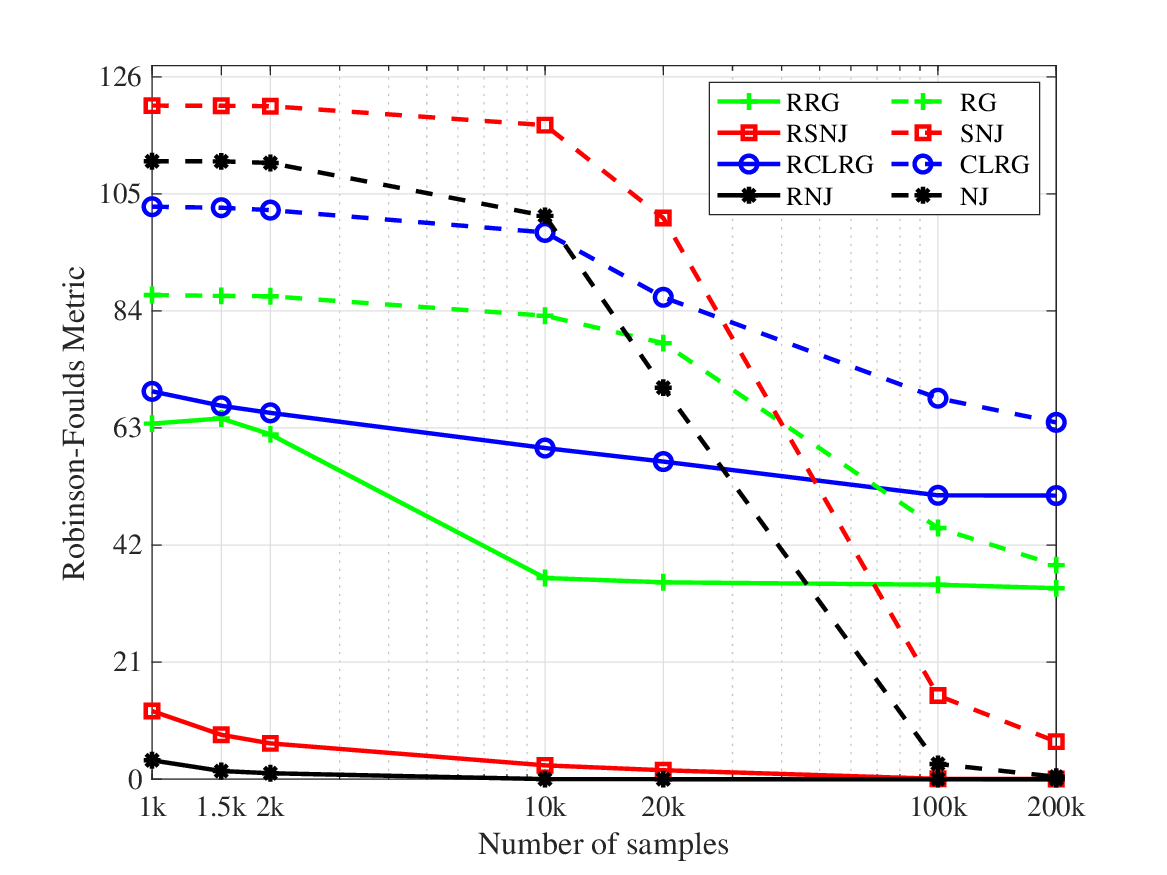}
%\caption{fig2}
\end{minipage}%
}%
\subfigure[Structure recovery error rate]{
\begin{minipage}[t]{0.48\linewidth}
\centering
\includegraphics[width=2.6in]{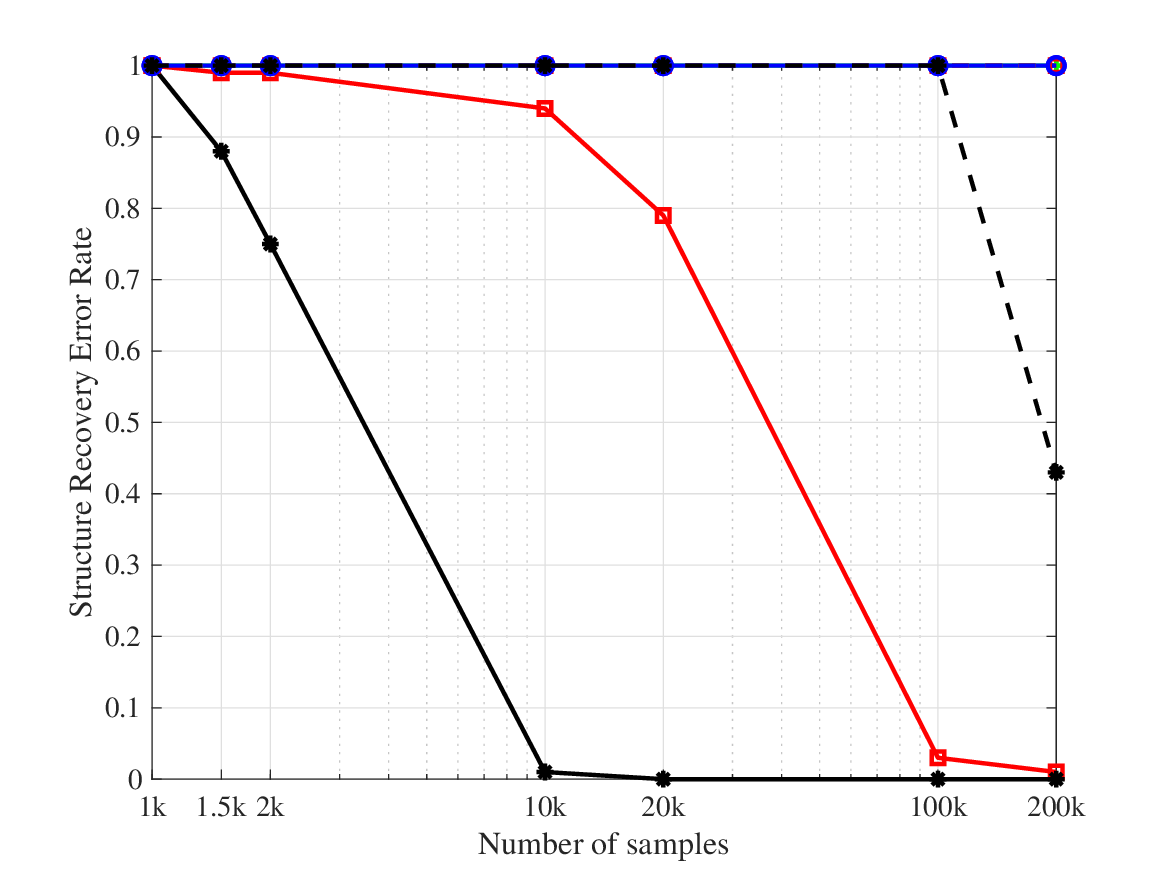}
%\caption{fig1}
\end{minipage}%
}%
\centering
\caption{Performances of robustified and original learning algorithms with double binary corruptions}
\end{figure}

\begin{figure}[H]
\centering
\subfigure[Robinson-Foulds distances]{
\begin{minipage}[t]{0.48\linewidth}
\centering
\includegraphics[width=2.6in]{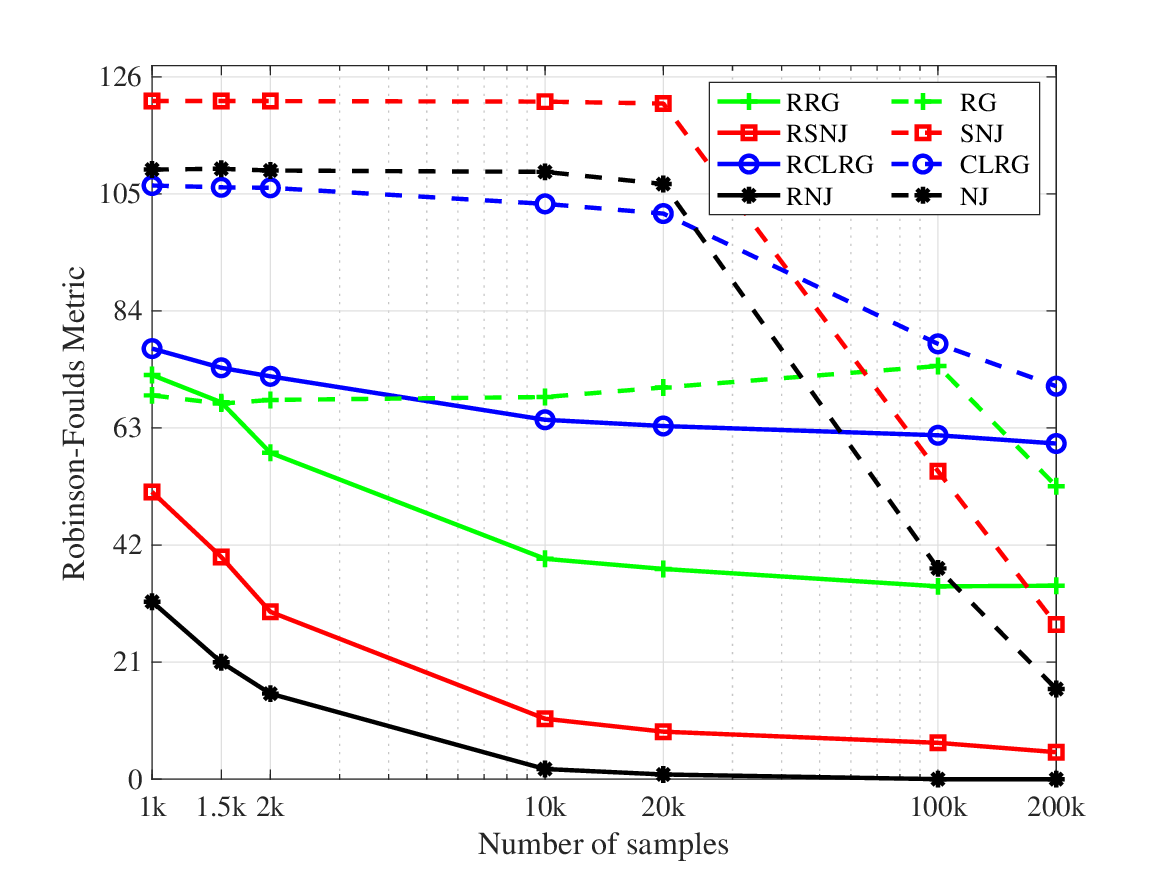}
%\caption{fig2}
\end{minipage}%
}%
\subfigure[Structure recovery error rate]{
\begin{minipage}[t]{0.48\linewidth}
\centering
\includegraphics[width=2.6in]{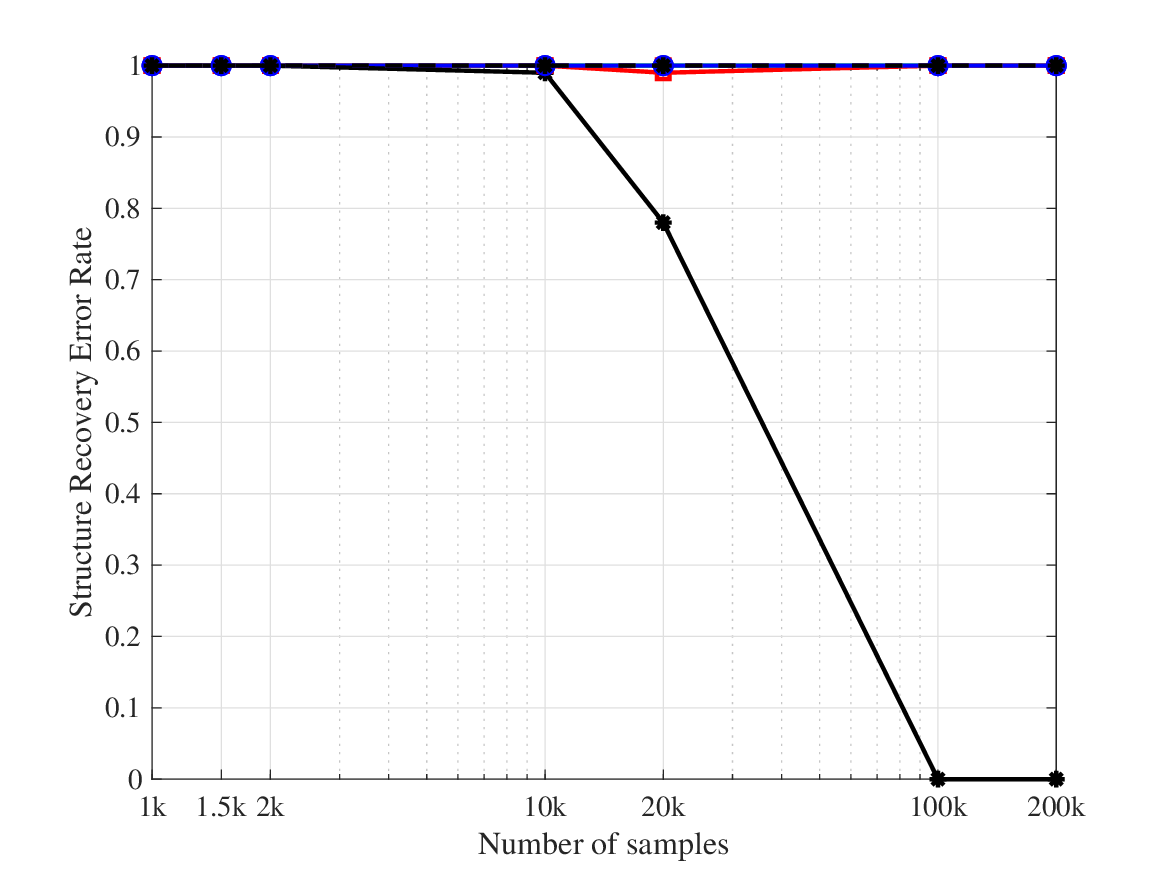}
%\caption{fig1}
\end{minipage}%
}%
\centering
\caption{Performances of robustified and original learning algorithms with Gaussian outliers}
\end{figure}

\begin{figure}[H]
\centering
\subfigure[Robinson-Foulds distances]{
\begin{minipage}[t]{0.48\linewidth}
\centering
\includegraphics[width=2.6in]{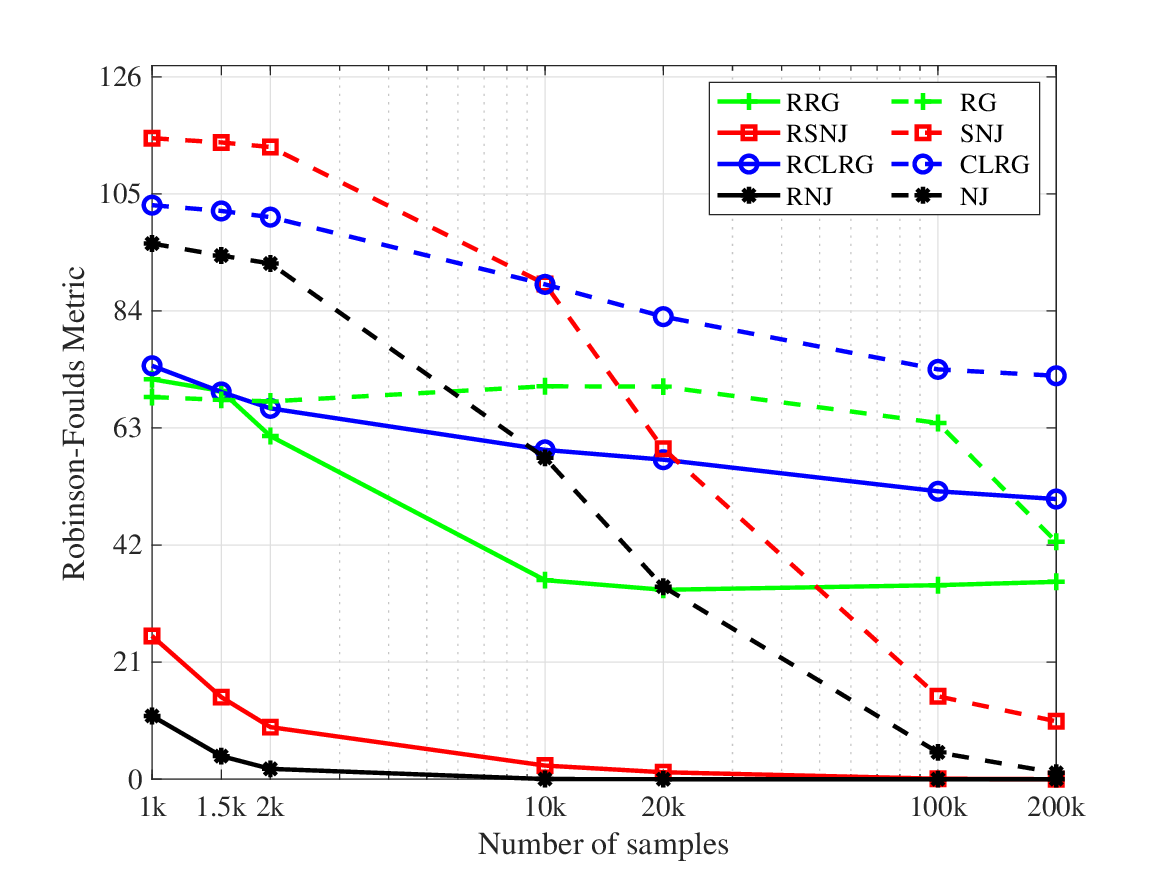}
%\caption{fig2}
\end{minipage}%
}%
\subfigure[Structure recovery error rate]{
\begin{minipage}[t]{0.48\linewidth}
\centering
\includegraphics[width=2.6in]{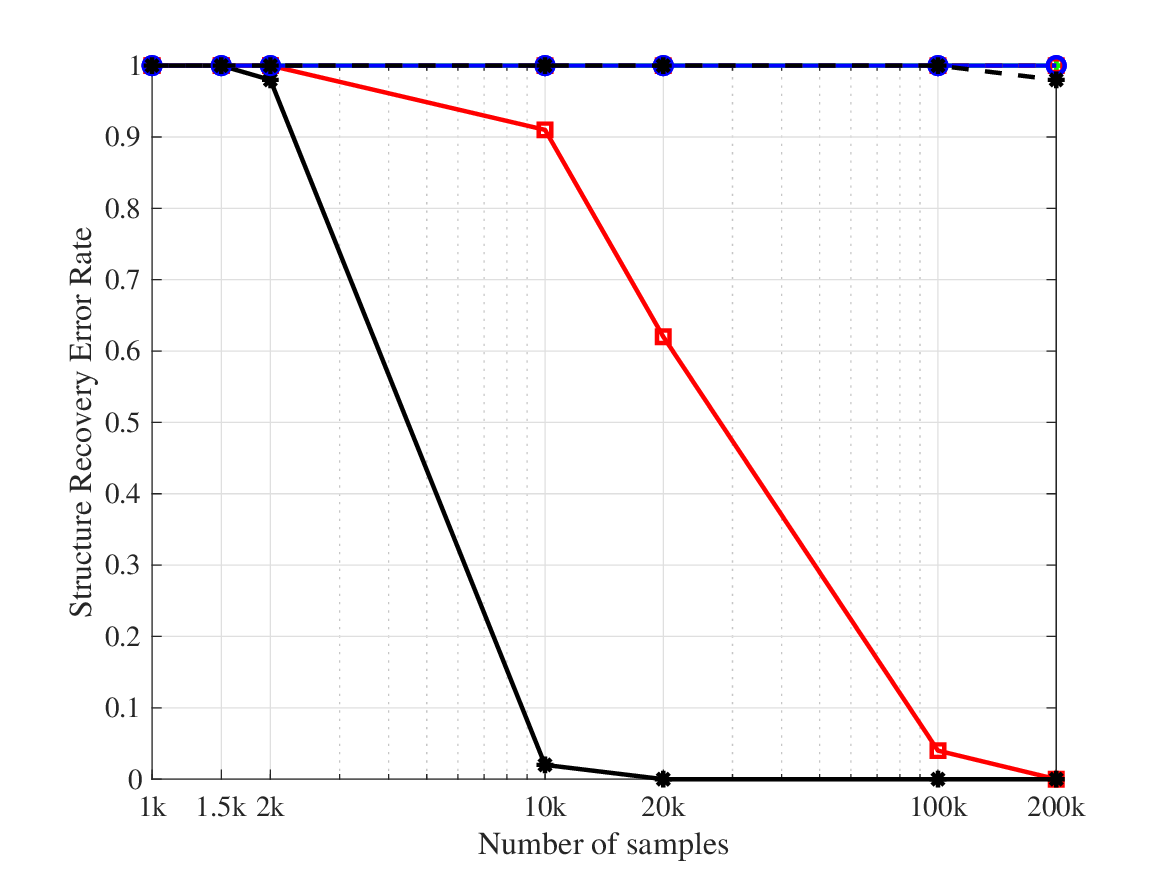}
%\caption{fig1}
\end{minipage}%
}%
\centering
\caption{Performances of robustified and original learning algorithms with \ac{hmm} outliers}
\end{figure}

\begin{figure}[H]
\centering
\subfigure[Robinson-Foulds distances]{
\begin{minipage}[t]{0.48\linewidth}
\centering
\includegraphics[width=2.6in]{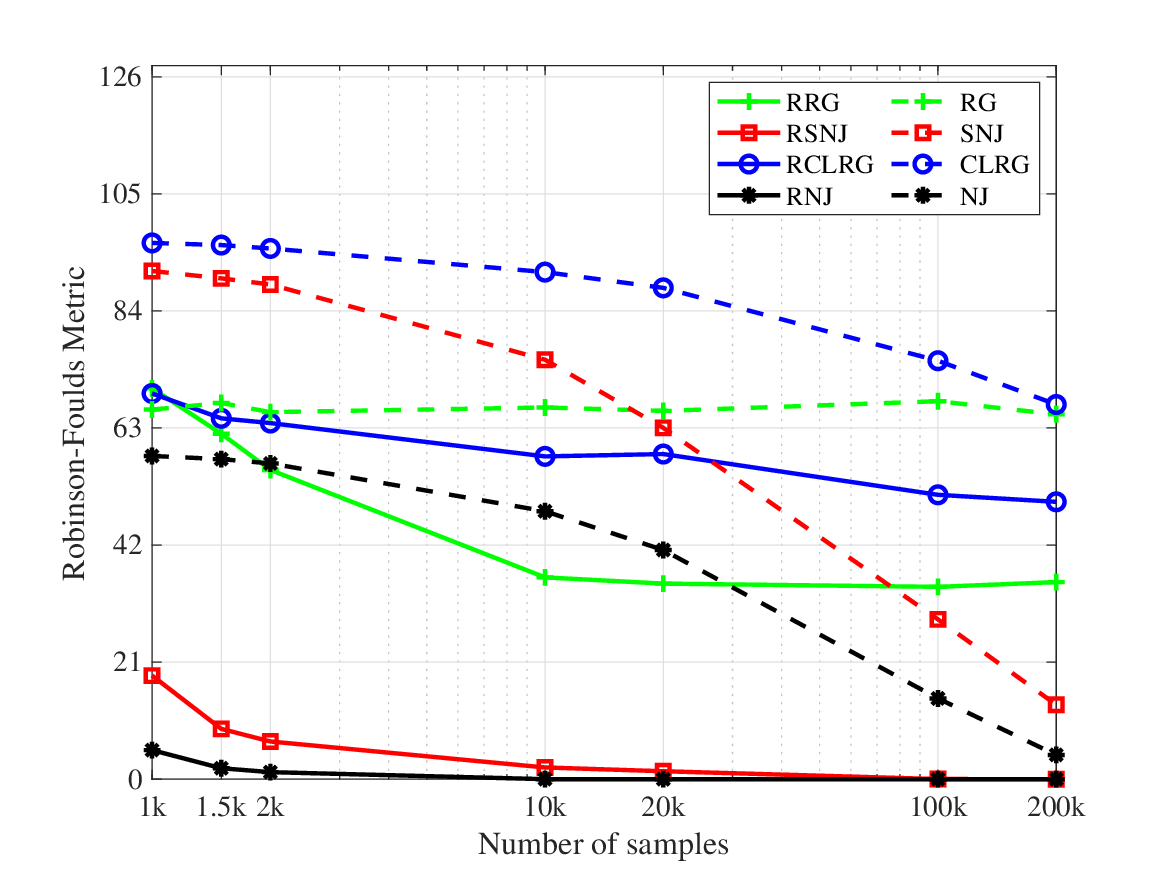}
%\caption{fig2}
\end{minipage}%
}%
\subfigure[Structure recovery error rate]{
\begin{minipage}[t]{0.48\linewidth}
\centering
\includegraphics[width=2.6in]{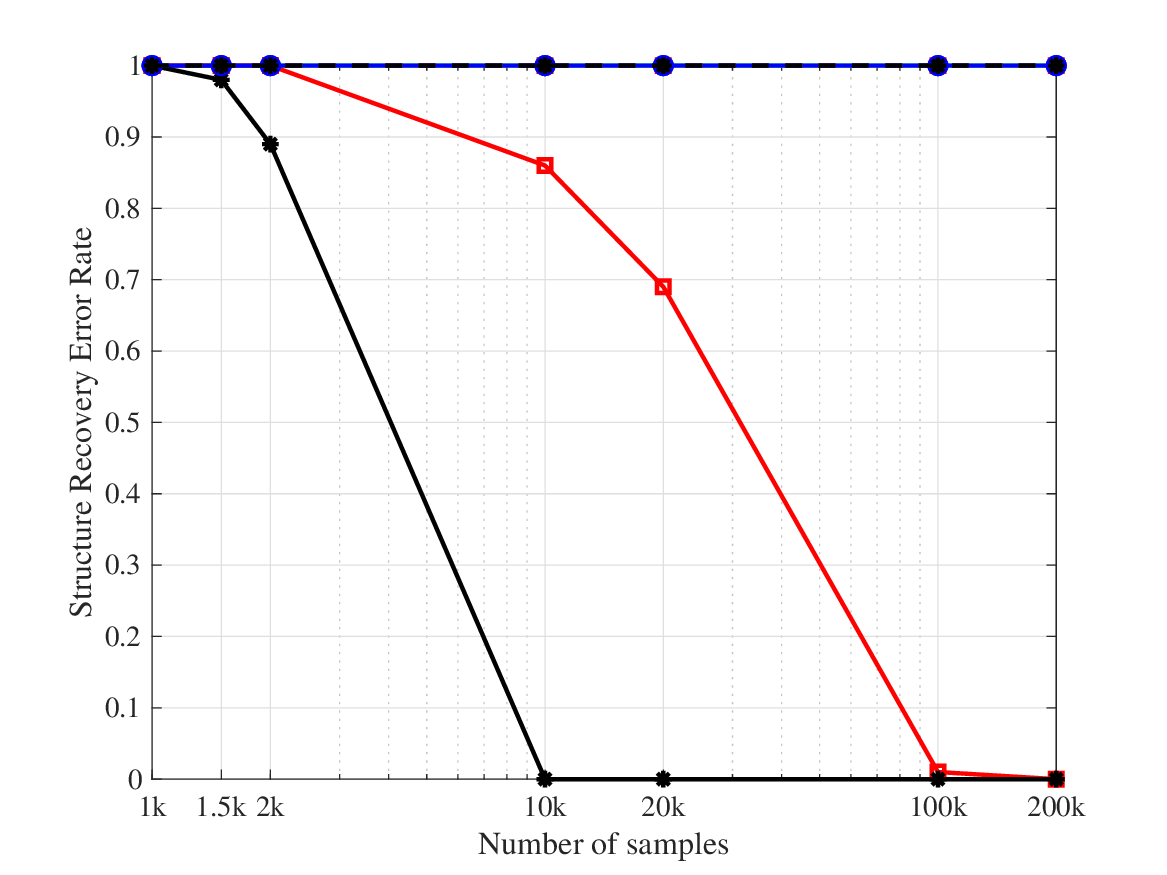}
%\caption{fig1}
\end{minipage}%
}%
\centering
\caption{Performances of robustified and original learning algorithms with double binary outliers}
\end{figure}

These figures reinforce that the robustification procedure is  highly effective in combating the corruptions. Furtheremore, we observe that \ac{rnj} performs  the  best among all these algorithms for the double binary tree. However, the simulation results in Jaffe et al.~\cite{jaffe2021spectral} shows that \ac{snj} performs better than \ac{nj}. This does not contradict our observations here. The reason lies on the choice of the parameters of the model $\rho_{\min}$ and $\rho_{\max}$. In the simulations of \cite{jaffe2021spectral}, the parameter $\delta$ (defined in therein) is set to $0.9$, but in our simulation, the equivalent parameter $e^{-2\rho_{\max}/\mathrm{Diam}(\mathbb{T})}$ is $0.1$. The exponential dependence on $\rho_{\max}$ of \ac{rsnj} listed in Table \ref{table:samplecompare} explains the difference between simulation results in \cite{jaffe2021spectral} and our simulation results.

\section{Proofs of results in Section~\ref{subsec:converse}} \label{app:converse}
To derive the impossibility results, we will apply  Fano's inequality on two special families of graphical models, each contained in $\mathcal{T}(|\mathcal{V}_{\mathrm{obs}}|,\rho_{\max},l_{\max})$. Each graphical model in the families is parameterized by a quartet $(\mathbf{A},\mathbf{\Sigma}_{\mathrm{r}},\mathbf{\Sigma}_{\mathrm{n}},\alpha)$. This quartet defines the Gaussian graphical model as follows. We choose a node in the tree as the root node $\mathbf{x}_{\mathrm{r}}$, and define the parent node and set of children nodes  (in the rooted tree) of any node $x_{i}$ as  $\mathrm{pa}(i)$ and $\mathcal{C}(x_{i})$ respectively. The depth of a node $x_{i}$  (with respect to the root node $x_{\mathrm{r}}$) is 
$\mathrm{d}_{\mathbb{T}}(x_{i},x_{\mathrm{r}})$. We specify the model in which   
\begin{align}\label{eqn:chanl}
	\mathbf{x}_{i}=\mathbf{A}\mathbf{x}_{\text{pa}(i)}+\mathbf{n}_{i} \quad\text{for all}\quad x_{i}\in \mathcal{V}
\end{align}
where $\mathbf{A} \in \mathbb{R}^{l_{\max}\times l_{\max}}$ is  non-singular, $\mathbf{n}_{i}\sim \mathcal{N}(\mathbf{0},\alpha^{\mathrm{d}_{\mathbb{T}}(x_{i},x_{\mathrm{r}})-1}\mathbf{\Sigma}_{\mathrm{n}})$ and $\mathbf{n}_{i}$'s are mutually independent. Since the root node has no parent, it is natural to set  $\mathbf{x}_{\text{pa}(\mathrm{r})}=\mathbf{0}$ and  $\mathbf{n}_{\mathrm{r}}\sim \mathcal{N}(\mathbf{0},\mathbf{\Sigma}_{\mathrm{r}})$. 
%assign the 
%following conditions to the root node
%\begin{align}\label{eqn:rt}
%\mathbf{x}_{\text{pa}(\mathrm{r})}=\mathbf{0} \quad \text{and}\quad %\mathbf{n}_{\mathrm{r}}\sim %\mathcal{N}(\mathbf{0},\mathbf{\Sigma}_{\mathrm{r}}).
%\end{align}
It is easy to verify that the model specified by \eqref{eqn:chanl} and this initial condition is an undirected \ac{ggm}. Then the covariance matrix of the random vector $\mathbf{x}_{i}$ is $\alpha^{\mathrm{d}_{\mathbb{T}}(x_{i},x_{\mathrm{r}})}\mathbf{\Sigma}_{\mathrm{r}}$.
\begin{proposition}\label{prop:homo}
	If $\mathbf{n}_{i}$'s for the variables at depth $l$ are distributed as  $\mathcal{N}(0,\alpha^{l-1}\mathbf{\Sigma}_{\mathrm{n}})$, and 
	\begin{align}\label{eqn:homo}
		\mathbf{A}\mathbf{\Sigma}_{\mathrm{r}}\mathbf{A}^{\top}+\mathbf{\Sigma}_{\mathrm{n}}=\alpha\mathbf{\Sigma}_{\mathrm{r}}
	\end{align}
	where $\alpha>0$ is a constant, then the covariance matrix of the variable at depth $l$ is $\alpha^{l}\mathbf{\Sigma}_{\mathrm{r}}$.
\end{proposition}
We term   \eqref{eqn:homo} as the \emph{$(\mathbf{A},\mathbf{\Sigma}_{\mathrm{r}},\mathbf{\Sigma}_{\mathrm{n}})$-homogenous} condition, which guarantees that covariance matrices of the random vectors in the tree are same up to a scale 
factor. 
\begin{proof}[Proof of Proposition ~\ref{prop:homo}]
	The statement in Proposition \ref{prop:homo} is equivalent to
	\begin{align}\label{eqn:stacov}
		\mathbf{A}^{l}\mathbf{\Sigma}_{\mathrm{r}}(\mathbf{A}^{l})^{\top}+\sum_{i=1}^{l}\alpha^{i-1}\mathbf{A}^{l-i}\mathbf{\Sigma}_{\mathrm{n}}(\mathbf{A}^{l-i})^{\top}=\alpha^{l}\mathbf{\Sigma}_{\mathrm{r}}.
	\end{align}
	We prove \eqref{eqn:stacov} by induction.

	When $l=1$, the homogenous condition guarantees that $\mathbf{A}\mathbf{\Sigma}_{\mathrm{r}}\mathbf{A}^{\top}+\mathbf{\Sigma}_{\mathrm{n}}=\alpha\mathbf{\Sigma}_{\mathrm{r}}$.

	If \eqref{eqn:stacov} holds for $l=1,\ldots,n$, then for $l=n+1$
	\begin{align}
		&\mathbf{A}^{n+1}\mathbf{\Sigma}_{\mathrm{r}}(\mathbf{A}^{n+1})^{\top}+\sum_{i=1}^{n+1}\alpha^{i-1}\mathbf{A}^{n+1-i}\mathbf{\Sigma}_{\mathrm{n}}(\mathbf{A}^{n+1-i})^{\top}\nonumber \\
		&=\mathbf{A}(\mathbf{A}^{n}\mathbf{\Sigma}_{\mathrm{r}}(\mathbf{A}^{n})^{\top}+\sum_{i=1}^{n+1}\alpha^{i-1}\mathbf{A}^{n-i}\mathbf{\Sigma}_{\mathrm{n}}(\mathbf{A}^{n-i})^{\top})\mathbf{A}^{\top} \\
		&=\mathbf{A}(\alpha^{n}\mathbf{\Sigma}_{\mathrm{r}}+\alpha^{n}\mathbf{A}^{-1}\mathbf{\Sigma}_{\mathrm{n}}\mathbf{A}^{-\top})\mathbf{A}^{\top}\\
		&=\alpha^{n}(\mathbf{A}\mathbf{\Sigma}_{\mathrm{r}}\mathbf{A}^{\top}+\mathbf{\Sigma}_{\mathrm{n}}) \\
		&=\alpha^{n+1}\mathbf{\Sigma}_{\mathrm{r}}
	\end{align}
	as desired.
\end{proof}

\begin{proposition}\label{prop:isggm}
	The undirected graphical model specified by \eqref{eqn:chanl} and the initial condition $\mathbf{x}_{\mathrm{pa}(\mathrm{r})}=\mathbf{0}$, $\mathbf{n}_{\mathrm{r}}\sim\mathcal{N}(\mathbf{0},\mathbf{\Sigma}_{r})$ is \ac{ggm}.
\end{proposition}
\begin{proof}[Proof of Proposition ~\ref{prop:isggm}]
	To prove that the specified model is a \ac{ggm}, we need to prove that the joint distribution of all variables is Gaussian   and that the conditional 
	independence relationship induced by the edges is achieved.

	According to \eqref{eqn:chanl} and the initial condition, it is easy to see that any linear combination of variables is the linear combination of independent 
	Gaussian variables, which is   Gaussian. Thus, the joint distribution of all variables is indeed Gaussian.

	To show that the conditional independence is guaranteed, we show that
	\begin{align}
		A\independent B \mid S \text{ for any }S\text{ separates }A \text{ and }B.
	\end{align}
	where $S$, $A$ and $B$ are all sets of nodes, and $S\text{ separates }A \text{ and }B$ means that any path connected nodes in $A$ and $B$ goes through a node in  $S$.

	Without loss of generality, we consider the case where $S$, $A$ and $B$ consist of a single node for conciseness of the proof. The case where these sets consist of multiple nodes can 
	be easily proved by generalizing the proof we show here.

	\begin{figure}[H]
		\centering\includegraphics[width=0.4\columnwidth,draft=false]{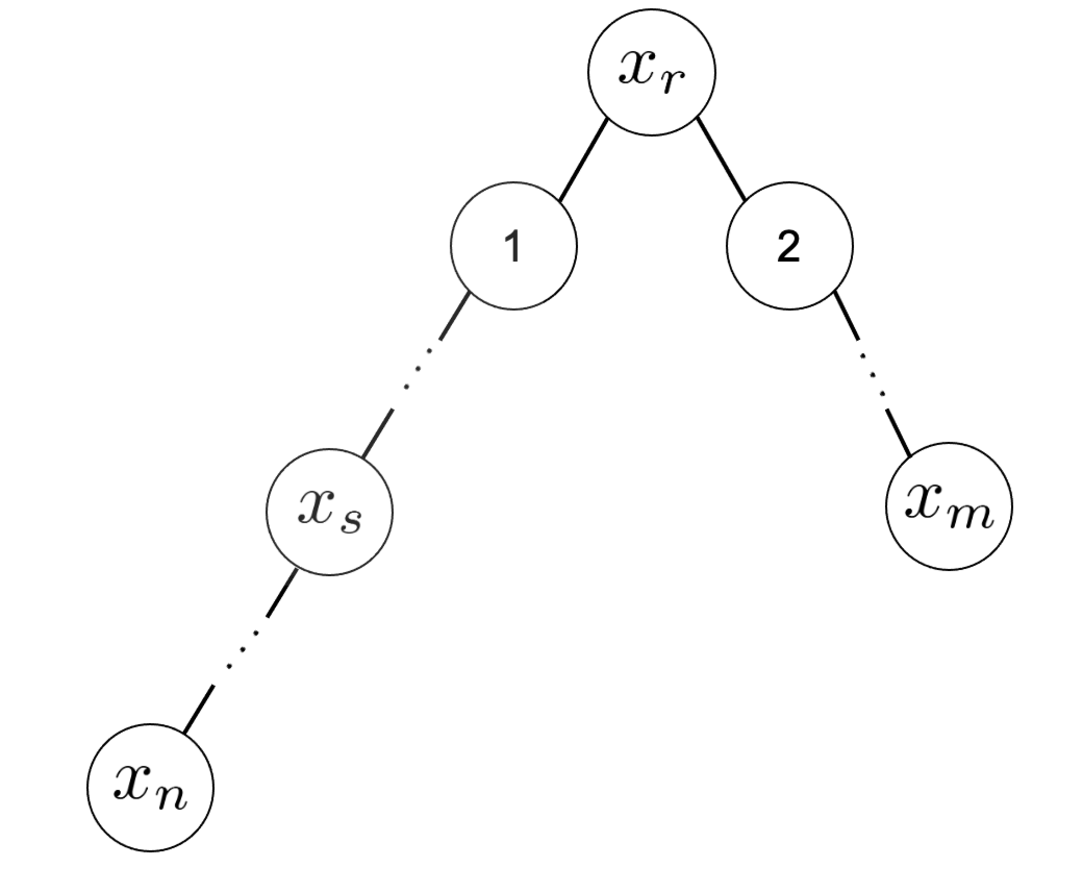}
		\caption{Illustration of the relationship among $x_{n}$, $x_{m}$ and $x_{s}$.}
		\label{fig:undirect_gm}
	\end{figure}
	We first consider the case where ${x}_{n}$ and ${x}_{m}$ belong to different branches, as shown in Fig. \ref{fig:undirect_gm}, and the depths of ${x}_{n}$ and ${x}_{m}$ are $n$ and 
	$m$, respectively. The separator node $x_{s}$ can be anywhere along the path connecting $x_{n}$ and $x_{m}$. Without loss of generality, we assume 
	it sits in the same branch as ${x}_{n}$, and its depth is $s$, where $s<n$. Then we have
	\begin{align}
		\mathbb{E}[\mathbf{x}_{n}\mathbf{x}_{n}^{\top}]&=\mathbf{A}^{n}\mathbf{\Sigma}_{\mathrm{r}}(\mathbf{A}^{n})^{\top}+\sum_{i=1}^{n}\mathbf{A}^{n-i}\mathbf{\Sigma}_{i}(\mathbf{A}^{n-i})^{\top}\\
		\mathbb{E}[\mathbf{x}_{m}\mathbf{x}_{m}^{\top}]&=\mathbf{A}^{m}\mathbf{\Sigma}_{\mathrm{r}}(\mathbf{A}^{m})^{\top}+\sum_{i=1}^{m}\mathbf{A}^{m-i}\mathbf{\Sigma}_{i}^{\prime}(\mathbf{A}^{m-i})^{\top}\\
		\mathbb{E}[\mathbf{x}_{t}\mathbf{x}_{n}^{\top}]&=\mathbf{A}^{t}\mathbf{\Sigma}_{\mathrm{r}}(\mathbf{A}^{n})^{\top}+\sum_{i=1}^{t}\mathbf{A}^{t-i}\mathbf{\Sigma}_{i}(\mathbf{A}^{n-i})^{\top},
	\end{align}
	where $\mathbf{\Sigma}_{i}$ and $\mathbf{\Sigma}_{i}^{\prime}$ are the covariance matrices of the independent noises in each branch.

	Then we calculate the distribution of conditional distribution
	\begin{align}
		\left[
			\begin{matrix}
				\mathbf{x}_{n}\\
				\mathbf{x}_{m}
				\end{matrix}
				\right]\mid \mathbf{x}_{t} \sim \mathcal{N}(\tilde{\mathbf{\mu}},\tilde{\mathbf{\Sigma}}),
	\end{align}
	where
	\begin{align}
		\tilde{\mathbf{\Sigma}}=\left[
			\begin{matrix}
				\tilde{\mathbf{\Sigma}}_{11} & \tilde{\mathbf{\Sigma}}_{12} \\
				\tilde{\mathbf{\Sigma}}_{21} & \tilde{\mathbf{\Sigma}}_{22}
				\end{matrix}
				\right].
	\end{align}
	We have
	\begin{align}
		\tilde{\mathbf{\Sigma}}_{12}=\mathbf{A}^{n}\mathbf{\Sigma}_{\mathrm{r}}(\mathbf{A}^{m})^{\top}&-\Big(\mathbf{A}^{n}\mathbf{\Sigma}_{\mathrm{r}}(\mathbf{A}^{t})^{\top}+\sum_{i=1}^{t}\mathbf{A}^{n-i}\mathbf{\Sigma}_{i}\mathbf{A}^{(t-i)\top}\Big)\nonumber\\
		&\times\Big(\mathbf{A}^{t}\mathbf{\Sigma}_{\mathrm{r}}(\mathbf{A}^{t})^{\top}+\sum_{i=1}^{t}\mathbf{A}^{t-i}\mathbf{\Sigma}_{i}\mathbf{A}^{(t-i)\top}\Big)^{-1}\mathbf{A}^{t}\mathbf{\Sigma}_{\mathrm{r}}(\mathbf{A}^{m})^{\top}=\mathbf{0}.
	\end{align}
	Thus, the conditional independence of $x_{n}$ and $x_{m}$ given $x_{s}$ is proved.

	When ${x}_{n}$ and ${x}_{m}$ are on the same branch, a similar calculation can be performed to prove the conditional independence property.
\end{proof}
\begin{proposition}\label{prop:distgraphdet}
	For a tree graph $\mathbb{T}=(\mathcal{V},\mathcal{E})$ where $\mathcal{V}=\{x_{1},x_{2},\ldots,x_{p}\}$ and any symmetric matrix $\mathbf{A}\in\mathbb{R}^{d\times d}$ whose absolute values of all the eigenvalues are less than 1, the determinant of 
	the matrix $\mathbf{\bar{D}(\mathbb{T},\mathbf{A})}$, which is defined below, is $\big[\det(\mathbf{I}-\mathbf{A}^{2})\big]^{p-1}$
	\begin{align}
		\mathbf{\bar{D}(\mathbb{T},\mathbf{A})}=\left[
			\begin{matrix}
				\mathbf{A}^{\mathrm{d_{\mathbb{T}}}(x_{1},x_{1})} & \mathbf{A}^{\mathrm{d_{\mathbb{T}}}(x_{1},x_{2})}& \cdots &\mathbf{A}^{\mathrm{d_{\mathbb{T}}}(x_{1},x_{p})} \\
				\mathbf{A}^{\mathrm{d_{\mathbb{T}}}(x_{2},x_{1})} & \mathbf{A}^{\mathrm{d_{\mathbb{T}}}(x_{2},x_{2})}& \cdots & \mathbf{A}^{\mathrm{d_{\mathbb{T}}}(x_{2},x_{p})} \\
				\vdots& \vdots &\ \ddots \ & \vdots\\
				\mathbf{A}^{\mathrm{d_{\mathbb{T}}}(x_{p},x_{1})} & \mathbf{A}^{\mathrm{d_{\mathbb{T}}}(x_{p},x_{2})}& \cdots & \mathbf{A}^{\mathrm{d_{\mathbb{T}}}(x_{p},x_{p})}
				\end{matrix}
				\right],
	\end{align}
\end{proposition}
\begin{proof}[Proof of Proposition ~\ref{prop:distgraphdet}]
	Since the underlying structure is a tree, we can always find a leaf and its neighbor. Without loss of generality, we assume $x_{p}$ is a leaf and $x_{p-1}$ 
	is $x_{p}$'s neighbor, otherwise we can exchange the rows and columns of $\mathbf{\bar{D}(\mathbb{T},\mathbf{A})}$ to satisfy this assumption. Then we have
	\begin{align}
		\mathrm{d_{\mathbb{T}}}(x_{p-1},x_{p})=1 \quad\mbox{and}\quad \mathrm{d_{\mathbb{T}}}(x_{p},x_{i})=\mathrm{d_{\mathbb{T}}}(x_{p-1},x_{i})+1 \quad\mbox{for all}\quad   i\in [p-2].
	\end{align}
	Thus, we have
	\begin{align}
		\mathbf{\bar{D}(\mathbb{T},\mathbf{A})}=\left[
			\begin{matrix}
				\mathbf{A}^{0} & \mathbf{A}^{\mathrm{d_{\mathbb{T}}}(x_{1},x_{2})}& \cdots &\mathbf{A}^{\mathrm{d_{\mathbb{T}}}(x_{1},x_{p-1})} &\mathbf{A}^{\mathrm{d_{\mathbb{T}}}(x_{1},x_{p-1})+1} \\
				\mathbf{A}^{\mathrm{d_{\mathbb{T}}}(x_{2},x_{1})} & \mathbf{A}^{0}& \cdots &\mathbf{A}^{\mathrm{d_{\mathbb{T}}}(x_{2},x_{p-1})}& \mathbf{A}^{\mathrm{d_{\mathbb{T}}}(x_{2},x_{p-1})+1} \\
				\vdots& \vdots &\ \ddots \ & \vdots & \vdots\\
				\mathbf{A}^{\mathrm{d_{\mathbb{T}}}(x_{p-1},x_{1})} & \mathbf{A}^{\mathrm{d_{\mathbb{T}}}(x_{p-1},x_{2})}& \cdots & \mathbf{A}^{0} & \mathbf{A}^{1}\\
				\mathbf{A}^{\mathrm{d_{\mathbb{T}}}(x_{p-1},x_{1})+1} & \mathbf{A}^{\mathrm{d_{\mathbb{T}}}(x_{p-1},x_{2})+1}& \cdots & \mathbf{A}^{1} & \mathbf{A}^{0}
				\end{matrix}
				\right].
	\end{align}
	Subtracting $\mathbf{A}$ times the penultimate row of $\mathbf{\bar{D}(\mathbb{T},\mathbf{A})}$ from the last row of $\mathbf{\bar{D}(\mathbb{T},\mathbf{A})}$, we have
	\begin{align}
		\left[
			\begin{matrix}
				\mathbf{A}^{0} & \mathbf{A}^{\mathrm{d_{\mathbb{T}}}(x_{1},x_{2})}& \cdots &\mathbf{A}^{\mathrm{d_{\mathbb{T}}}(x_{1},x_{p-1})} &\mathbf{A}^{\mathrm{d_{\mathbb{T}}}(x_{1},x_{p-1})+1} \\
				\mathbf{A}^{\mathrm{d_{\mathbb{T}}}(x_{2},x_{1})} & \mathbf{A}^{0}& \cdots &\mathbf{A}^{\mathrm{d_{\mathbb{T}}}(x_{2},x_{p-1})}& \mathbf{A}^{\mathrm{d_{\mathbb{T}}}(x_{2},x_{p-1})+1} \\
				\vdots& \vdots &\ \ddots \ & \vdots & \vdots\\
				\mathbf{A}^{\mathrm{d_{\mathbb{T}}}(x_{p-1},x_{1})} & \mathbf{A}^{\mathrm{d_{\mathbb{T}}}(x_{p-1},x_{2})}& \cdots & \mathbf{A}^{0} & \mathbf{A}^{1}\\
				\mathbf{0} & \mathbf{0}& \cdots & \mathbf{0} & \mathbf{A}^{0}-\mathbf{A}^{2}
				\end{matrix}
				\right].
	\end{align}
	Applying the similar column transformation, we have
	\begin{align}
		\left[
			\begin{matrix}
				\mathbf{A}^{0} & \mathbf{A}^{\mathrm{d_{\mathbb{T}}}(x_{1},x_{2})}& \cdots &\mathbf{A}^{\mathrm{d_{\mathbb{T}}}(x_{1},x_{p-1})} &\mathbf{0} \\
				\mathbf{A}^{\mathrm{d_{\mathbb{T}}}(x_{2},x_{1})} & \mathbf{A}^{0}& \cdots &\mathbf{A}^{\mathrm{d_{\mathbb{T}}}(x_{2},x_{p-1})}& \mathbf{0} \\
				\vdots& \vdots &\ \ddots \ & \vdots & \vdots\\
				\mathbf{A}^{\mathrm{d_{\mathbb{T}}}(x_{p-1},x_{1})} & \mathbf{A}^{\mathrm{d_{\mathbb{T}}}(x_{p-1},x_{2})}& \cdots & \mathbf{A}^{0} & \mathbf{0}\\
				\mathbf{0} & \mathbf{0}& \cdots & \mathbf{0} & \mathbf{A}^{0}-\mathbf{A}^{2}
				\end{matrix}
				\right].
	\end{align}
	By repeating these row and column transformations, we will acquire
	\begin{align}
		\mathrm{diag}(\mathbf{I},\mathbf{I}-\mathbf{A}^{2},\ldots,\mathbf{I}-\mathbf{A}^{2}),
	\end{align}
	which has the same determinant as $\mathbf{\bar{D}(\mathbb{T},\mathbf{A})}$. Thus,  $\mathrm{det}(\mathbf{\bar{D}(\mathbb{T},\mathbf{A})})= \big[\det(\mathbf{I}-\mathbf{A}^{2})\big]^{p-1}$.
\end{proof}
The proof of Theorem \ref{theo:converse} follows from the following non-asymptotic result. 
\begin{theorem}\label{theo:converse2}
	Consider the class of graphs $\mathcal{T}(|\mathcal{V}_{\mathrm{obs}}|,\rho_{\max},l_{\max})$, where $|\mathcal{V}_{\mathrm{obs}}|\geq 3$. If the number of i.i.d.\ samples $n$ is upper bounded as follows, 
	\begin{align}
		n<\max\bigg\{\frac{2(1-\delta)\big(\log 3^{1/3} \lfloor\log_{3}(|\mathcal{V}_{\mathrm{obs}}|)\rfloor-1\big)-\frac{2}{|\mathcal{V}_{\mathrm{obs}}|}}{-l_{\max}\log\big(1-e^{-\frac{\rho_{\max}}{\lfloor\log_{3}(|\mathcal{V}_{\mathrm{obs}}|)\rfloor l_{\max}}}\big)},\frac{(1-\delta)/5-\frac{2}{|\mathcal{V}_{\mathrm{obs}}|}}{-l_{\max}\log\big(1-e^{-\frac{2\rho_{\max}}{3l_{\max}}}\big)}
		\bigg\} \label{eqn:boundThm10}
	\end{align}
	then for any graph decoder $\phi: \mathbb{R}^{n|\mathcal{V}_{\mathrm{obs}}|l_{\max}} \rightarrow \mathcal{T}(|\mathcal{V}_{\mathrm{obs}}|,\rho_{\max},l_{\max})$
	\begin{align}
		\max_{\theta(\mathbb{T})\in \mathcal{T}(|\mathcal{V}_{\mathrm{obs}}|,\rho_{\max},l_{\max})} \mathbb{P}_{\theta(\mathbb{T})}(\phi(\mathbf{X}_{1}^{n})\neq \mathbb{T})\geq \delta.
	\end{align}
\end{theorem}

\begin{proof}[Proof of Theorem ~\ref{theo:converse}]
	To prove Theorem \ref{theo:converse}, we simply implement the Taylor expansion $\log (1+x)=\sum_{k=1}^{\infty} (-1)^{k+1}\frac{x^{k}}{k}$ on \eqref{eqn:boundThm10} in Theorem \ref{theo:converse2} taking $\rho_{\max}\to\infty$ and $|\mathcal{V}_{\mathrm{obs}}|\to\infty$. 
\end{proof}	

	It remains to prove   Theorem \ref{theo:converse2}. 
\begin{proof}[Proof of Theorem ~\ref{theo:converse2}]	
	To prove this non-asymptotic converse bound, we consider $M$ models in $\mathcal{T}(|\mathcal{V}_{\mathrm{obs}}|,\rho_{\max},l_{\max})$, whose parameters are enumerated as $\{\theta^{(1)},\theta^{(2)},\ldots,\theta^{(M)}\}$. We choose a 
	model $K=k$ uniformly in $\{1,\ldots,M\}$ and generate $n$ i.i.d.\ samples $\mathbf{X}_{1}^{n}$ from $\mathbb{P}_{\theta^{(k)}}$. A latent tree learning algorithm is a decoder $\phi:\mathbb{R}^{n|\mathcal{V}_{\mathrm{obs}}|l_{\max}}\rightarrow  \{1,\ldots,M\}$.

	Two families are built to derive the converse bound. We separately describe the families of $M$ graphical models we consider here. 
	\paragraph{Graphical model family A} 
	We specify the structure of trees as full-$m$ trees, except the top layer, as shown in Fig.~\ref{fig:conversetree}. All the observed nodes are leaves. The parameters of each tree are set to satisfy the conditions in Proposition \ref{prop:homo}. Additionally, we set $\alpha=1$ in the homogeneous condition \eqref{eqn:homo} and 
	set $\mathbf{A}$ to be a symmetric matrix that commutes with $\mathbf{\Sigma}_{r}$. We set $m=3$ and $L=\lfloor\log_{3}(|\mathcal{V}_{\mathrm{obs}}|)\rfloor$, then the number of residual nodes is $r=|\mathcal{V}_{\mathrm{obs}}|-3^{L}$. All these residual nodes are connected to one of parents of the observed nodes.

	\begin{figure}[htbp]
		\centering
		\subfigure[The full 3-tree. All the observed nodes are leaves, and residual nodes are connected to one of parents of the observed nodes.]{
		\begin{minipage}[t]{.92\linewidth}
		\centering
		\includegraphics[width=5.1in]{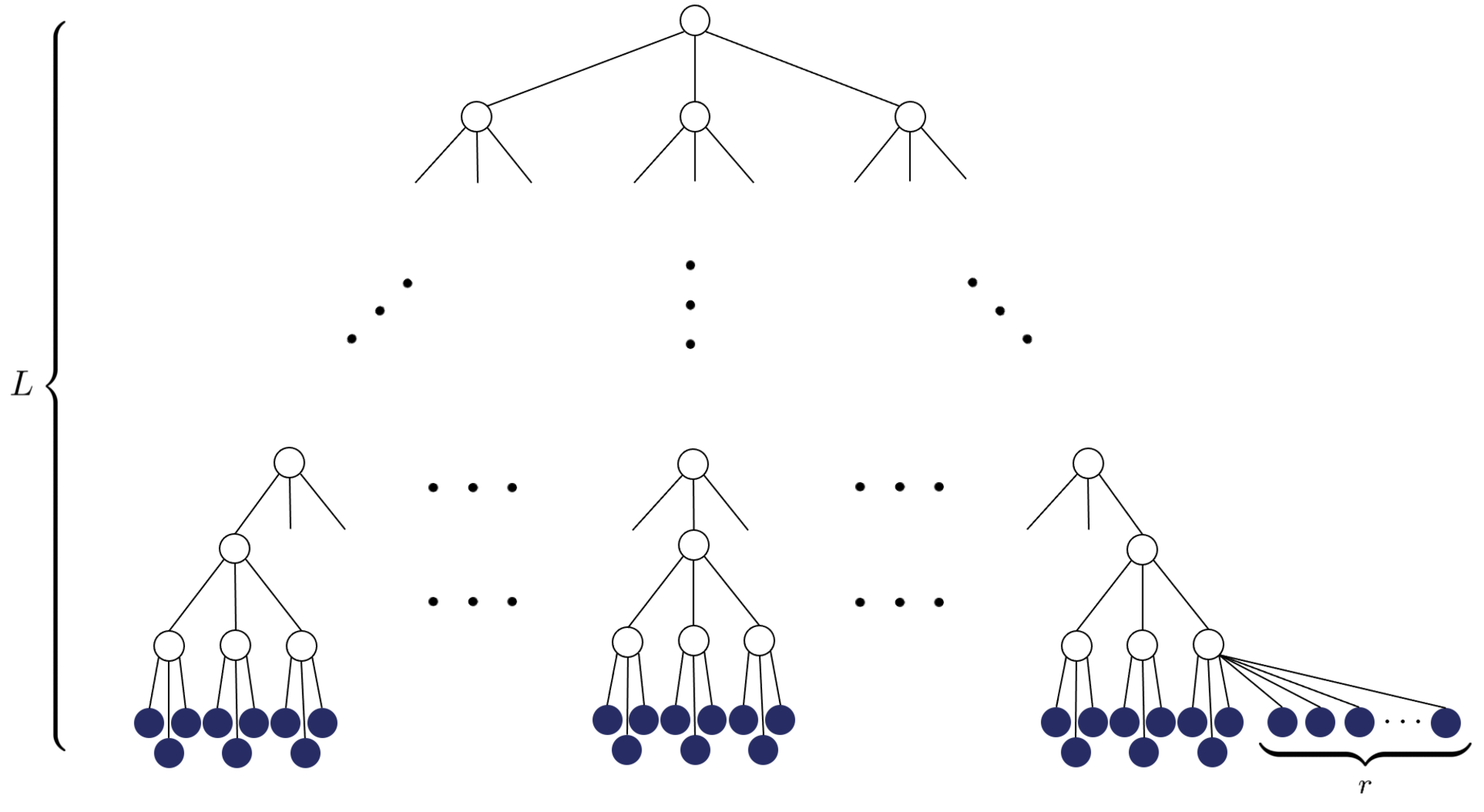}
		%\caption{fig1}
		\end{minipage}%
		}

		\subfigure[The full tree with depth $k$, where all the internal nodes have three children except the root node.]{
		\begin{minipage}[t]{.92\linewidth}
		\centering
		\includegraphics[width=5.1in]{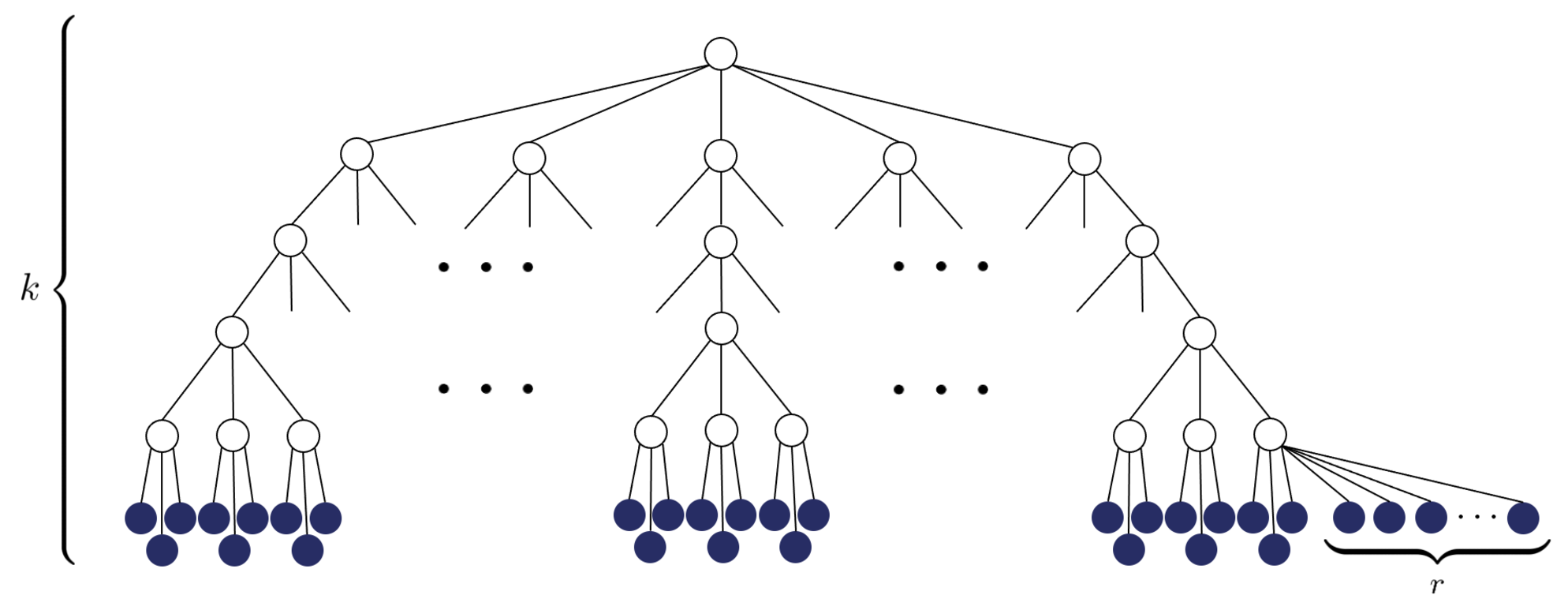}
		%\caption{fig2}
		\end{minipage}%
		}%
		\centering
		\caption{The family A of graphical models considered in the impossibility result.}
		\label{fig:conversetree}
	\end{figure}
	
	To derive the converse result, we use the Fano's method. Namely, Fano's method says that if the sample size
	\begin{align}
		n<\frac{(1-\delta)\log M}{I(\mathbf{X}_{1};K)},
	\end{align}
	then for any decoder
	\begin{align}
		\max_{k=1,\ldots,M} \mathbb{P}_{\theta^{(k)}}\big[\phi(\mathbf{X}_{1}^{n})\neq k\big]\geq \delta-\frac{1}{\log M}.
	\end{align}

	We first evaluate the cardinality of this family of graphical models. We first count the number of graphical models with depth $1\leq k\leq L$ in   Fig.~\ref{fig:conversetree}. For a 
	specific order of labels (e.g., $1,2,\ldots,m^{L}$), exchanging the labels in a family does not change the topology of the tree. For instance, exchanging the position of node $1$ and 
	node $m$, we obtain an identical tree. By changing the orders in the last layer, it is obvious that there are $(m!)^{m^{L-1}}$ different orders representing the same structure. For the 
	penultimate layer, there are $(m!)^{m^{L-2}}$ different orders represent an identical structure. Thus, for a specific graphical model with depth $k$, there are
	\begin{align}
		(m^{L-k+1})!\prod_{i=L-1}^{L-k+1}(m!)^{m^{i}}=(m^{L-k+1})!(m!)^{\frac{m^{L}-m^{L-k+1}}{m-1}}
	\end{align}
	graphical models with the same distribution. %that have the same distribution as it.

	Then the number of different structures of graphical models with depth $k$ can be calculated as 
	\begin{align}
		\frac{(m^{L})!}{(m^{L-k+1})!(m!)^{\frac{m^{L}-m^{L-k+1}}{m-1}}}.
	\end{align}
	The total number of different graphical models in the family we consider is
	\begin{align}
		M=\sum_{k=1}^{L} \frac{(m^{L})!}{(m^{L-k+1})!(m!)^{\frac{m^{L}-m^{L-k+1}}{m-1}}}.
	\end{align}
	Using Stirling's formula, we have the following simplification of $M$:
	\begin{align}\label{neq:Mbound}
		M%&=\sum_{k=1}^{L} \frac{(m^{L})!}{(m^{L-k+1})!(m!)^{\frac{m^{L}-m^{L-k+1}}{m-1}}} \\
		& \geq \sum_{k=1}^{L} \frac{\sqrt{2\pi}(m^{L})^{m^{L}+1/2}e^{-m^{L}}}{e(m^{L-k+1})^{m^{L-k+1}+1/2}e^{-m^{L-k+1}}}\frac{1}{(e^{-(m-1)}m^{m+1/2})^{(m^{L}-m^{L-k+1})/(m-1)}}\\
		&= \sum_{k=1}^{L}\frac{\sqrt{2\pi}}{e}m^{Lm^{L}-(L-k+1)m^{L-k+1}+(k-1)/2-(m+1/2)(m^{L}-m^{L-k+1})/(m-1)}\\
		&= \frac{\sqrt{2\pi}}{e} m^{(L-(m+1/2)/(m-1))m^{L}}\sum_{k=1}^{L} m^{-(L-k+1-(m+1/2)/(m-1))m^{L-k+1}+(k-1)/2}\\
		&> \frac{\sqrt{2\pi}}{e} m^{(L-(m+1/2)/(m-1))m^{L}}m^{3m/(2m-2)+(L-1)/2}
	\end{align}
	and
	\begin{align}
		\log M&>\log\Big(\frac{\sqrt{2\pi}}{e}\Big)+m^{L}\Big(L-\frac{m+\frac{1}{2}}{m-1}\Big)\log m +\Big(\frac{3m}{2m-2}+\frac{L-1}{2}\Big)\log m
		\end{align}
		Thus,
		\begin{align}
		\frac{\log M}{|\mathcal{V}_{\mathrm{obs}}|}&>\frac{1}{|\mathcal{V}_{\mathrm{obs}}|}\log\Big(\frac{\sqrt{2\pi}}{e}\Big)+\frac{m^{L}}{|\mathcal{V}_{\mathrm{obs}}|}\Big(L-\frac{m+\frac{1}{2}}{m-1}\Big)\log m+\frac{1}{|\mathcal{V}_{\mathrm{obs}}|}\Big(\frac{3m}{2m-2}+\frac{L-1}{2}\Big)\log m \\
		&>\frac{\log m}{m}\Big(L-\frac{m+\frac{1}{2}}{m-1}\Big)+\frac{1}{|\mathcal{V}_{\mathrm{obs}}|}\log\Big(\frac{\sqrt{2\pi}}{e}\Big) \\
		&\overset{(a)}{>}\frac{\log 3}{3}L-1,
	\end{align}
	where inequality $(a)$ is derived by substituting $m=3$.
	
	Next we calculate an upper bound of $I(\mathbf{X}_{1};K)$.
	Since $\mathbb{P}_{\mathbb{T}_{k}}=\mathcal{N}(0,\mathbf{\Sigma}_{\mathrm{obs}}(\mathbb{T}_{k}))$, where $\mathbf{\Sigma}_{\mathrm{obs}}$ is the covariance matrix of observed variables, we have \cite{wang2010information}
	\begin{align}
		I(\mathbf{X}_{1};K)\leq \mathbb{E}_{\mathbb{T}_{k}}\big[D(\mathbb{P}_{\mathbb{T}_{k}}\|\mathbb{Q})\big],
	\end{align}
	for any distribution $\mathbb{Q}$. By choosing $\mathbb{Q}=\mathcal{N}(0,\mathbf{I}_{l_{\max}|\mathcal{V}_{\mathrm{obs}}|\times l_{\max}|\mathcal{V}_{\mathrm{obs}}|})$, we have
	\begin{align}
		D(\mathbb{P}_{\mathbb{T}_{k}}\|\mathbb{Q})&=\frac{1}{2}\bigg\{\log\Big(\det\big(\mathbf{\Theta}_{\mathrm{obs}}(\mathbb{T}_{k})\big)\Big)+\mathrm{trace}\big(\mathbf{\Sigma}_{\mathrm{obs}}(\mathbb{T}_{k})\big)-l_{\max}|\mathcal{V}_{\mathrm{obs}}|\bigg\} \\
		&=\frac{1}{2}\bigg\{-\log\Big(\det\big(\mathbf{\Sigma}_{\mathrm{obs}}(\mathbb{T}_{k})\big)\Big)+\mathrm{trace}\big(\mathbf{\Sigma}_{\mathrm{obs}}(\mathbb{T}_{k})\big)-l_{\max}|\mathcal{V}_{\mathrm{obs}}|\bigg\}
	\end{align}
	Since we consider models that satisfy the conditions in Proposition \ref{prop:homo}, the covariance matrix of any two variables is
	\begin{align}
		\mathbb{E}\big[\mathbf{x}_{i}\mathbf{x}_{j}^{\top}\big]=\mathbf{\Sigma}_{\mathrm{r}}\mathbf{A}^{\mathrm{d_{\mathbb{T}}}(x_{i},x_{j})}.
	\end{align}
	The covariance matrix $\mathbf{\Sigma}(\mathbb{T}_{k})$ for all the observed variables and latent variables $\mathcal{V}_{\mathrm{obs}}\cup\mathcal{V}_{\mathrm{hid}}$ is
	\begin{align}
		&\left[
			\begin{matrix}
				\mathbf{\Sigma}_{\mathrm{r}}\mathbf{A}^{\mathrm{d_{\mathbb{T}}}(x_{1},x_{1})} & \cdots & \mathbf{\Sigma}_{\mathrm{r}}\mathbf{A}^{\mathrm{d_{\mathbb{T}}}(x_{1},x_{|\mathcal{V}_{\mathrm{obs}}|})}& \mathbf{\Sigma}_{\mathrm{r}}\mathbf{A}^{\mathrm{d_{\mathbb{T}}}(x_{1},y_{1})} & \cdots &\mathbf{\Sigma}_{\mathrm{r}}\mathbf{A}^{\mathrm{d_{\mathbb{T}}}(x_{1},y_{|\mathcal{V}_{\mathrm{hid}}|})} \\
				\vdots & \ddots &\vdots &\vdots &\ddots &\vdots\\
				\mathbf{\Sigma}_{\mathrm{r}}\mathbf{A}^{\mathrm{d_{\mathbb{T}}}(x_{|\mathcal{V}_{\mathrm{obs}}|},x_{1})} & \cdots & \mathbf{\Sigma}_{\mathrm{r}}\mathbf{A}^{\mathrm{d_{\mathbb{T}}}(x_{|\mathcal{V}_{\mathrm{obs}}|},x_{|\mathcal{V}_{\mathrm{obs}}|})}& \mathbf{\Sigma}_{\mathrm{r}}\mathbf{A}^{\mathrm{d_{\mathbb{T}}}(x_{|\mathcal{V}_{\mathrm{obs}}|},y_{1})} & \cdots &\mathbf{\Sigma}_{\mathrm{r}}\mathbf{A}^{\mathrm{d_{\mathbb{T}}}(x_{|\mathcal{V}_{\mathrm{obs}}|},x_{|\mathcal{V}_{\mathrm{hid}}|})} \\
				\mathbf{\Sigma}_{\mathrm{r}}\mathbf{A}^{\mathrm{d_{\mathbb{T}}}(y_{1},x_{1})} & \cdots & \mathbf{\Sigma}_{\mathrm{r}}\mathbf{A}^{\mathrm{d_{\mathbb{T}}}(y_{1},x_{|\mathcal{V}_{\mathrm{obs}}|})}& \mathbf{\Sigma}_{\mathrm{r}}\mathbf{A}^{\mathrm{d_{\mathbb{T}}}(y_{1},y_{1})} & \cdots &\mathbf{\Sigma}_{\mathrm{r}}\mathbf{A}^{\mathrm{d_{\mathbb{T}}}(y_{1},y_{|\mathcal{V}_{\mathrm{hid}}|})} \\
				\vdots & \ddots &\vdots &\vdots &\ddots &\vdots\\
				\mathbf{\Sigma}_{\mathrm{r}}\mathbf{A}^{\mathrm{d_{\mathbb{T}}}(y_{|\mathcal{V}_{\mathrm{hid}}|},x_{1})} & \cdots & \mathbf{\Sigma}_{\mathrm{r}}\mathbf{A}^{\mathrm{d_{\mathbb{T}}}(y_{|\mathcal{V}_{\mathrm{hid}}|},x_{|\mathcal{V}_{\mathrm{obs}}|})}& \mathbf{\Sigma}_{\mathrm{r}}\mathbf{A}^{\mathrm{d_{\mathbb{T}}}(y_{|\mathcal{V}_{\mathrm{hid}}|},y_{1})} & \cdots &\mathbf{\Sigma}_{\mathrm{r}}\mathbf{A}^{\mathrm{d_{\mathbb{T}}}(y_{|\mathcal{V}_{\mathrm{hid}}|},y_{|\mathcal{V}_{\mathrm{hid}}|})} 
				\end{matrix}
				\right] \nonumber\\
				&=\Big(\mathbf{I}_{|\mathcal{V}|\times|\mathcal{V}|}\otimes \mathbf{\Sigma}_{\mathrm{r}}\Big) \left[
				\begin{matrix}
					\mathbf{V}  & \mathbf{B}\\
					\mathbf{B}^{\top}  & \mathbf{H}\\
				\end{matrix}
					\right]\\
			&=\Big(\mathbf{I}_{|\mathcal{V}|\times|\mathcal{V}|}\otimes \mathbf{\Sigma}_{\mathrm{r}}\Big) \mathbf{\bar{D}}(\mathbb{T}_{k},\mathbf{A})
	\end{align}
	where $\mathbf{A}\otimes\mathbf{B}$ is the Kronecker product of matrices $\mathbf{A}$ and $\mathbf{B}$. Letting $\mathbf{\Sigma}_{\mathrm{r}}=\mathbf{I}$, it is obvious that $\mathbf{\bar{D}}(\mathbb{T},\mathbf{A})$ is a positive definite matrix. Furthermore, $\mathbf{H}-\mathbf{B}^{\top}\mathbf{V}^{-1}\mathbf{B}$ is 
	positive semi-definite matrix, since it is the inverse of the principal minor of $\mathbf{\bar{D}}^{-1}$. Thus  we have
	\begin{align}
		\det\big(\mathbf{\bar{D}}(\mathbb{T}_{k},\mathbf{A})\big)&=\det(\mathbf{V})\det\big(\mathbf{H}-\mathbf{B}^{\top}\mathbf{V}^{-1}\mathbf{B}\big)\\
		&\overset{(a)}{\leq}\det(\mathbf{V})\det\big(\mathbf{H}\big)\overset{(b)}{=}\det(\mathbf{V})\big[\det(\mathbf{I}-\mathbf{A}^{2})\big]^{|\mathcal{V}_{\mathrm{hid}}|-1},
	\end{align}
	where inequality $(a)$ is derived from Minkowski determinant theorem \cite{marcus1971extension}, and $(b)$ comes from the fact that all the latent variables themselves form a tree. Also, we have that
	\begin{align}
		\det\big(\mathbf{\bar{D}}(\mathbb{T}_{k},\mathbf{A})\big)=\big[\det(\mathbf{I}-\mathbf{A}^{2})\big]^{|\mathcal{V}_{\mathrm{hid}}|+|\mathcal{V}_{\mathrm{obs}}|-1}.
	\end{align}
	Thus, we have
	\begin{align}
		\det(\mathbf{V})\geq \big[\det(\mathbf{I}-\mathbf{A}^{2})\big]^{|\mathcal{V}_{\mathrm{obs}}|},
	\end{align}
	which implies that
	\begin{align}
		\log\Big(\det\big(\mathbf{\Sigma}_{\mathrm{obs}}(\mathbb{T}_{k})\big)\Big)&\geq \log\Big(\big(\det(\mathbf{\Sigma}_{\mathrm{r}})\big)^{|\mathcal{V}_{\mathrm{obs}}|}\big[\det(\mathbf{I}-\mathbf{A}^{2})\big]^{|\mathcal{V}_{\mathrm{obs}}|}\Big)\\
		&=|\mathcal{V}_{\mathrm{obs}}|\log\Big(\det(\mathbf{\Sigma}_{\mathrm{r}})\det(\mathbf{I}-\mathbf{A}^{2})\Big).
	\end{align}
	The mutual information can thus be upper bounded as 
	\begin{align}\label{neq:mibound}
		I(\mathbf{X}_{1};K)\leq \frac{1}{2}|\mathcal{V}_{\mathrm{obs}}|(-\log\big(\det(\mathbf{I}-\mathbf{A}^{2})\big)+\mathrm{trace}(\mathbf{\Sigma}_{\mathrm{r}})-\log\big(\det(\mathbf{\Sigma}_{\mathrm{r}})\big)-l_{\max})
	\end{align}
	Combining inequalities~\eqref{neq:Mbound} and \eqref{neq:mibound}, we can deduce that the any decoder will construct the wrong tree with probability at least $\delta$ if
	\begin{align}
		n<\frac{2(1-\delta)\big(\log 3^{1/3} L-1\big)}{-\log\big(\det(\mathbf{I}-\mathbf{A}^{2})\big)+\mathrm{trace}(\mathbf{\Sigma}_{\mathrm{r}})-\log\big(\det(\mathbf{\Sigma}_{\mathrm{r}})\big)-l_{\max}}
	\end{align}
	By choosing $\mathbf{\Sigma}_{\mathrm{r}}=\mathbf{I}$ and letting the eigenvalues of $\mathbf{A}$ are all the same, we have
	\begin{align}
		\rho_{\max}&=-\frac{2L}{2}\log\big(\det(\mathbf{A}^{2})\big)=-2l_{\max}L\log\big(\lambda(\mathbf{A})\big) 
		\end{align}
		and 
		\begin{align}
		\mathrm{trace}(\mathbf{\Sigma}_{\mathrm{r}})-\log\big(\det(\mathbf{\Sigma}_{\mathrm{r}})\big)-l_{\max}=0. 
	\end{align}
	Furthermore, we have
	\begin{align}
		\log\big(\det(\mathbf{I}-\mathbf{A}^{2})\big)=l_{\max}\log\big(1-\lambda(\mathbf{A})^{2}\big)=l_{\max}\log\big(1-e^{-\frac{\rho_{\max}}{Ll_{\max}}}\big).
	\end{align}
	By choosing $\delta^{\prime}=\delta+\frac{1}{\log(M)}$, we have that the condition
	\begin{align}
		n<\frac{2(1-\delta^{\prime})\big(\log 3^{1/3} L-1\big)-\frac{2}{|\mathcal{V}_{\mathrm{obs}}|}}{-l_{\max}\log\big(1-e^{-\frac{\rho_{\max}}{Ll_{\max}}}\big)}
	\end{align}
	guarantees that 
	\begin{align}
		\max_{\theta(\mathbb{T})\in \mathcal{T}(|\mathcal{V}_{\mathrm{obs}}|,\rho_{\max},l_{\max})} \mathbb{P}_{\theta(\mathbb{T})}(\phi(\mathbf{X}_{1}^{n})\neq \mathbb{T})\geq \delta^{\prime}
	\end{align}
	
	\paragraph{Graphical model family B} We consider the family of graphical models with double-star substructures, as shown in Fig.~\ref{fig:conversetree3}. Then the number of graphical models $M$ in this family is lower bounded as
	\begin{align}
		M&> \frac{1}{2}\binom{|\mathcal{V}_{\mathrm{obs}}|}{\lceil |\mathcal{V}_{\mathrm{obs}}|/2 \rceil}>\frac{\sqrt{2\pi}|\mathcal{V}_{\mathrm{obs}}|^{|\mathcal{V}_{\mathrm{obs}}|+1/2}e^{-|\mathcal{V}_{\mathrm{obs}}|}}{2\big(e n^{n+1/2}e^{-n}\big)\big(e (|\mathcal{V}_{\mathrm{obs}}|-n)^{|\mathcal{V}_{\mathrm{obs}}|-n+1/2}e^{-(|\mathcal{V}_{\mathrm{obs}}|-n)}\big)}\nonumber\\
		&=\frac{\sqrt{2\pi}}{2e^{2}}\frac{|\mathcal{V}_{\mathrm{obs}}|^{|\mathcal{V}_{\mathrm{obs}}|+1/2}}{n^{n+1/2}(|\mathcal{V}_{\mathrm{obs}}|-n)^{|\mathcal{V}_{\mathrm{obs}}|-n+1/2}}
	\end{align}
	
	\begin{figure}[t]
	\centering\includegraphics[width=0.75\columnwidth,draft=false]{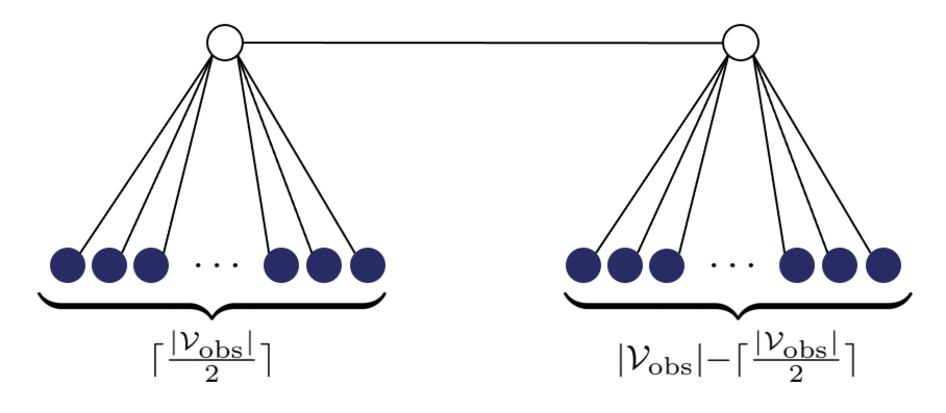}
	\caption{The family B of graphical models considered in the impossibility result.}
	\label{fig:conversetree3}
\end{figure}
	
	when $n=\lceil |\mathcal{V}_{\mathrm{obs}}|/2 \rceil$. Since $n\geq |\mathcal{V}_{\mathrm{obs}}|-n$, we further have
	\begin{align}
		M&>\frac{\sqrt{2\pi}}{2e^{2}} \frac{|\mathcal{V}_{\mathrm{obs}}|^{|\mathcal{V}_{\mathrm{obs}}|+1/2}}{n^{|\mathcal{V}_{\mathrm{obs}}|+1}}=\sqrt{\frac{\pi}{2e^{4}|\mathcal{V}_{\mathrm{obs}}|}}\Big(\frac{|\mathcal{V}_{\mathrm{obs}}|}{n}\Big)^{|\mathcal{V}_{\mathrm{obs}}|+1}
	\end{align}
		and 
	\begin{align}
		\frac{\log(M)}{|\mathcal{V}_{\mathrm{obs}}|}&>\frac{|\mathcal{V}_{\mathrm{obs}}|+1}{|\mathcal{V}_{\mathrm{obs}}|}\log\Big(\frac{|\mathcal{V}_{\mathrm{obs}}|}{\lceil |\mathcal{V}_{\mathrm{obs}}|/2 \rceil}\Big)+\frac{1}{|\mathcal{V}_{\mathrm{obs}}|}\log\Big(\sqrt{\frac{\pi}{2e^{4}|\mathcal{V}_{\mathrm{obs}}|}}\Big) \\
		&>\frac{|\mathcal{V}_{\mathrm{obs}}|+1}{|\mathcal{V}_{\mathrm{obs}}|}\log 2+\frac{1}{|\mathcal{V}_{\mathrm{obs}}|}\log\Big(\sqrt{\frac{\pi}{2e^{4}|\mathcal{V}_{\mathrm{obs}}|}}\Big)>\frac{1}{10}
	\end{align}
	By choosing $\mathbf{\Sigma}_{\mathrm{r}}=\mathbf{I}$ and letting all the eigenvalues of $\mathbf{A}$ to be  the same, we have
	\begin{align}
		\rho_{\max}=-\frac{3}{2}\log\big(\det(\mathbf{A}^{2})\big)=-3l_{\max}\log\big(\lambda(\mathbf{A})\big)
		\end{align}
		and 
		\begin{align}
		\mathrm{trace}(\mathbf{\Sigma}_{\mathrm{r}})-\log\big(\det(\mathbf{\Sigma}_{\mathrm{r}})\big)-l_{\max}=0.
	\end{align}
	Furthermore, we have
	\begin{align}
		\log\big(\det(\mathbf{I}-\mathbf{A}^{2})\big)=l_{\max}\log\big(1-\lambda(\mathbf{A})^{2}\big)=l_{\max}\log\big(1-e^{-\frac{2\rho_{\max}}{3l_{\max}}}\big).
	\end{align}
	By choosing $\delta^{\prime}=\delta+\frac{1}{\log(M)}$, we have that the condition
	\begin{align}
		n<\frac{(1-\delta^{\prime})/5-\frac{2}{|\mathcal{V}_{\mathrm{obs}}|}}{-l_{\max}\log\big(1-e^{-\frac{2\rho_{\max}}{3l_{\max}}}\big)}
	\end{align}
	guarantees that 
	\begin{align}
		\max_{\theta(\mathbb{T})\in \mathcal{T}(|\mathcal{V}_{\mathrm{obs}}|,\rho_{\max},l_{\max})} \mathbb{P}_{\theta(\mathbb{T})}(\phi(\mathbf{X}_{1}^{n})\neq \mathbb{T})\geq \delta^{\prime}
	\end{align}
	as desired. 
\end{proof}

\end{document}